\documentclass[letter,11pt]{article}
\usepackage[margin=1in]{geometry}
\usepackage{algorithm,algorithmic}
\usepackage[numbers,sort&compress]{natbib}
\usepackage[utf8]{inputenc} % allow utf-8 input
\usepackage[T1]{fontenc}    % use 8-bit T1 fonts
\usepackage{hyperref}       % hyperlinks
\usepackage{url}            % simple URL typesetting
\usepackage{booktabs}       % professional-quality tables
\usepackage{amsfonts}       % blackboard math symbols
\usepackage{nicefrac}       % compact symbols for 1/2, etc.
\usepackage{microtype}      % microtypography
\usepackage{xcolor}         % colors
\usepackage{amsmath}
\usepackage{graphicx}
\usepackage{amsthm}
\usepackage{amsfonts}
\usepackage{multicol}
\usepackage{amssymb}
\usepackage{bm}
\usepackage{epstopdf}
\usepackage{subcaption}
\usepackage{url}
\usepackage{mathtools}
\usepackage{xcolor}

\newtheorem{assumption}[]{Assumption}
\newtheorem{remark}{\textbf{Remark}}

\newtheorem{theorem}{Theorem}[]
\newtheorem{corollary}{Corollary}[theorem]
\newtheorem{lemma}[]{Lemma}
\usepackage{multirow}
\usepackage{makecell}   % \makecell, \thead
\usepackage{colortbl}   % \rowcolor
\usepackage{hhline}
\usepackage{cleveref}
\usepackage{enumitem}
\usepackage{blindtext}
\definecolor{tabcolor}{RGB}{245,237,227} % light beige (tune if needed)
\usepackage{tcolorbox}
\usepackage{titletoc}

\usepackage[table]{xcolor}
\usepackage{pifont}
\newcommand{\yesmark}{\ding{51}}
\newcommand{\nomark}{\ding{55}}
\usepackage{tabularx}
\newcommand{\E}{\mathbb{E}}
\newcommand{\R}{\mathbb{R}}

\newcommand{\cF}{\mathcal{F}}
\newcommand{\Ocal}{\mathcal{O}}
\newcommand{\Lcal}{\mathcal{L}}

\title{\vspace{-10mm}\textbf{Beyond Bounded Variance: Variance-Reduced Normalized Methods for Nonconvex Optimization under Blum--Gladyshev Noise}}
\date{}
\vspace{-0.05in}
\author{
Antesh Upadhyay, Arda Fazla, Abolfazl Hashemi\thanks{Authors are with the School of Electrical and Computer Engineering, Purdue University, West Lafayette, IN 47907, USA.}}

\begin{document}

\maketitle

\begin{abstract}
Stochastic optimization theory often assumes that gradient noise has uniformly bounded variance. However, this assumption can be inadequate even in simple settings; e.g., it does not hold for linear least squares regression. The Blum--Gladyshev ($\mathsf{BG}$-0) noise model instead allows the stochastic gradient variance to grow quadratically with the distance from the initialization and has recently been identified as the weakest known viable variance assumption. We study nonconvex stochastic optimization under $\mathsf{BG}$-0 noise and prove that normalized momentum converges with one stochastic gradient per iteration, reaching an $\varepsilon$-stationary point with the optimal $\mathcal{O}(\varepsilon^{-6})$ complexity. We then show that, under mean-square smoothness, a variance-reduced normalized STORM method achieves the minimax optimal complexity of $\mathcal{O}(\varepsilon^{-4})$. We further extend our analysis to the symmetric generalized smoothness framework, which captures objectives whose local curvature can scale with the gradient norm. To our knowledge, these are the first convergence guarantees for normalized methods in nonconvex stochastic optimization under $\mathsf{BG}$-0 noise without bounded domains.
\end{abstract}

\section{Introduction}
\label{sec:intro}
Large-scale nonconvex learning is typically formulated as the stochastic optimization problem
\begin{equation}
    \label{eq:op_problm}
    \min_{x\in\mathbb R^d} f(x),
    \qquad
    f(x):=\mathbb E_\xi[f(x;\xi)],
\end{equation}
where \(f:\mathbb R^d\to\mathbb R\) is generally nonconvex. The goal is to find an \(\varepsilon\)-stationary point, i.e., a point \(x\) satisfying \(\E\|\nabla f(x)\|\le \varepsilon\).

Under standard smoothness and uniformly bounded stochastic noise, the complexity theory for this problem is well developed. In the deterministic setting, gradient descent achieves the optimal \(\Ocal(\varepsilon^{-2})\) iteration complexity for smooth nonconvex optimization~\cite{nesterov2004introductory,nesterov2018lectures}. In the stochastic setting, SGD achieves the classical \(\Ocal(\varepsilon^{-4})\) sample complexity under bounded variance $\mathbb E_\xi \bigl\| \nabla f(x;\xi)-\nabla f(x) \bigr\|^2 \le \sigma^2,$ and this rate is unimprovable for general smooth nonconvex stochastic optimization~\cite{ghadimi2013stochastic,arjevani2023lower}. Under stronger mean-square smoothness or average-smoothness conditions, variance-reduced methods such as SVRG, SAGA, SARAH, SPIDER, SpiderBoost, SNVRG, PAGE, and STORM improve the complexity to the optimal $\Ocal(\varepsilon^{-3})$~\cite{johnson2013accelerating,defazio2014saga,nguyen2017sarah,fang2018spider,wang2018spiderboost,zhou2020stochastic,li2021page,cutkosky2019momentum}.

These classical stochastic guarantees typically rely on a uniform bounded-noise model. While mathematically convenient, this assumption can fail even in elementary unconstrained stochastic optimization problems, such as least squares, and in modern deep learning, where the stochastic gradient variance is not uniformly constant but can grow as the iterate moves farther from a reference point. This motivates the $\mathsf{BG}$-0 noise model~\cite{blum1954approximation,gladyshev1965stochastic}, which is the weakest known viable variance assumption for theoretical analysis~\cite{alacaoglu2025towards}. Under $\mathsf{BG}$-0, the variance is allowed to grow quadratically with the distance from the initialization 
$\mathbb E_\xi \bigl\| \nabla f(x;\xi)-\nabla f(x) \bigr\|^2 \le B^2\|x-x^0\|^2+G^2.$
The classical bounded variance setting is recovered by taking $B=0$, while the deterministic case corresponds to $B=G=0$.

Although \(\mathsf{BG}\)-0 is more permissive than bounded variance, it creates an inherent feedback mechanism. If the iterates move far from $x^0$, then the oracle becomes noisier; the larger stochastic noise can then perturb future directions and push the trajectory even farther from $x^0$. Thus, over unbounded domains, the variance scale $B^2\|x^k-x^0\|^2+G^2$ cannot be treated as a harmless constant. Recently, \citet{fazla2026lower} showed that this difficulty is intrinsic: under $\mathsf{BG}$-0 noise, the lower bound for smooth nonconvex optimization worsens from the classical $\Omega(\varepsilon^{-4})$ to $\Omega(\varepsilon^{-6})$, and even under mean-square smoothness, the best possible rate is $\Omega(\varepsilon^{-4})$ rather than the bounded variance $\Omega(\varepsilon^{-3})$ rate~\cite{arjevani2023lower}. They also proposed PASTA, which matches these limits by combining Halpern-style anchoring~\cite{halpern1967fixed}, Tikhonov regularization, and dynamic batching, connecting $\mathsf{BG}$-0 stochastic approximation with classical fixed-point schemes such as Halpern and Krasnoselskii--Mann iterations~\cite{mann1953mean,krasnosel1955two,halpern1967fixed,fazla2026lower}.

In this paper, we take a complementary viewpoint. Rather than adding an explicit stabilizing anchor or increasing per-iteration batch sizes, we ask whether normalized stochastic methods with momentum or recursive variance reduction can handle $\mathsf{BG}$-0 noise without bounded domains.

\textbf{Why Normalization?} The key observation is that normalization decouples the length of the update from the magnitude of the stochastic gradient estimator. Given an estimator $v^k$, define\footnote{Throughout, we use the convention \(v/\|v\|=0\) when \(v=0\), so that the update is well defined in degenerate cases.}
\[
    d^k=\frac{v^k}{\|v^k\|},
    \qquad
    x^{k+1}=x^k-\gamma d^k.
\]
Then, $\|x^{k+1}-x^k\|=\gamma\|d^k\|\le\gamma,$ regardless of the size of $\|v^k\|$. Consequently, $$\|x^k -x^0\|\le\sum_{t=0}^{k-1}\|x^{t+1}-x^t\|\le k\gamma.$$ Thus, along a normalized trajectory, the $\mathsf{BG}$-0 variance scale satisfies $\left( B^2\|x^k-x^0\|^2+G^2 \right)^{1/2} \le G+B k\gamma.$ Normalization, therefore, transforms the potentially uncontrolled effect of large stochastic gradient magnitudes into controlled trajectory growth. In contrast, PASTA~\cite{fazla2026lower} controls the $\mathsf{BG}$-0 variance via dynamic batching. There, the batch size must scale with the upper bound on the local variance, i.e., $N_k \ \propto\ \tfrac{B^2\|x^k-x^0\|^2+G^2}{\sigma^2}$, up to problem-dependent accuracy and stepsize factors. Such a rule requires
knowledge, or at least reliable estimates, of the variance-growth constants $B$ and $G$, which may be difficult to obtain in practice.

\textbf{Generalized Smoothness.} A second challenge is that many learning objectives are not globally smooth, i.e., their local smoothness can grow with the gradient norm~\cite{jin2021non,zhang2019gradient}. This is a natural setting for normalized updates~\cite{chen2023generalized}. We therefore study $\mathsf{BG}$-0 noise under the deterministic class $\Lcal_{\rm sym}^*(\alpha)$ and its stochastic analogue $\E\Lcal_{\rm sym}^*(\alpha)$ from~\citet{chen2023generalized}, for $\alpha\in(0,1]$. The class $\Lcal^*_{\rm sym}(\alpha)$ contains standard smooth functions $(\Lcal)$, asymmetric generalized smooth functions $(\Lcal_{\rm asym}^*)$, and Hessian-based generalized smooth functions $(\Lcal_{\rm H}^*)$, and also includes high-order polynomial and exponential-type objectives~\cite{chen2023generalized,zhang2019gradient,levy2020large}. Appendix~\ref{apdx:primer_smoothness} summarizes these definitions.

Prior work shows that, under suitable bounded or relative variance assumptions, generalized smooth nonconvex optimization can be as efficient as smooth nonconvex optimization: normalized gradient methods recover the deterministic $\Ocal(\varepsilon^{-2})$ rate, while SPIDER-type variance reduction recovers the stochastic $\Ocal(\varepsilon^{-3})$ rate under expected generalized smoothness~\cite{chen2023generalized,fang2018spider}. These results, however, do not address $\mathsf{BG}$-0 noise. In this setting, the noise level is distance-dependent, while under generalized smoothness, the local curvature is gradient-dependent. This creates a new interplay between trajectory growth and gradient growth that is absent from the bounded variance setting. It is therefore unclear a priori whether the principle that generalized smooth optimization is ``as efficient as'' smooth optimization continues to hold under $\mathsf{BG}$-0 noise. This paper studies exactly this interaction.

\subsection{Contributions}
\label{subsec:contribution}
We develop a convergence theory for normalized stochastic methods under $\mathsf{BG}$-0 noise across standard and generalized smoothness regimes. The results answer four questions.
\begin{itemize}[leftmargin=*,label={}]
    \item \textbf{Q1. Can single-sample normalized momentum converge under $\mathsf{BG}$-0?}

    Yes. We prove that normalized stochastic gradient descent with momentum ($\mathsf{NSGDM}$) converges under $\mathsf{BG}$-0 using one fresh stochastic gradient sample per iteration. The method does not require bounded domains, bounded stochastic gradients, bounded variance, dynamic batching, or explicit anchoring. Under $\mathsf{BG}$-0 noise, $\mathsf{NSGDM}$ achieves an oracle complexity of $\Ocal(\varepsilon^{-6})$. Thus, normalization and momentum provide an implicit stabilization mechanism: normalization controls trajectory growth, while momentum stabilizes the noisy update direction.

    \item \textbf{Q2. Does replacing standard smoothness with generalized smoothness worsen the oracle complexity of $\mathsf{NSGDM}$ under $\mathsf{BG}$-0?}

    No. For $\mathsf{NSGDM}$, generalized smoothness is rate-neutral at the level of the $\varepsilon$-complexity exponent: under $\Lcal_{\rm sym}^*(\alpha)$-type smoothness, the stochastic first-order oracle (SFO) complexity remains $\Ocal(\varepsilon^{-6})$, up to constants depending on the generalized smoothness parameters.

    \item \textbf{Q3. Can additional stochastic regularity recover the optimal $\mathsf{BG}$-0 rate?}

    Yes. Mean-square smoothness ($\mathsf{MSS}$) provides additional stochastic regularity by controlling differences of stochastic gradients $\E_\xi \bigl\| \nabla f(y;\xi)-\nabla f(x;\xi) \bigr\|^2 \le L^2\|y-x\|^2.$ Under this stronger condition, normalized STORM ($\mathsf{NSTORM}$) improves over the $\Ocal(\varepsilon^{-6})$ single-sample momentum rate. With sharp initialization, meaning a one-time batch chosen so that the initial estimator-error contribution does not dominate the convergence rate, it achieves the optimal oracle complexity of $\Ocal(\varepsilon^{-4})$.

    \item \textbf{Q4. Does the optimal $\mathsf{NSTORM}$ rate persist under expected generalized smoothness?}

    Not exactly. Under mean-square smoothness, $\mathsf{NSTORM}$ recovers the optimal $\mathsf{BG}$-0 rate $\Ocal(\varepsilon^{-4})$. Under $\E\Lcal_{\rm sym}^*(\alpha)$ generalized smoothness, however, the estimator-difference recursion depends on gradient-dependent smoothness terms, which interact with the distance-dependent $\mathsf{BG}$-0 variance. As a result, $\mathsf{NSTORM}$ achieves $ \Ocal(\varepsilon^{-(4+\alpha)})$ for $\alpha \in (0,1)$ and $\Ocal(\varepsilon^{-5})$ for $\alpha=1$. Thus, generalized smoothness is free for $\mathsf{NSGDM}$ at the level of the $\varepsilon$-exponent, but it induces a quantifiable $\alpha$-dependent price for variance-reduced $\mathsf{NSTORM}$ under $\mathsf{BG}$-0.
\end{itemize}
We also provide synthetic experiments on generalized smooth objectives from~\citet{chen2023generalized} under a \(\mathsf{BG}\)-0 oracle, comparing normalized momentum and normalized variance reduction with dynamic batching baselines. Finally, Table~\ref{tab:related_work_compact} summarizes the comparison with prior work, emphasizing the combination of \(\mathsf{BG}\)-0 noise, normalized methods, and smoothness regimes studied here.
%%%%%%%%%%%%%%%%%%%%%%Table 1%%%%%%%%%%%%%%%%%%%%%%%%%%%%%%%
\begin{table}[t]
\centering
\caption{Comparison with representative works under different notions of smoothness $(\Lcal, \Lcal_{\rm H}^*, \Lcal_{\rm asym}^*, \Lcal_{\rm sym}^*(\alpha), \mathsf{MSS}, \E \Lcal_{\rm sym}^*(\alpha))$. Complexity is measured in oracle calls to find an $\varepsilon$-stationary point.}
\label{tab:related_work_compact}
\vskip 0.05in
\renewcommand{\arraystretch}{1.18}
\setlength{\tabcolsep}{3pt}
\footnotesize
\resizebox{\textwidth}{!}{%
\begin{tabular}{@{}
p{3.1cm}
p{3.0cm}
p{1.25cm}
p{2.90cm}
p{7.95cm}
@{}}
\toprule
\textbf{Work}
&
\textbf{Smoothness class}
&
\textbf{$\mathsf{BG}$-0?}
&
\textbf{Method}
&
\textbf{Oracle complexity}
\\
\midrule
\citet{cutkosky2020momentum}
&
$\Lcal$
&
\nomark
&
$\mathsf{NSGDM}$
&
$\Ocal(\varepsilon^{-4})$
\\
\citet{zhang2019gradient}
&
$\Lcal_{\rm H}^*$
&
\nomark
&
clipped SGD
&
$\Ocal(\varepsilon^{-4})$
\\

\citet{reisizadeh2025variance}
&
$\Lcal_{\rm asym}^*$
&
\nomark
&
SPIDER + clipping
&
$\Ocal(\varepsilon^{-3})$
\\

\citet{chen2023generalized}
&
$\Lcal_{\rm sym}^*(\alpha)$, $\E\Lcal_{\rm sym}^*(\alpha)$
&
\nomark
&
$\beta$-normalized GD; SPIDER
&
Det. ($\Lcal_{\rm sym}^*(\alpha)$): $\Ocal(\varepsilon^{-2})$; Stoch. ($\E\Lcal_{\rm sym}^*(\alpha)$): $\Ocal(\varepsilon^{-3})$
\\

\citet{khirirat2026error}
&
$\Lcal_{\rm sym}^*(1)$
&
\nomark
&
$\mathsf{NSGDM}$
&
$\Ocal(\varepsilon^{-4})$
\\

\citet{fazla2026lower}
&
$\Lcal$, $\mathsf{MSS}$
&
\yesmark
&
PASTA
&
$\Lcal$: $\Theta(\varepsilon^{-6})$; $\mathsf{MSS}$: $\Theta(\varepsilon^{-4})$
\\

\midrule
\rowcolor{black!6}
\textbf{This paper}
&
$\Lcal$, $\Lcal_{\rm sym}^*(\alpha)$
&
\yesmark
&
$\mathsf{NSGDM}$
&
$\Ocal(\varepsilon^{-6})$
\\

\rowcolor{black!6}
\textbf{This paper}
&
$\mathsf{MSS}$, $\E\Lcal_{\rm sym}^*(\alpha)$
&
\yesmark
&
$\mathsf{NSTORM}$
&
$\begin{aligned}
\mathsf{MSS}:\Ocal(\varepsilon^{-4});
\E\Lcal_{\rm sym}^*(\alpha):\Ocal(\varepsilon^{-(4+\alpha)});
\E\Lcal_{\rm sym}^*(1):\Ocal(\varepsilon^{-5})
\end{aligned}$
\\
\bottomrule
\end{tabular}%
}
\vspace{0.4em}
\begin{minipage}{0.96\textwidth}
\footnotesize
\textbf{Note.}
When $B=0$, our bounds recover the standard bounded variance rates: $\Ocal(\varepsilon^{-4})$ for normalized momentum, $\Ocal(\varepsilon^{-3})$ for variance reduction, and
$\Ocal(\varepsilon^{-2})$ in the deterministic case, for standard and generalized smoothness classes.
\end{minipage}
\vspace{-1.9em}
\end{table}
%%%%%%%%%%%%%%%%%%%%%%%%%%%%%End Table%%%%%%%%%%%%%%%%%%%%%%
\section{Related Work}
\label{sec:rel_work}
\textbf{From bounded variance to Blum--Gladyshev noise.}
Standard nonconvex guarantees usually assume uniform bounds on variance or subgradients, a convenient but often unrealistic simplification \cite{ghadimi2013stochastic, nemirovski2009robust, bottou2018optimization,moulines2011non, rakhlin2012making, shamir2013stochastic, jain2018parallelizing}. More recent research explores noise models where variance scales with the gradient norm or iterate location~\cite{gorbunov2020unified,khaled2022better,khaled2023unified,ilandarideva2023accelerated,grimmer2019convergence,alacaoglu2025towards}. A key example is the Blum--Gladyshev condition~\cite{blum1954approximation,gladyshev1965stochastic}, which allows variance to grow with the squared distance from a reference point. This is particularly difficult in nonconvex settings, as there is no natural ``pull'' to keep iterates from drifting away. While recent work like the PASTA framework \cite{fazla2026lower} uses anchoring, regularization, and dynamic batching to stabilize this drift, achieving sharp rates of $\Theta(\varepsilon^{-6})$ or $\Theta(\varepsilon^{-4})$, our approach takes a different path. We utilize normalized updates to restrict movement, combined with momentum and recursive variance reduction to keep our stochastic directions stable.

\textbf{Normalization and clipping for nonuniform smoothness.} Clipping and normalization are widely used to stabilize first-order methods when gradients or local curvature can be large. Under $(L_0, L_1)$-type smoothness, clipped methods have been analyzed in deterministic and stochastic settings~\cite{zhang2019gradient,koloskova2023revisiting}, with extensions to momentum, variance reduction, and adaptive stepsizes~\cite{zhang2020improved,reisizadeh2025variance,wang2024provable,li2023convergence,takezawa2024parameter}. Normalized methods, including normalized gradient descent, generalized SignSGD, and normalized momentum, provide a related way to control large  updates~\cite{zhao2021convergence,hubler2024parameter,crawshaw2022robustness}. \citet{khirirat2026error} study normalized error-feedback under $\Lcal^*_{\rm sym}(1)$ in distributed optimization. These results, however, rely on a stronger notion of variance bounds. In contrast, we study normalized momentum and recursive variance reduction under distance-dependent $\mathsf{BG}$-0 noise for the generalized smooth classes $\Lcal^*_{\rm sym}(\alpha)$ and $\E\Lcal^*_{\rm sym}(\alpha)$.

\textbf{Momentum and variance reduction.} Momentum and recursive variance reduction stabilize stochastic gradient estimates. Momentum smooths sample noise~\cite{cutkosky2020momentum}, while methods such as SARAH, SPIDER, SpiderBoost, PAGE, and STORM use recursive estimators to obtain sharper control~\cite{nguyen2017sarah,fang2018spider,wang2018spiderboost,li2021page,cutkosky2019momentum}. Under bounded variance and mean-square smoothness, these methods improve the stochastic nonconvex complexity from $\Ocal(\varepsilon^{-4})$ to $\Ocal(\varepsilon^{-3})$. Under $\mathsf{BG}$-0 noise, this acceleration is no longer attainable in general; the optimal $\mathsf{MSS}$ rate becomes $\Ocal(\varepsilon^{-4})$~\cite{fazla2026lower}. PASTA reaches this rate using dynamic batching with PAGE-type variance reduction. In contrast, our normalized STORM analysis recovers the same $\Ocal(\varepsilon^{-4})$ rate with single-sample-transition updates and sharp initialization. Under $\E\Lcal^*_{\rm sym}(\alpha)$, the interaction between recursive estimation and trajectory-dependent variance yields $\Ocal(\varepsilon^{-(4+\alpha)})$, with $\Ocal(\varepsilon^{-5})$ at $\alpha=1$.

\section{Preliminaries}
\label{sec:prelim}
\textbf{Notation.}
All norms $\|\cdot\|$ are Euclidean norms, and $\langle \cdot,\cdot\rangle$ denotes the Euclidean inner product. We use $\Ocal(\cdot)$ to hide numerical constants
independent of the target accuracy $\varepsilon$. For a random variable $X$, $\E[X]$ denotes expectation over all algorithmic randomness. We let $\cF_k$ be the filtration generated by the algorithm up to iteration $k$, and write $\E_k[X]:=\E[X\mid\cF_k]$. We use \(\E_\xi\) in assumptions for expectations over an oracle sample at fixed query points. Throughout, we assume that $f$ is bounded below, i.e., $f^{\inf}:=\inf_{x\in\R^d}f(x)>-\infty$, and define the initial optimality gap $\Delta:=f(x^0)-f^{\inf}$ as well as the noise at initialization as $\sigma_0^2:= \E_\xi\|\nabla f(x^0;\xi)-\nabla f(x^0)\|^2$. After $K+1$ iterations, the output $\widehat{x}$ is sampled uniformly from $\{x^0,x^1,\ldots,x^K\}$.

\textbf{Problem Formulation.} We consider the problem in \eqref{eq:op_problm}. The stochastic first-order oracle returns $\nabla f(x;\xi)$ at a queried point $x$. We measure complexity by the number of SFO calls. As mentioned previously, the goal is to find an $\varepsilon$-stationary point, i.e., a point $x$ such that $\E\|\nabla f(x)\| \le \varepsilon.$

\textbf{Assumptions.} We first introduce the noise model and then state the smoothness assumptions used in our analysis.

\begingroup
\renewcommand{\theassumption}{$\mathsf{BG}$-0}
\begin{assumption}[$\mathsf{BG}$-0 stochastic oracle]
    \label{as:bg0_oracle}
    For every $x\in\R^d$, the stochastic gradient oracle is unbiased, i.e., $\E_\xi[\nabla f(x;\xi)]=\nabla f(x).$ Moreover, there exist constants $B,G\ge0$ such that $$\E_\xi \left\| \nabla f(x;\xi)-\nabla f(x) \right\|^2 \le B^2\|x-x^0\|^2+G^2.$$
\end{assumption}
\endgroup
When $B=0$, Assumption~\ref{as:bg0_oracle} reduces to the classical bounded variance condition. When $B=G=0$, the oracle is deterministic.
% We define the local variance scale $q(x):=\left(B^2\|x-x^0\|^2+G^2\right)^{1/2}$, $q_k:=q(x^k).$

We pair this noise model with several smoothness conditions of increasing generality. The first two are standard: $L_0$-smoothness controls population gradients~\cite{ghadimi2013stochastic}, while mean-square smoothness controls the stochastic gradient differences~\cite{cutkosky2019momentum,li2021page}.
\begin{assumption}[$L_0$-smoothness of $f$]
\label{as:l0_smooth}
There exists $L_0>0$ such that for all $x,y\in\R^d$, $$\|\nabla f(y)-\nabla f(x)\| \le L_0\|y-x\|.$$
\end{assumption}
\begin{assumption}[Mean-square smoothness]
\label{as:mss}
There exists $L>0$ such that for all $x,y\in\R^d$, $$\E_\xi \left\| \nabla f(y;\xi)-\nabla f(x;\xi) \right\|^2 \le L^2\|y-x\|^2.$$
\end{assumption}
The next two assumptions extend smoothness to the generalized setting of~\citet{chen2023generalized}, where the local smoothness parameter is allowed to grow with the gradient norm. The deterministic class $\Lcal_{\rm sym}^*(\alpha)$ governs the normalized momentum analysis, while the expected class $\E\Lcal_{\rm sym}^*(\alpha)$ governs the normalized STORM analysis.
\begin{assumption}[$\Lcal_{\rm sym}^*(\alpha)$-type generalized smoothness]
\label{as:lsym_alpha}
Let $\alpha\in(0,1]$. We say that $f\in\Lcal_{\rm sym}^*(\alpha)$ if $f:\R^d\to\R$ is differentiable and there exist constants $L_0, L_1>0$ such that, for all $x,y\in\R^d$,

\[
\|\nabla f(y)-\nabla f(x)\| \le \left( L_0 + L_1 \max_{\theta\in[0,1]} \|\nabla f(x_\theta)\|^\alpha \right) \|y-x\|, \text{ where }  x_\theta:=\theta y+(1-\theta)x.
\]
\end{assumption}
The generalized-smoothness framework accommodates objectives that are not globally smooth. For example, \(f(x)=|x|^3\) does not have a globally Lipschitz gradient but satisfies \(\mathcal{L}_{\mathrm{sym}}^*(1/2)\), illustrating the broader class of objectives covered in our analysis. Moreover, for twice-differentiable functions, the endpoint \(\mathcal{L}_{\mathrm{sym}}^*(1)\) is equivalent, up to constants, to the Hessian-growth condition \(\|\nabla^2 f(x)\|\leq L_0+L_1\|\nabla f(x)\|\)~\cite{zhang2019gradient,chen2023generalized}.
\begin{assumption}[Expected $\alpha$-symmetric generalized smoothness]
\label{as:elsym_alpha}
Let $\alpha\in(0,1]$. We say that $f \in \E\Lcal_{\rm sym}^*(\alpha)$ if there exist constants $L_0, L_1>0$ such that, for all $x,y\in\R^d$, where $x_\theta:=\theta y+(1-\theta)x$, we have

\[
\E_\xi \big\| \nabla f(y;\xi)-\nabla f(x;\xi) \big\|^2 \le \|y-x\|^2 \E_\xi \big[ \big( L_0 + L_1 \max_{\theta\in[0,1]} \|\nabla f(x_\theta;\xi)\|^\alpha \big)^2 \big].
\]

\end{assumption}
Formally, setting $L_1=0$ in the deterministic and expected generalized smoothness conditions recovers $L_0$-smoothness and mean-square smoothness, respectively. At $\alpha=1$, the conditions correspond to symmetric $(L_0,L_1)$-generalized smoothness and its expected stochastic counterpart. The latter controls same-sample stochastic-gradient differences, as required by the $\mathsf{NSTORM}$ recursion. The expected condition is more delicate under $\mathsf{BG}$-0 noise because it involves sample-gradient moments, which must be controlled through the distance-dependent variance scale. This interaction is the source of the $\alpha$-dependent price in our variance-reduced rates. Existing empirical studies motivate population-level generalized smoothness along neural-network training trajectories~\cite{zhang2019gradient}, while direct large-scale verification of the stronger expected stochastic condition remains limited.

\section{Normalized Momentum under \texorpdfstring{$\mathsf{BG}$}{BG}-0}
\label{sec:NSGDM}
We first consider normalized stochastic gradient descent with momentum, denoted by $\mathsf{NSGDM}$. Given $x^0\in\R^d$, stepsize $\gamma>0$, momentum parameter $\eta\in(0,1]$, and horizon $K\ge0$, initialize $v^0=\nabla f(x^0;\xi^0).$ For $k=0,1,\ldots,K$, $\mathsf{NSGDM}$ performs
\begin{equation}
\label{eq:nsgdm_updates}
\mathsf{NSGDM:}\qquad
x^{k+1} = x^k-\gamma\frac{v^k}{\|v^k\|};
\qquad
v^{k+1} = (1-\eta)v^k+\eta\nabla f(x^{k+1};\xi^{k+1}),
\end{equation}
where the $v^{k+1}$ update is used only for $k=0,1,\ldots,K-1$. The samples are fresh and conditionally independent given the past, and the stochastic oracle satisfies Assumption~\ref{as:bg0_oracle}. 

%%%%%%%%%%%%%%%%%%%%%%%%%Theorem: NSGDM%%%%%%%%%%%%%%%%%%%%%%%

\begin{theorem}[Normalized momentum under $\mathsf{BG}$-0]
\label{thm:nsgdm_three_regimes}
Consider problem~\eqref{eq:op_problm} and the
$\mathsf{NSGDM}$ update in~\eqref{eq:nsgdm_updates}.
Suppose that $f$ is bounded below and that Assumption~\ref{as:bg0_oracle} holds. Let $\eta=(K+1)^{-2/3}$, $\gamma=\gamma_0(K+1)^{-5/6}$, where $\gamma_0>0$. Then the output $\widehat{x}$, sampled uniformly from $\{x^0,x^1,\ldots,x^K\}$, satisfies the following.

\begin{enumerate}[label=(\roman*),leftmargin=*,nosep]
\item \textbf{$L_0$-smooth case.}
Suppose Assumption~\ref{as:l0_smooth} holds. Then, for every $\gamma_0>0$,
\begin{equation*}
    \E\|\nabla f(\widehat{x})\| \le \left( \frac{2\Delta}{\gamma_0} +(18L_0+2B)\gamma_0 \right)(K+1)^{-1/6} +\Ocal\left( (\sigma_0+G)(K+1)^{-1/3} \right).
\end{equation*}

\item \textbf{$\Lcal_{\rm sym}^*(\alpha)$ generalized smooth case, $\alpha\in(0,1)$.}
Suppose Assumption~\ref{as:lsym_alpha} holds for some fixed $\alpha\in(0,1)$. Then, for every $0<\gamma_0\le1$,
\begin{equation*}
    \E\|\nabla f(\widehat{x})\| \le \left( \frac{2\Delta}{\gamma_0} +(\Lambda_\alpha+2B)\gamma_0 \right)(K+1)^{-1/6} + \Ocal\left( (\sigma_0+G)(K+1)^{-1/3} \right),
\end{equation*}
where $\Lambda_\alpha<\infty$ depends only on
$L_0,L_1$, and $\alpha$, with its explicit
definition given in
Appendix~\ref{apdx:nsgdm_proofs}.

\item \textbf{$\Lcal_{\rm sym}^*(1)$ generalized smooth case.}
Suppose Assumption~\ref{as:lsym_alpha} holds with $\alpha=1$. Then, for every $0<\gamma_0\le 1/(8L_1)$,
\begin{equation*}
    \E\|\nabla f(\widehat{x})\| \le \left( \frac{2\Delta}{\gamma_0} +32L_1\Delta +(36L_0+4B)\gamma_0 \right)(K+1)^{-1/6} + \Ocal\left( (\sigma_0+G)(K+1)^{-1/3} \right).
\end{equation*}
\end{enumerate}
Consequently, all three smoothness regimes satisfy $\E\|\nabla f(\widehat{x})\| =\Ocal((K+1)^{-1/6}),$
giving $\Ocal(\varepsilon^{-6})$ SFO complexity for $\mathsf{NSGDM}$.
\end{theorem}

\textbf{Discussion on Theorem~\ref{thm:nsgdm_three_regimes}.}
We highlight several aspects of this result; the full proof is in Appendix~\ref{apdx:nsgdm_proofs}.

\noindent \textbullet\  \textit{Generalized smoothness is rate-neutral for $\mathsf{NSGDM}$.}
Under $\mathsf{BG}$-0 noise, the oracle complexity exponent $\Ocal(\varepsilon^{-6})$ is identical across all three regimes: standard $L_0$-smoothness, $\Lcal_{\rm sym}^*(1)$ generalized smoothness, and $\Lcal_{\rm sym}^*(\alpha)$ generalized smoothness for any fixed $\alpha\in(0,1)$. The generalized smoothness parameters affect only the leading constants, not the $\varepsilon$-exponent. This is consistent with the principle articulated by~\citet{chen2023generalized} for the deterministic setting, that generalized smooth nonconvex optimization is ``as efficient as'' standard smooth nonconvex optimization, and extends it to the stochastic momentum regime under $\mathsf{BG}$-0 noise.

\noindent \textbullet\  \textit{Novelty of the $\alpha\in(0,1)$ analysis.}
To our knowledge, the $\Lcal_{\rm sym}^*(\alpha)$ analysis for $\mathsf{NSGDM}$ with $\alpha\in(0,1)$ is new even under bounded variance. Prior stochastic results under $\alpha$-symmetric generalized
smoothness are limited to the two-loop SPIDER estimator under $\E\Lcal_{\rm sym}^* (\alpha)$~\cite{chen2023generalized}, which controls variance through periodic batch corrections. The single-sample momentum recursion requires us to control the estimator error over time. After unrolling the recursion, this error separates into gradient-drift terms and stochastic noise terms. The key technical challenge is that the drift bound under $\Lcal_{\rm sym}^*(\alpha)$ yields terms proportional to $\gamma\|\nabla f(x^k)\|^\alpha$, which, unlike the $\alpha=1$ case cannot be converted into suboptimality gaps $(f(x^k)-f^{\inf})$ via the gradient-growth
inequality~\cite{chen2023generalized,khirirat2026error}: when $\alpha=1$, the bound $$\|\nabla f(x)\|\le 8L_1(f(x)-f^{\inf})+L_0/L_1$$ directly absorbs $\|\nabla f(x^k)\|$ into $(f(x^k)-f^{\inf})$, but for $\alpha\in(0,1)$ the fractional power breaks this conversion. Instead, we need to utilize Young's inequality, giving us a free parameter to absorb the $\|\nabla f(x^k)\|^\alpha$ contribution into a fraction of the $\|\nabla f(x^k)\|$ descent budget, at the cost of producing remainder terms that scale as $\gamma^{1/(1-\alpha)}$. These remainders vanish faster than $\gamma$ and do not affect the dominant rate.

\noindent \textbullet\  \textit{Recovery of known rates.}
When $B=0$, the $\mathsf{BG}$-0 condition reduces to bounded variance. In this case, choosing $\eta=(K+1)^{-1/2}$ and $\gamma=\gamma_0(K+1)^{-3/4}$ recovers the $\Ocal(\varepsilon^{-4})$ rate of~\citet{cutkosky2020momentum} for standard smoothness and~\citet{khirirat2026error} for $\Lcal_{\rm sym}^*(1)$, and establishes the same rate for $\Lcal_{\rm sym}^*(\alpha)$ with $\alpha\in(0,1)$. Additionally when $G=0$ as well, the oracle is deterministic, and with $\eta=1$ recovers the $\Ocal(\varepsilon^{-2})$ rate of~\citet{chen2023generalized} for normalized GD.

\noindent \textbullet\  \textit{Schedule comparison.} Under bounded variance, normalized momentum uses the standard schedule $\eta=(K+1)^{-1/2}$ and $\gamma=\gamma_0(K+1)^{-3/4},$ leading to the rate $(K+1)^{-1/4}$, or equivalently $\Ocal(\varepsilon^{-4})$ SFO complexity~\cite{cutkosky2020momentum,khirirat2026error}. The faster decay of $\gamma=\gamma_0(K+1)^{-5/6}$ under $\mathsf{BG}$-0 reflects the growing noise scale $G+B k\gamma$. Likewise, $\eta=(K+1)^{-2/3}$ is chosen smaller to control the stochastic error term $\sqrt{\eta}(G+BK\gamma)$. This changes the rate from the bounded variance rate $(K+1)^{-1/4}$ to the $\mathsf{BG}$-0 rate $(K+1)^{-1/6}$, matching the lower-bound degradation from $\Omega(\varepsilon^{-4})$ to $\Omega(\varepsilon^{-6})$~\cite{fazla2026lower}. Additionally, when $\eta=1$, meaning momentum is absent, the temporal averaging due to it disappears, and the resulting noise terms prevent our analysis from yielding a vanishing stationarity guarantee under general $\mathsf{BG}$-0 noise. More precisely, under the BG-0 schedule, the trajectory-dependent noise contributes at order \(B\gamma_0T^{-1/6}\), whereas the constant-noise contribution decays faster, at order \(GT^{-1/3}\). Thus, the trajectory-dependent noise is co-leading with the optimization and momentum-tracking terms, while the direct constant-noise contribution is lower order.

%%%%%%%%%%%%%%%%%%%%%Theorem: NSGDM End%%%%%%%%%%%%%%%%%%%%%%

\section{Normalized STORM under \texorpdfstring{$\mathsf{BG}$}{BG}-0}
\label{sec:NSTORM}
We next study the normalized STORM recursion under $\mathsf{BG}$-0 noise. Unlike $\mathsf{NSGDM}$, which averages stochastic gradients directly, $\mathsf{NSTORM}$ uses a recursive estimator based on the difference of two stochastic gradients evaluated with the same fresh sample. Controlling these same-sample gradient differences requires sample-level regularity, provided here by mean-square smoothness or expected generalized smoothness, whereas $\mathsf{NSGDM}$ relies only on population-level smoothness.

Fix $x^0\in\R^d$, a stepsize $\gamma>0$, a STORM parameter $\eta\in(0,1]$, and a horizon $K\ge0$. Let $N_{\rm init}\ge1$ be an initialization batch size, and initialize $v^0 = \frac1{N_{\rm init}} \sum_{i=1}^{N_{\rm init}} \nabla f(x^0;\xi_i^{\rm init}),$ where the samples $\{\xi_i^{\rm init}\}_{i=1}^{N_{\rm init}}$ are independent. For $k=0,1,\ldots,K$, $\mathsf{NSTORM}$ performs
\begin{equation}
\label{eq:nstorm_updates}
\mathsf{NSTORM:}\quad
x^{k+1} = x^k-\gamma\frac{v^k}{\|v^k\|};
\ \ 
v^{k+1} = \nabla f(x^{k+1};\xi^{k+1}) + (1-\eta) \left(v^k-\nabla f(x^k;\xi^{k+1})\right).
\end{equation}
The sample $\xi^{k+1}$ is fresh and conditionally independent given the past, and is evaluated at both $x^k$ and $x^{k+1}$. Thus, each transition uses two stochastic gradient evaluations. The initialization batch only controls the initial error $\E\|v^0-\nabla f(x^0)\|$, and $N_{\rm init}$ is specified for the smoothness regime.

\noindent\textbf{NSTORM parameter schedules.}
Let $T:=K+1$. We use the following parameter schedules in Theorem~\ref{thm:nstorm_three_regimes}.
\begin{table}[htbp]
\centering
\begin{tabular}{lcccc}
\toprule
\textbf{Setting} & $\eta$ & $\gamma$ & $N_{\mathrm{init}}$ & $\lambda$ \\
\midrule
$\mathsf{MSS}$ & $T^{-1}$ & $\gamma_0 T^{-3/4}$ & $\lceil T^{1/2} \rceil$ & -- \\
\addlinespace
$\mathbb{E}\mathcal{L}_{\mathrm{sym}}^*(\alpha)$, $\alpha \in (0,1)$ & $\eta_0 T^{-\frac{4}{4+\alpha}}$ & $\gamma_0 T^{-\frac{3+\alpha}{4+\alpha}}$ & $\lceil T^{\frac{2(1-\alpha)}{4+\alpha}} \rceil$ & $\lambda_0 T^{-\frac{1-\alpha}{4+\alpha}}$ \\
\addlinespace
$\mathbb{E}\mathcal{L}_{\mathrm{sym}}^*(1)$ & $T^{-4/5}$ & $\gamma_0 T^{-4/5}$ & $1$ & -- \\
\bottomrule
\end{tabular}
\end{table}

Here $\eta_0\in(0,1]$, while $\gamma_0>0$ and $\lambda_0>0$ are fixed constants satisfying the smoothness-dependent restrictions below. These schedules do not require knowledge of the $\mathsf{BG}$-0 parameters $B$ or $G$. The parameters $\eta$, $\gamma$, and $N_{\rm init}$ are algorithmic schedules, whereas $\lambda$ is used only in the analysis.
% All subsequent transitions remain single-sample in the sense of~\eqref{eq:nstorm_updates}. 
\begin{theorem}[$\mathsf{NSTORM}$ under $\mathsf{BG}$-0]
\label{thm:nstorm_three_regimes}
Consider problem~\eqref{eq:op_problm} and the $\mathsf{NSTORM}$ update in~\eqref{eq:nstorm_updates}.
Suppose that $f$ is bounded below and Assumption~\ref{as:bg0_oracle} holds. Let $T := K + 1$ and run $\mathsf{NSTORM}$ with the corresponding parameter schedule above. Then the output $\widehat{x}$, sampled uniformly from $\{x^0,x^1,\ldots,x^K\}$, satisfies the following.
\begin{enumerate}[label=(\roman*),leftmargin=*,nosep]
\item \textbf{Mean-square smoothness.}
Under Assumption~\ref{as:mss} with constant $L$, for any $\gamma_0 > 0$,
\[
\mathbb{E}\|\nabla f(\widehat{x})\|
\le \left(\frac{\Delta}{\gamma_0} + 2\sigma_0 + \frac{5}{2}L\gamma_0 + 2B\gamma_0\right) T^{-1/4} + 2G\,T^{-1/2}.
\]
Consequently, $\mathbb{E}\|\nabla f(\widehat{x})\| = \mathcal{O}(T^{-1/4})$, achieving the minimax optimal $\mathcal{O}(\varepsilon^{-4})$ SFO complexity under $\mathsf{BG}$-0 and mean-square smoothness.

\item $\mathbb{E}\mathcal{L}_{\mathrm{sym}}^*(\alpha)$\textbf{ generalized smooth case with} $\alpha \in (0,1)$.
Under Assumption~\ref{as:elsym_alpha} for fixed $\alpha \in (0,1)$, define $K_1 := 2^{\frac{2-\alpha}{1-\alpha}}L_1$. Let $\eta_0 \in (0,1]$ and $\gamma_0, \lambda_0 > 0$ satisfy $2^{\alpha/2}K_1\lambda_0\gamma_0 \le 1$, $2^{\frac{6+\alpha}{2}}K_1\lambda_0\frac{\gamma_0}{\eta_0} \le 1$. Then $$\mathbb{E}\|\nabla f(\widehat{x})\| \le \left(\frac{4\Delta}{\gamma_0} + \frac{8\sigma_0}{\eta_0} + 8\sqrt{\eta_0}(G + B\gamma_0) + \frac{16\gamma_0}{\sqrt{\eta_0}}\Phi_\alpha\right) T^{-\frac{1}{4+\alpha}},$$ where $\Phi_\alpha<\infty$ is independent of $T$ and collects the remaining dependence on $L_0$, $L_1$, $\alpha$, $\gamma_0$, $\lambda_0$, $B^\alpha$, and $G^\alpha$. Its explicit definition is provided in Appendix~\ref{app:nstorm_theorem_proof}. 

Consequently, $\E\|\nabla f(\widehat{x})\| =\Ocal(T^{-1/(4+\alpha)})$, giving an $\Ocal(\varepsilon^{-(4+\alpha)})$ SFO complexity.

\item $\mathbb{E}\mathcal{L}_{\mathrm{sym}}^*(1)$\textbf{ generalized smooth case}.
Under Assumption~\ref{as:elsym_alpha} with $\alpha = 1$, for any $0 < \gamma_0 \le \big(16\sqrt{2e^{3/4}}\,L_1\big)^{-1}$,
\[
\mathbb{E}\|\nabla f(\widehat{x})\|
\le \left(\frac{4\Delta}{\gamma_0} + 8\sigma_0 + 10B\gamma_0\right) T^{-1/5} + (23L_0\gamma_0 + 10G) T^{-2/5}.
\]
Consequently, $\mathbb{E}\|\nabla f(\widehat{x})\| = \mathcal{O}(T^{-1/5})$, giving an $\mathcal{O}(\varepsilon^{-5})$ SFO complexity.
\end{enumerate}
\end{theorem}

\textbf{Discussion on Theorem~\ref{thm:nstorm_three_regimes}.} We now discuss several aspects of Theorem~\ref{thm:nstorm_three_regimes}. The complete details of the proof, including the proofs on bounded variance and deterministic cases, are provided in Appendix~\ref{apdx:nstorm_proofs}.

\noindent \textbullet\  \textit{Optimality under mean-square smoothness.}
Under Assumption~\ref{as:mss}, i.e., $\mathsf{MSS}$ and sharp initialization, $\mathsf{NSTORM}$ achieves $\Ocal(\varepsilon^{-4})$ SFO complexity which is minimax optimal in this setting~\cite{fazla2026lower}. The key structural reason is that $\mathsf{MSS}$ gives $\E_\xi\|\nabla f(y;\xi)-\nabla f(x;\xi)\|^2 \le L^2\|y-x\|^2,$ with no dependence on gradient magnitudes. Hence, the STORM centered-noise term satisfies $$(\E[\|Z_{k+1}\|^2\big|\cF_k])^{\tfrac12}\le L\gamma\text{, with } Z_{k+1} = \nabla f(x^{k+1};\xi^{k+1}) - \nabla f(x^k;\xi^{k+1}) - (\nabla f(x^{k+1}) - \nabla f(x^k)).$$ So, the estimator recursion is decoupled from the gradient sequence. Furthermore, sharp initialization removes the initial-error bottleneck $b_0/(\eta T)$, allowing the schedule $\eta=(K+1)^{-1}$ and $\gamma=\gamma_0(K+1)^{-3/4}$ and yielding the rate $(K+1)^{-1/4}$.

\noindent \textbullet\  \textit{Gradient feedback and the $\alpha$-dependent price under expected generalized smoothness.}
Under $\E\Lcal_{\rm sym}^*(\alpha)$, the STORM centered-difference bound on $Z_{k+1}$ depends on the sample-gradient moment $\E_\xi\|\nabla f(x^k;\xi)\|^\alpha$ via the expected reduction of~\citet{chen2023generalized} (see Proposition 4, item (1)--(2)). Under bounded variance, this moment is controlled by a fixed noise term, but under $\mathsf{BG}$-0 it couples a gradient-dependent term with trajectory-dependent variance, the main reason the analysis of~\citet{chen2023generalized} cannot be applied directly. To handle this coupling, we use Young's inequality to linearize the fractional-gradient term: $$\|\nabla f(x^k)\|^\alpha \le \lambda\|\nabla f(x^k)\| + c_\alpha\lambda^{-\alpha/(1-\alpha)},$$ yielding an estimator bound of the form $$(\E[\|Z_{k+1}\|^2\mid\cF_k])^{1/2} \le \gamma(a+h\|\nabla f(x^k)\|)$$ with $h \propto \lambda$. The gradient-dependent part creates a feedback loop: large gradients increase the estimator error, which in turn weakens the descent guarantee. Our proof separates the estimator error into a gradient-independent part, controlled by the $\mathsf{NSTORM}$ estimator recursion, and the gradient-dependent part, which is absorbed into the descent inequality via the stepsize conditions. Choosing $\lambda$ to decay with the horizon weakens this feedback and leads to the $\Ocal(\varepsilon^{-(4+\alpha)})$ rate for $\alpha\in(0,1)$. At $\alpha=1$, the term $\|\nabla f(x^k)\|^\alpha = \|\nabla f(x^k)\|$ is already linear, so there is no fractional power to absorb via Young's inequality; the feedback coefficient $h$ is then a fixed constant, yielding the slower $\Ocal(\varepsilon^{-5})$ rate.

\noindent \textbullet\  \textit{Dominant stochastic contributions.} 
Under mean-square smoothness, the leading error contributions include the initialization noise level \(\sigma_0\) and the trajectory-dependent noise coefficient \(B\), both appearing at order \(T^{-1/4}\), whereas the direct \(G\)-dependent contribution decays as \(T^{-1/2}\). Under expected $\alpha$-symmetric generalized smoothness, the coupling between distance-dependent variance and gradient-dependent smoothness introduces an additional leading contribution involving \(L_1B^\alpha\). Hence, several terms can contribute at the same leading order, rather than there being a unique dominant term. The corresponding constant-noise contributions decay faster under the $\mathsf{BG}$-0 schedules.

\noindent \textbullet\ \textit{Sharp initialization and its dependence on the smoothness regime.} The initialization batch size $N_{\rm init}$ is determined by a single mechanism: whether the schedule $\eta$ suppresses the initial-error contribution $b_0/(\eta T)$ at the target convergence rate. Under $\mathsf{MSS}$, the clean STORM recursion permits $\eta = T^{-1}$, which yields $b_0/(\eta T) = \Ocal(b_0)$, a $T$-independent term, forcing $N_{\rm init} = \Ocal(T^{\nicefrac{1}{2}})$ to suppress $b_0$ to $\Ocal(T^{\nicefrac{-1}{4}})$. Under $\E\Lcal_{\rm sym}^*(\alpha)$ with $\alpha \in (0,1)$, the gradient-dependent feedback requires $\eta = \eta_0 T^{\nicefrac{-4}{(4+\alpha)}}$, giving ${b_0/(\eta T)} = \Ocal(b_0 T^{\nicefrac{-\alpha}{(4+\alpha)}})$, which decays more slowly than the target rate $T^{\nicefrac{-1}{(4+\alpha)}}$. This necessitates $N_{\rm init} = \Ocal(T^{\nicefrac{2(1-\alpha)}{(4+\alpha)}})$, whose exponent decreases to zero as $\alpha \to 1^-$ and grows toward $1/2$ as $\alpha \to 0^+$. At $\alpha = 1$, the schedule $\eta = T^{-4/5}$ makes $b_0/(\eta T) = \Ocal(T^{\nicefrac{-1}{5}})$ match the target rate, so one-sample initialization is sufficient. Thus, sharp initialization is needed whenever the schedule is aggressive relative to the target rate.

\noindent \textbullet\  \textit{Comparison with prior variance-reduced methods.}
\citet{chen2023generalized} prove that SPIDER achieves $\Ocal(\varepsilon^{-3})$ under $\E\Lcal_{\rm sym}^*(\alpha)$ in the bounded variance setting. \citet{reisizadeh2025variance} obtain the same rate under asymmetric generalized smoothness using clipped SPIDER. Our setting differs in two ways: the estimator is single-loop STORM rather than two-loop SPIDER, and the noise is distance-dependent rather than bounded. When $B=0$, the $\mathsf{BG}$-0 trajectory-growth terms disappear, and the bounded variance specialization recovers the standard $\Ocal(\varepsilon^{-3})$ rate in our setting as well. 

\section{Experiments}
\label{sec:exp}
We evaluate the methods analyzed in this paper on two experiments whose deterministic objectives satisfy $\alpha$-symmetric generalized smoothness. In both cases, we use objectives from~\citet{chen2023generalized} and apply a stochastic $\mathsf{BG}$-0 wrapper that injects distance-dependent noise at each query point. The
wrapper constructs a proper stochastic loss whose oracle is unbiased, satisfies the $\mathsf{BG}$-0 variance model exactly, and preserves the $\E\Lcal_{\rm sym}^*(\alpha)$ condition (see Appendix~\ref{app:exp_bg0_oracle}).

Across both experiments, we compare $\mathsf{NSGDM}$ and $\mathsf{NSTORM}$ using a single fresh sample per iteration against $\mathsf{SGD}$ and $\mathsf{STORM}$ with dynamic batching following~\citet{fazla2026lower}. The normalized methods use horizon-dependent schedules motivated by Theorems~\ref{thm:nsgdm_three_regimes} and~\ref{thm:nstorm_three_regimes}, with experimental configurations detailed in Appendix~\ref{app:exp_hyperparams}. Results are averaged over three random seeds and reported as mean $\pm$ standard deviation.

\textbullet \ \textbf{Phase retrieval} (Figure~\ref{fig:phase_retrieval}).
We minimize $$f(x)=\frac{1}{2m}\sum_{r=1}^m(y_r-|a_r^\top x|^2)^2,$$ with $m=3000$ Gaussian measurements $a_r\in\R^{100}$ ($a_{r,j}\sim\mathcal N(0,0.01)$), target signal $x_{\star,j}\sim\mathcal N(0,1)$, and noiseless observations $y_r=|a_r^\top x_\star|^2$; the objective satisfies $\Lcal_{\rm sym}^* (\tfrac{2}{3})$~\cite{chen2023generalized}. We initialize at $x^0_j\sim\mathcal N(5,1)$, far from the target, so that
$\mathsf{BG}$-0 variance growth is visible. $\mathsf{NSTORM}$ reaches a lower gradient norm than $\mathsf{NSGDM}$, consistent with the improved exponent of variance reduction, while the dynamic batching baselines require batch sizes growing to $\sim\!10^3$ (right panel). The middle panel confirms that normalization controls trajectory growth: $\|x^k\!-\!x^0\|^2$ stabilizes for both normalized methods, while single-sample $\mathsf{SGD}$ drifts until the dynamic batch size compensates.
\begin{figure}[t]
    \centering
    \includegraphics[width=\linewidth]{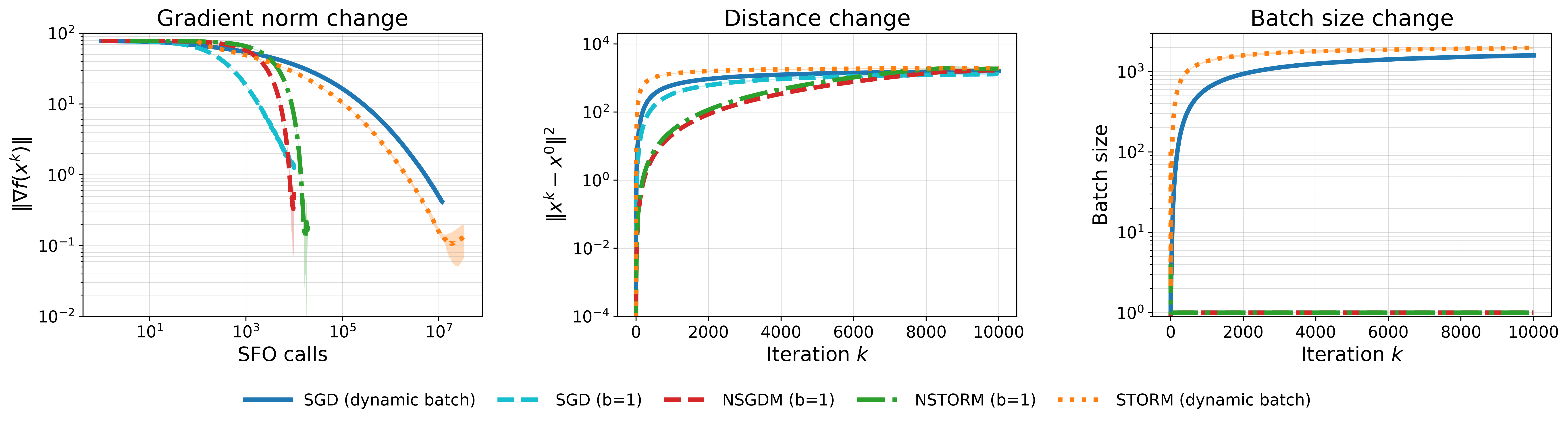}
    \caption{\emph{Phase retrieval experiment under the $\mathsf{BG}$-0
    oracle ($\alpha\!=\!2/3$).} Gradient norm vs.\ SFO calls (left),
    drift $\|x^k\!-\!x^0\|^2$ (center), and batch size (right).}
    \label{fig:phase_retrieval}
\end{figure}

\begin{figure}[t]
    \centering
    \includegraphics[width=\linewidth]{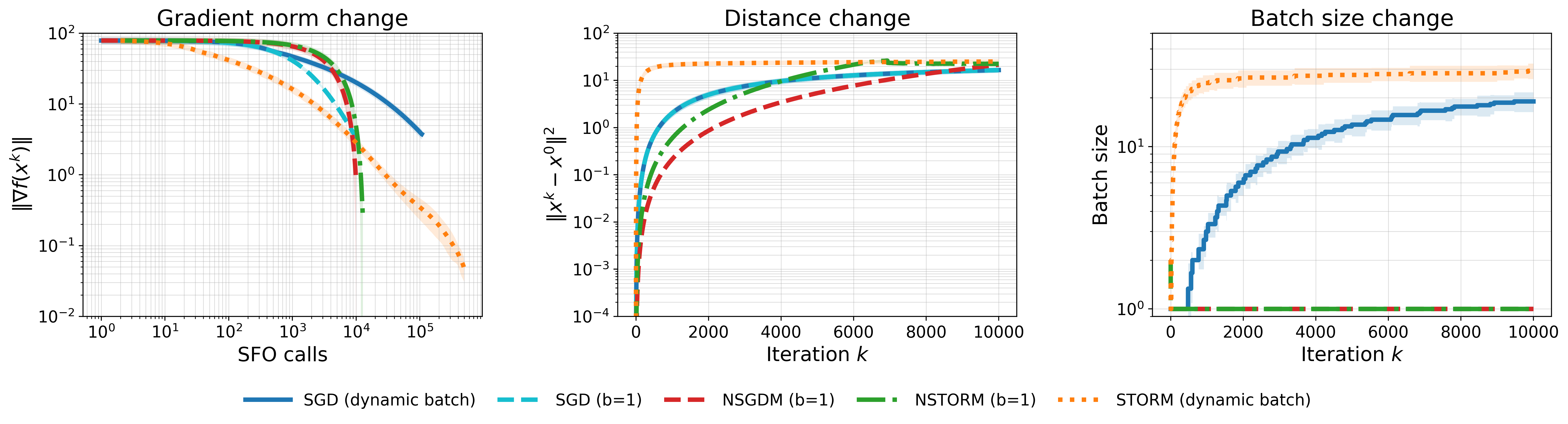}
    \caption{\emph{Cubic polynomial under the
    $\mathsf{BG}$-0 oracle ($\alpha\!=\!1/2$).} Same layout as
    Figure~\ref{fig:phase_retrieval}. $\mathsf{NSTORM}$ reaches a
    lower gradient norm than $\mathsf{NSGDM}$, reflecting the benefit
    of variance reduction.}
    \label{fig:poly_separation}
\end{figure}
\noindent \textbullet\ \textbf{Cubic polynomial} (Figure~\ref{fig:poly_separation}).
In this experiment, we minimize $$f(x)=|x|^3$$ on $\R$, which satisfies $\Lcal_{\rm sym}^*(\tfrac{1}{2})$~\cite{chen2023generalized}, and initialize away from the target $x_{\star}=0$ at $x^0\sim\mathcal N(5,0.1)$. $\mathsf{NSTORM}$ again reaches a lower gradient norm than $\mathsf{NSGDM}$, reflecting the benefit of variance reduction. The dynamic batching baselines require batch sizes that grow by over an order of magnitude (right panel), whereas the normalized methods maintain batch size one throughout, except at initialization for $\mathsf{NSTORM}$.

Across both experiments, normalized methods and dynamic batching offer two complementary mechanisms for handling $\mathsf{BG}$-0 noise on generalized smooth landscapes: normalization controls trajectory growth through bounded step norms, while dynamic batching controls the effective noise by increasing oracle calls as the distance from initialization grows. 

\noindent\textbf{Discussion: Memory and parallelization trade-offs.} A practical advantage of our methods is that their parameter schedules do not require knowledge of the $\mathsf{BG}$-0 noise constants $B$ and $G$. In contrast, dynamic batching of PASTA~\cite{fazla2026lower} typically selects $N_k\propto (B^2\|x^k-x^0\|^2+G^2)/\sigma^2$, which requires estimates of these constants. Moreover, $\mathsf{NSGDM}$ uses one fresh sample per iteration, while $\mathsf{NSTORM}$ uses one fresh sample per transition, evaluated at two points, apart from its one-time (sharp) initialization batch. Thus, the per-transition sample workload remains constant, whereas dynamic batch sizes may grow with the distance from initialization, reaching approximately $10^3$ in our experiments (Figures~\ref{fig:phase_retrieval} and~\ref{fig:poly_separation}). 

This advantage comes with a trade-off in parallelization. Dynamic batches can be evaluated concurrently, requiring approximately $\lceil N_k/P\rceil$ rounds of gradient evaluations on $P$ workers, ignoring communication overhead. In contrast, successive normalized updates are sequentially dependent, so lower SFO complexity does not necessarily translate into lower wall-clock time. A fixed mini-batch of size $b$ provides a possible compromise: under independent sampling, the stochastic noise standard-deviation terms decrease by a factor of $b^{1/2}$, while the convergence-rate exponent remains unchanged for fixed $b$. Unlike dynamic batching, this choice does not require adapting $b$ to the iterate distance or estimating $B$ and $G$, although the SFO cost per iteration increases proportionally to $b$.

\section{Conclusion}
\label{sec:conclusion}
We developed a convergence theory for normalized stochastic methods under $\mathsf{BG}$-0 noise across standard and generalized smoothness regimes. For $\mathsf{NSGDM}$, we showed that normalization and momentum provide implicit stabilization against distance-dependent noise, achieving $\Ocal(\varepsilon^{-6})$ with a single sample per iteration and no anchoring or dynamic batching, which is minimax optimal~\cite{fazla2026lower}. The same rate holds under generalized smoothness. For $\mathsf{NSTORM}$, variance reduction recovers the optimal $\Ocal(\varepsilon^{-4})$ rate under mean-square smoothness with sharp initialization~\cite{fazla2026lower}, while expected $\alpha$-symmetric generalized smoothness yields $\Ocal(\varepsilon^{-(4+\alpha)})$ for $\alpha\in(0,1)$ and
$\Ocal(\varepsilon^{-5})$ for $\alpha=1$.

The $\alpha$-dependent price in $\mathsf{NSTORM}$ raises a natural question: is it information-theoretically optimal, or is there a better rate? The price seems to originate from a coupling intrinsic to the problem structure. Under $\E\Lcal^*_{\rm sym}(\alpha)$ with $\mathsf{BG}$-0, stochastic gradient differences depend simultaneously on $\|\nabla f(x)\|^\alpha$ from generalized smoothness and on the distance-dependent variance. This coupling vanishes when either source of difficulty is removed. We conjecture that it would confront any variance-reduced method that bounds stochastic gradient differences under $\E\Lcal_{\rm sym}^*(\alpha)$. Thus, it remains an interesting open avenue to determine whether one can come up with a matching lower bound, or an alternative idea that can help curb this under generalized smoothness, leading to the optimal $\Ocal({\varepsilon^{-4})}$ rate.

\section*{Acknowledgements}
This work is supported in part by NSF under grants DMS-2502560 and CNS-2313109.
%%%%%%%%%%%%%%%%%%%%%%%%%%References%%%%%%%%%%%%%%%%%%%%%%%
\bibliographystyle{plainnat}
\bibliography{refs.bib}

@article{jain2018parallelizing,
  title={Parallelizing stochastic gradient descent for least squares regression: mini-batching, averaging, and model misspecification},
  author={Jain, Prateek and Kakade, Sham M and Kidambi, Rahul and Netrapalli, Praneeth and Sidford, Aaron},
  journal={Journal of machine learning research},
  volume={18},
  number={223},
  pages={1--42},
  year={2018}
}

@inproceedings{nguyen2017sarah,
  title={SARAH: A novel method for machine learning problems using stochastic recursive gradient},
  author={Nguyen, Lam M and Liu, Jie and Scheinberg, Katya and Tak{\'a}{\v{c}}, Martin},
  booktitle={International conference on machine learning},
  pages={2613--2621},
  year={2017},
  organization={PMLR}
}

@inproceedings{rakhlin2012making,
  title={Making gradient descent optimal for strongly convex stochastic optimization},
  author={Rakhlin, Alexander and Shamir, Ohad and Sridharan, Karthik},
  booktitle={Proceedings of the 29th International Coference on International Conference on Machine Learning},
  pages={1571--1578},
  year={2012}
}

@inproceedings{li2021page,
  title={PAGE: A simple and optimal probabilistic gradient estimator for nonconvex optimization},
  author={Li, Zhize and Bao, Hongyan and Zhang, Xiangliang and Richt{\'a}rik, Peter},
  booktitle={International conference on machine learning},
  pages={6286--6295},
  year={2021},
  organization={PMLR}
}

@article{krasnosel1955two, title={Two remarks on the method of successive approximations}, author={Krasnosel'ski\u{\i}, Mark Aleksandrovich}, journal={Uspekhi matematicheskikh nauk}, volume={10}, number={1}, pages={123--127}, year={1955}, publisher={Russian Academy of Sciences, Steklov Mathematical Institute of Russian~…} }

@article{mann1953mean, title={Mean value methods in iteration}, author={Mann, W Robert}, journal={Proceedings of the American Mathematical Society}, volume={4}, number={3}, pages={506--510}, year={1953} }

@article{arjevani2023lower,
  title={Lower bounds for non-convex stochastic optimization},
  author={Arjevani, Yossi and Carmon, Yair and Duchi, John C and Foster, Dylan J and Srebro, Nathan and Woodworth, Blake},
  journal={Mathematical Programming},
  volume={199},
  number={1},
  pages={165--214},
  year={2023},
  publisher={Springer}
}

@article{bottou2018optimization,
  title={Optimization methods for large-scale learning},
  author={Bottou, L{\'e}on and Curtis, Frank E and Nocedal, Jorge},
  journal={SIAM Review},
  volume={60},
  number={2},
  pages={223--311},
  year={2018},
  publisher={SIAM}
}

@article{carmon2019lower_i,
  title={Lower bounds for finding stationary points {I}},
  author={Carmon, Yair and Duchi, John C and Hinder, Oliver and Sidford, Aaron},
  journal={Mathematical Programming},
  volume={177},
  pages={193--230},
  year={2019},
  publisher={Springer}
}

@inproceedings{cutkosky2019momentum,
  title={Momentum-based variance reduction in non-convex {SGD}},
  author={Cutkosky, Ashok and Orabona, Francesco},
  booktitle={Advances in Neural Information Processing Systems},
  year={2019}
}

@inproceedings{fang2018spider,
  title={{SPIDER}: Near-optimal non-convex optimization via stochastic path-integrated differential estimator},
  author={Fang, Cong and Li, Chris Jiao and Lin, Zhouchen and Zhang, Tong},
  booktitle={Advances in Neural Information Processing Systems},
  pages={689--699},
  year={2018}
}

@article{ghadimi2013stochastic,
  title={Stochastic first- and zeroth-order methods for nonconvex stochastic programming},
  author={Ghadimi, Saeed and Lan, Guanghui},
  journal={SIAM Journal on Optimization},
  volume={23},
  number={4},
  pages={2341--2368},
  year={2013},
  publisher={SIAM}
}

@book{nesterov2004introductory,
  title={Introductory lectures on convex optimization: A basic course},
  author={Nesterov, Yurii},
  year={2004},
  publisher={Springer Science \& Business Media}
}

@article{wang2018spiderboost,
  title={Spiderboost: A class of faster variance-reduced algorithms for nonconvex optimization},
  author={Wang, Zhe and Ji, Kaiyi and Zhou, Yi and Liang, Yingbin and Tarokh, Vahid},
  journal={arXiv preprint arXiv:1810.10690},
  year={2018}
}

@inproceedings{moulines2011non,
  title={Non-asymptotic analysis of stochastic approximation algorithms for machine learning},
  author={Bach, Francis and Moulines, Eric},
  booktitle={Advances in Neural Information Processing Systems},
  volume={24},
  year={2011}
}

@article{blum1954approximation,
  title={Approximation methods which converge with probability one},
  author={Blum, Julius R},
  journal={The Annals of Mathematical Statistics},
  pages={382--386},
  year={1954},
  publisher={JSTOR}
}

@article{gladyshev1965stochastic,
  title={On stochastic approximation},
  author={Gladyshev, EG},
  journal={Theory of Probability \& Its Applications},
  volume={10},
  number={2},
  pages={275--278},
  year={1965},
  publisher={SIAM}
}

@inproceedings{gorbunov2020unified,
  title={A unified theory of {SGD}: Variance reduction, sampling, quantization and coordinate descent},
  author={Gorbunov, Eduard and Hanzely, Filip and Richt{\'a}rik, Peter},
  booktitle={International Conference on Artificial Intelligence and Statistics},
  pages={680--690},
  year={2020},
  organization={PMLR}
}

@article{grimmer2019convergence,
  title={Convergence rates for deterministic and stochastic subgradient methods without {Lipschitz} continuity},
  author={Grimmer, Benjamin},
  journal={SIAM Journal on Optimization},
  volume={29},
  number={2},
  pages={1350--1365},
  year={2019},
  publisher={SIAM}
}

@article{halpern1967fixed,
  title={Fixed points of nonexpanding maps},
  author={Halpern, Benjamin},
  journal={Bulletin of the American Mathematical Society},
  volume={73},
  number={6},
  pages={957--961},
  year={1967}
}

@article{ilandarideva2023accelerated,
  title={Accelerated stochastic approximation with state-dependent noise},
  author={Ilandarideva, Sasila and Juditsky, Anatoli and Lan, Guanghui and Li, Tianjiao},
  journal={arXiv preprint arXiv:2307.01497},
  year={2023}
}

@article{khaled2022better,
  title={Better theory for {SGD} in the nonconvex world},
  author={Khaled, Ahmed and Richt{\'a}rik, Peter},
  journal={Transactions on Machine Learning Research},
  year={2023}
}

@article{khaled2023unified,
  title={Unified analysis of stochastic gradient methods for composite convex and smooth optimization},
  author={Khaled, Ahmed and Sebbouh, Othmane and Loizou, Nicolas and Gower, Robert M and Richt{\'a}rik, Peter},
  journal={Journal of Optimization Theory and Applications},
  volume={199},
  number={2},
  pages={499--540},
  year={2023},
  publisher={Springer}
}

@article{nemirovski2009robust,
  title={Robust stochastic approximation approach to stochastic programming},
  author={Nemirovski, Arkadi and Juditsky, Anatoli and Lan, Guanghui and Shapiro, Alexander},
  journal={SIAM Journal on Optimization},
  volume={19},
  number={4},
  pages={1574--1609},
  year={2009},
  publisher={SIAM}
}

@book{nesterov2018lectures,
  title={Lectures on convex optimization},
  author={Nesterov, Yurii},
  volume={137},
  year={2018},
  publisher={Springer}
}

@inproceedings{shamir2013stochastic,
  title={Stochastic gradient descent for non-smooth optimization: Convergence results and optimal averaging schemes},
  author={Shamir, Ohad and Zhang, Tong},
  booktitle={International Conference on Machine Learning},
  pages={71--79},
  year={2013},
  organization={PMLR}
}

@article{alacaoglu2025towards,
  title={Towards weaker variance assumptions for stochastic optimization},
  author={Alacaoglu, Ahmet and Malitsky, Yura and Wright, Stephen J},
  journal={arXiv preprint arXiv:2504.09951},
  year={2025}
}

@article{zhou2020stochastic,
  title={Stochastic nested variance reduction for nonconvex optimization},
  author={Zhou, Dongruo and Xu, Pan and Gu, Quanquan},
  journal={Journal of Machine Learning Research},
  volume={21},
  number={103},
  pages={1--63},
  year={2020}
}

@article{johnson2013accelerating,
  title={Accelerating stochastic gradient descent using predictive variance reduction},
  author={Johnson, Rie and Zhang, Tong},
  journal={Advances in neural information processing systems},
  volume={26},
  year={2013}
}

@article{defazio2014saga,
  title={SAGA: A fast incremental gradient method with support for non-strongly convex composite objectives},
  author={Defazio, Aaron and Bach, Francis and Lacoste-Julien, Simon},
  journal={Advances in neural information processing systems},
  volume={27},
  year={2014}
}

@article{fazla2026lower,
  title={Lower Bounds and Proximally Anchored SGD for Non-Convex Minimization Under Unbounded Variance},
  author={Fazla, Arda and Kaya, Ege C and Upadhyay, Antesh and Hashemi, Abolfazl},
  journal={arXiv preprint arXiv:2604.16620},
  year={2026}
}

@article{jin2021non,
  title={Non-convex distributionally robust optimization: Non-asymptotic analysis},
  author={Jin, Jikai and Zhang, Bohang and Wang, Haiyang and Wang, Liwei},
  journal={Advances in Neural Information Processing Systems},
  volume={34},
  pages={2771--2782},
  year={2021}
}

@article{zhang2019gradient,
  title={Why gradient clipping accelerates training: A theoretical justification for adaptivity},
  author={Zhang, Jingzhao and He, Tianxing and Sra, Suvrit and Jadbabaie, Ali},
  journal={arXiv preprint arXiv:1905.11881},
  year={2019}
}

@inproceedings{chen2023generalized,
  title={Generalized-smooth nonconvex optimization is as efficient as smooth nonconvex optimization},
  author={Chen, Ziyi and Zhou, Yi and Liang, Yingbin and Lu, Zhaosong},
  booktitle={International Conference on Machine Learning},
  pages={5396--5427},
  year={2023},
  organization={PMLR}
}

@article{levy2020large,
  title={Large-scale methods for distributionally robust optimization},
  author={Levy, Daniel and Carmon, Yair and Duchi, John C and Sidford, Aaron},
  journal={Advances in neural information processing systems},
  volume={33},
  pages={8847--8860},
  year={2020}
}

@inproceedings{cutkosky2020momentum,
  title={Momentum improves normalized sgd},
  author={Cutkosky, Ashok and Mehta, Harsh},
  booktitle={International conference on machine learning},
  pages={2260--2268},
  year={2020},
  organization={PMLR}
}

@inproceedings{reisizadeh2025variance,
  title={Variance-reduced clipping for non-convex optimization},
  author={Reisizadeh, Amirhossein and Li, Haochuan and Das, Subhro and Jadbabaie, Ali},
  booktitle={ICASSP 2025-2025 IEEE International Conference on Acoustics, Speech and Signal Processing (ICASSP)},
  pages={1--5},
  year={2025},
  organization={IEEE}
}

@inproceedings{koloskova2023revisiting,
  title={Revisiting gradient clipping: Stochastic bias and tight convergence guarantees},
  author={Koloskova, Anastasia and Hendrikx, Hadrien and Stich, Sebastian U},
  booktitle={International Conference on Machine Learning},
  pages={17343--17363},
  year={2023},
  organization={PMLR}
}

@article{zhang2020improved,
  title={Improved analysis of clipping algorithms for non-convex optimization},
  author={Zhang, Bohang and Jin, Jikai and Fang, Cong and Wang, Liwei},
  journal={Advances in Neural Information Processing Systems},
  volume={33},
  pages={15511--15521},
  year={2020}
}

@inproceedings{wang2024provable,
  title={Provable adaptivity of adam under non-uniform smoothness},
  author={Wang, Bohan and Zhang, Yushun and Zhang, Huishuai and Meng, Qi and Sun, Ruoyu and Ma, Zhi-Ming and Liu, Tie-Yan and Luo, Zhi-Quan and Chen, Wei},
  booktitle={Proceedings of the 30th ACM SIGKDD Conference on Knowledge Discovery and Data Mining},
  pages={2960--2969},
  year={2024}
}

@article{li2023convergence,
  title={Convergence of adam under relaxed assumptions},
  author={Li, Haochuan and Rakhlin, Alexander and Jadbabaie, Ali},
  journal={Advances in Neural Information Processing Systems},
  volume={36},
  pages={52166--52196},
  year={2023}
}

@article{takezawa2024parameter,
  title={Parameter-free clipped gradient descent meets polyak},
  author={Takezawa, Yuki and Bao, Han and Sato, Ryoma and Niwa, Kenta and Yamada, Makoto},
  journal={Advances in Neural Information Processing Systems},
  volume={37},
  pages={44575--44599},
  year={2024}
}

@article{zhao2021convergence,
  title={On the convergence and improvement of stochastic normalized gradient descent},
  author={Zhao, Shen-Yi and Xie, Yin-Peng and Li, Wu-Jun},
  journal={Science China Information Sciences},
  volume={64},
  number={3},
  pages={132103},
  year={2021},
  publisher={Springer}
}

@inproceedings{hubler2024parameter,
  title={Parameter-agnostic optimization under relaxed smoothness},
  author={H{\"u}bler, Florian and Yang, Junchi and Li, Xiang and He, Niao},
  booktitle={International Conference on Artificial Intelligence and Statistics},
  pages={4861--4869},
  year={2024},
  organization={PMLR}
}

@article{crawshaw2022robustness,
  title={Robustness to unbounded smoothness of generalized signsgd},
  author={Crawshaw, Michael and Liu, Mingrui and Orabona, Francesco and Zhang, Wei and Zhuang, Zhenxun},
  journal={Advances in neural information processing systems},
  volume={35},
  pages={9955--9968},
  year={2022}
}

@inproceedings{
khirirat2026error,
title={Error Feedback under \$(L\_0,L\_1)\$-Smoothness: Normalization and Momentum},
author={Sarit Khirirat and Abdurakhmon Sadiev and Artem Riabinin and Eduard Gorbunov and Peter Richt{\'a}rik},
booktitle={The Thirty-ninth Annual Conference on Neural Information Processing Systems},
year={2026},
url={https://openreview.net/forum?id=LmcTbBvgjP}
}
%%%%%%%%%%%%%%%%%%%%%%%%%%%%%%%%%%%%%%%%%%%%%%%%%%%%%%%%%%%%
\newpage
\appendix
\section*{Appendix Table of Contents}
\addcontentsline{toc}{section}{Appendix}
\markboth{Appendix}{Appendix}
% Create a mini-TOC for the appendix
\startcontents[appendix]
\printcontents[appendix]{l}{1}{\setcounter{tocdepth}{3}}

\newpage
\section{A Short Primer on Smoothness Classes}
\label{apdx:primer_smoothness}
We summarize the smoothness notions used throughout the paper. The baseline class is the standard smoothness class $\Lcal$, while the generalized smoothness classes $\Lcal_{\rm H}^*$, $\Lcal_{\rm asym}^*$, and $\Lcal_{\rm sym}^*(\alpha)$ allow the local smoothness constant to grow with the gradient norm. The expected class $\E\Lcal_{\rm sym}^*(\alpha)$ is the stochastic analogue used for variance-reduced methods. See~\citet{chen2023generalized} for a more detailed discussion.

\paragraph{Standard smoothness $\Lcal$.}
A differentiable function $f:\R^d\to\R$ belongs to $\Lcal$ if there exists $L_0>0$ such that, for all $x,y\in\R^d$,
\begin{equation}
    \label{eq:primer_l_smooth}
    \|\nabla f(y)-\nabla f(x)\|
    \le L_0\|y-x\|.
\end{equation}
Equivalently, the gradient is globally Lipschitz continuous. This is the classical smoothness assumption used in standard nonconvex optimization.

\paragraph{Hessian-based generalized smoothness $\Lcal_{\rm H}^*$.}
For twice differentiable objectives, the Hessian-based generalized smoothness class $\Lcal_{\rm H}^*$ is defined by the existence of constants $L_0,L_1>0$ such that, for all $x\in\R^d$,
\begin{equation}
    \label{eq:primer_hessian_gs}
    \|\nabla^2 f(x)\|
    \le L_0 + L_1\|\nabla f(x)\|.
\end{equation}
This is often called $(L_0,L_1)$-smoothness. Unlike standard smoothness, the curvature is allowed to increase when the gradient norm is large.

\paragraph{Asymmetric generalized smoothness: $\Lcal_{\rm asym}^*$.}
The asymmetric generalized smoothness class $\Lcal_{\rm asym}^*$ consists of differentiable functions satisfying, for some $L_0,L_1>0$ and all $x,y\in\R^d$,
\begin{equation}
    \label{eq:primer_asym_gs}
    \|\nabla f(y)-\nabla f(x)\|
    \le
    \left(L_0+L_1\|\nabla f(y)\|\right)\|y-x\|.
\end{equation}

\paragraph{$\alpha$-symmetric generalized smoothness $\Lcal_{\rm sym}^*(\alpha)$.}
Following \citet{chen2023generalized}, for $\alpha\in(0,1]$, we say that $f\in\Lcal_{\rm sym}^*(\alpha)$ if there exist constants $L_0,L_1>0$ such that, for all $x,y\in\R^d$,
\begin{equation}
    \label{eq:primer_sym_gs}
    \|\nabla f(y)-\nabla f(x)\| \le \left( L_0 + L_1 \max_{\theta\in[0,1]} \|\nabla f(x_\theta)\|^\alpha \right) \|y-x\|,
    \qquad
    x_\theta:=\theta y+(1-\theta)x.
\end{equation}
This condition is symmetric in the sense that the local smoothness is controlled along the entire line segment between $x$ and $y$, rather than only at one endpoint (as in the asymmetric condition in ~\eqref{eq:primer_asym_gs}). The parameter $\alpha$ determines how strongly the local smoothness depends on the gradient norm. The case $\alpha=1$ corresponds to the usual symmetric $(L_0,L_1)$ generalized smoothness regime, while $\alpha\in(0,1)$ gives a sublinear dependence on the gradient norm. Standard smoothness is recovered as the special case $L_1=0$.

\paragraph{Expected $\alpha$-symmetric generalized smoothness $\E\Lcal_{\rm sym}^*(\alpha)$.}
For stochastic objectives $f(x)=\E_\xi[f(x;\xi)]$, the expected generalized smoothness class $\E\Lcal_{\rm sym}^*(\alpha)$ requires that, for some $L_0,L_1>0$ and all $x,y\in\R^d$,
\begin{equation}
    \label{eq:primer_expected_sym_gs}
    \E_\xi \big\| \nabla f(y;\xi)-\nabla f(x;\xi) \big\|^2 \le \|y-x\|^2 \E_\xi \Big[ \big( L_0 + L_1 \max_{\theta\in[0,1]} \|\nabla f(x_\theta;\xi)\|^\alpha \big)^2 \Big],
    \
    x_\theta:=\theta y+(1-\theta)x.
\end{equation}
This is a mean-square, sample-level analogue of $\Lcal_{\rm sym}^*(\alpha)$. It is stronger than requiring only the population objective $f$ to satisfy $\Lcal_{\rm sym}^*(\alpha)$, and it is the relevant condition for analyzing recursive variance-reduced estimators such as $\mathsf{NSTORM}$.

\paragraph{Mean-square smoothness $\mathsf{MSS}$.}
Mean-square smoothness is the standard stochastic analogue of $L$-smoothness. It requires that there exists $L>0$ such that, for all $x,y\in\R^d$,
\begin{equation}
    \label{eq:primer_mss}
    \E_\xi \left\| \nabla f(y;\xi)-\nabla f(x;\xi) \right\|^2 \le L^2\|y-x\|^2.
\end{equation}
Thus, $\mathsf{MSS}$ controls stochastic gradient differences without any dependence on gradient magnitudes. In contrast, $\E\Lcal_{\rm sym}^*(\alpha)$ permits the stochastic smoothness factor to grow with the sample-gradient norm.

\paragraph{Relations among the classes.}
The classes satisfy the following inclusions, up to constants:
\begin{equation}
    \label{eq:primer_class_relations}
    \Lcal \subset \Lcal_{\rm sym}^*(\alpha), \qquad \Lcal_{\rm asym}^* \subset \Lcal_{\rm sym}^*(1), \qquad \Lcal_{\rm H}^* \subset \Lcal_{\rm sym}^*(1).
\end{equation}
Moreover, on twice differentiable functions, $\Lcal_{\rm H}^*$ and $\Lcal_{\rm sym}^*(1)$ are equivalent up to constants. Hence, $\Lcal_{\rm sym}^*(\alpha)$ provides a unified deterministic framework that contains standard smoothness, asymmetric generalized smoothness, and Hessian-based $(L_0,L_1)$-smoothness. The expected class $\E\Lcal_{\rm sym}^*(\alpha)$ plays the analogous role for stochastic gradient differences.

\paragraph{Role in this paper.}
In this paper, $\Lcal$ and $\Lcal_{\rm sym}^*(\alpha)$ are used for the analysis of normalized stochastic gradient descent with momentum, $\mathsf{NSGDM}$. Under $\mathsf{BG}$-0 noise, the normalized update controls the trajectory length, and the generalized smoothness parameters affect constants but do not change the $\Ocal(\varepsilon^{-6})$ complexity exponent.

For normalized STORM, $\mathsf{NSTORM}$, the relevant conditions are $\mathsf{MSS}$ and $\E\Lcal_{\rm sym}^*(\alpha)$. Under $\mathsf{MSS}$, the transition noise in the recursive estimator is controlled only by the step length, leading to the optimal $\Ocal(\varepsilon^{-4})$ rate under $\mathsf{BG}$-0. Under $\E\Lcal_{\rm sym}^*(\alpha)$, however, the transition noise also depends on sample-gradient moments. These moments interact with the distance-dependent $\mathsf{BG}$-0 variance, producing the $\alpha$-dependent rates established in the main text.

\paragraph{Related Work: Generalized smoothness.}
Standard theory often assumes a constant smoothness parameter, but recent research focuses on objectives where local smoothness grows alongside the gradient norm \cite{carmon2019lower_i, ghadimi2013stochastic}. A prime example is $\Lcal^*_{\rm H}$ or $(L_0, L_1)$-smoothness, where the Hessian is bounded by $\|\nabla^2 f(x)\| \le L_0 + L_1\|\nabla f(x)\|$ \cite{zhang2019gradient}. This concept was further expanded by \citet{chen2023generalized} into the $\alpha$-symmetric generalized smoothness framework ($\Lcal_{\rm sym}^*(\alpha)/\E\Lcal_{\rm sym}^*(\alpha)$), which covers standard smoothness, asymmetric generalized smoothness~\cite{jin2021non,levy2020large,reisizadeh2025variance}, and Hessian-based generalized smoothness, and also covers high-order polynomial and exponential-type objectives. While prior work shows that generalized smooth problems can match the efficiency of standard smooth optimization, achieving deterministic rates of $\Ocal(\varepsilon^{-2})$ with normalized gradients or $\Ocal(\varepsilon^{-3})$ with SPIDER-style variance reduction, these results rely on relatively tame variance assumptions. Our work investigates whether these efficiencies still hold under $\mathsf{BG}$-0 noise, where variance is trajectory-dependent.

%%%%%%%%%%%%%%%%%%%%%%%%%%%%%%%%%%%%%%%%%%%%%%%%%%%%%%%%%%%%%%%%%%%%%%%%%%%%
%%%%%%%%%%%%%%Proofs for NSGDM %%%%%%%%%%%%%%%%%%%%%%%%%%%%%%%%%%%%%%%%%%%%%
%%%%%%%%%%%%%%%%%%%%%%%%%%%%%%%%%%%%%%%%%%%%%%%%%%%%%%%%%%%%%%%%%%%%%%%%%%%%%
\newpage
\section{Proofs for Normalized Momentum under \texorpdfstring{$\mathsf{BG}$}{BG}-0 Noise}
\label{apdx:nsgdm_proofs}

This appendix proves the normalized-momentum guarantees stated in
Theorem~\ref{thm:nsgdm_three_regimes}. Throughout this section, let
\[
    T:=K+1,
    \qquad
    \Delta:=f(x^0)-f^{\inf}.
\]
We define
\[
    g_k:=\E\|\nabla f(x^k)\|,
    \qquad
    b_k:=\E\|v^k-\nabla f(x^k)\|,
\]
and
\[
    \Delta^k:=\E[f(x^k)-f^{\inf}],
    \qquad
    S_b(K):=\sum_{k=0}^K b_k,
    \qquad
    S_\Delta(K):=\sum_{k=0}^K \Delta^k.
\]
The output \(\widehat x\) is sampled uniformly from
\(\{x^0,x^1,\ldots,x^K\}\), and hence
\begin{equation}
\label{eq:app_uniform_output_identity}
    \E\|\nabla f(\widehat x)\|
    =
    \frac1T\sum_{k=0}^K g_k.
\end{equation}
We use the convention that \(v/\|v\|=0\) whenever \(v=0\).

\subsection{Auxiliary inequalities}
\label{app:nsgdm_aux}

\begin{lemma}[Trajectory control by normalization]
\label{lem:app_traj_control}
For the \(\mathsf{NSGDM}\) iterates in~\eqref{eq:nsgdm_updates},
\[
    \|x^{k+1}-x^k\|\le \gamma,
    \qquad
    \|x^k-x^0\|\le k\gamma.
\]
Consequently, under Assumption~\ref{as:bg0_oracle},
\begin{equation}
\label{eq:app_bg0_along_path}
    \E\|\nabla f(x^k;\xi^k)-\nabla f(x^k)\|^2
    \le B^2k^2\gamma^2+G^2.
\end{equation}
In particular, since \(v^0=\nabla f(x^0;\xi^0)\),
\begin{equation}
\label{eq:app_B0_bound}
    b_0\le G.
\end{equation}
\end{lemma}

\begin{proof}
The normalized direction has norm at most one. Hence
\[
    \|x^{k+1}-x^k\|\le \gamma.
\]
Summing the increments gives
\[
    \|x^k-x^0\|
    \le
    \sum_{t=0}^{k-1}\|x^{t+1}-x^t\|
    \le
    k\gamma.
\]
Substitution into Assumption~\ref{as:bg0_oracle} gives
\eqref{eq:app_bg0_along_path}. Finally,
\[
    b_0
    =
    \E\|v^0-\nabla f(x^0)\|
    \le
    \sqrt{\E\|v^0-\nabla f(x^0)\|^2}
    \le
    G,
\]
because \(x^0-x^0=0\).
\end{proof}

\begin{lemma}[Direction-error inequality]
\label{lem:app_direction_error}
For any \(a,e\in\R^d\), define
\[
    d=
    \begin{cases}
        \dfrac{a+e}{\|a+e\|}, & a+e\neq0,\\[4pt]
        0, & a+e=0.
    \end{cases}
\]
Then
\begin{equation}
\label{eq:app_direction_error}
    \langle a,d\rangle\ge \|a\|-2\|e\|.
\end{equation}
\end{lemma}

\begin{proof}
If \(a+e=0\), then \(d=0\) and \(e=-a\). Hence
\[
    \langle a,d\rangle=0\ge \|a\|-2\|a\|.
\]
Now suppose \(a+e\neq0\). Then
\[
    \langle a,d\rangle
    =
    \frac{\|a\|^2+\langle a,e\rangle}{\|a+e\|}
    \ge
    \frac{\|a\|(\|a\|-\|e\|)}{\|a+e\|}.
\]
If \(\|a\|\ge\|e\|\), then using
\(\|a+e\|\le \|a\|+\|e\|\), we obtain
\[
    \langle a,d\rangle
    \ge
    \frac{\|a\|(\|a\|-\|e\|)}{\|a\|+\|e\|}
    =
    \|a\|-\frac{2\|a\|\|e\|}{\|a\|+\|e\|}
    \ge
    \|a\|-2\|e\|.
\]
If \(\|a\|<\|e\|\), then by Cauchy--Schwarz,
\[
    \langle a,d\rangle\ge -\|a\|\ge \|a\|-2\|e\|.
\]
The claim follows.
\end{proof}

\begin{lemma}[Young inequality for fractional powers]
\label{lem:app_young_alpha}
For every \(x\ge0\), \(\alpha\in(0,1)\), and \(\rho>0\),
\begin{equation}
\label{eq:app_young_alpha}
    x^\alpha
    \le
    \alpha\rho x+(1-\alpha)\rho^{-\alpha/(1-\alpha)}.
\end{equation}
\end{lemma}

\begin{proof}
Apply Young's inequality with conjugate exponents
\(p=1/\alpha\) and \(q=1/(1-\alpha)\) to
\[
    u=(\rho x)^\alpha,
    \qquad
    v=\rho^{-\alpha}.
\]
\end{proof}

\subsection{Standard smoothness and \texorpdfstring{$\Lcal^*_{\rm sym}(1)$}{lsym1}}
\label{app:nsgdm_l1_case}

We first prove the standard smoothness case and the endpoint
\(\mathcal L_{\rm sym}^*(1)\) case. For the endpoint generalized smooth case, we use the following standard consequences of symmetric generalized smoothness (see~\citet{chen2023generalized}, Proposition 1):
for all \(x,y\in\R^d\),
\begin{equation}
\label{eq:app_l1_grad_consequence}
    \|\nabla f(y)-\nabla f(x)\|
    \le
    \bigl(L_0+L_1\|\nabla f(x)\|\bigr)
    e^{L_1\|y-x\|}\|y-x\|,
\end{equation}
and
\begin{equation}
\label{eq:app_l1_func_consequence}
    f(y)
    \le
    f(x)+\langle\nabla f(x),y-x\rangle
    +\frac12
    \bigl(L_0+L_1\|\nabla f(x)\|\bigr)
    e^{L_1\|y-x\|}
    \|y-x\|^2.
\end{equation}
When \(L_1=0\), these reduce to the standard \(L_0\)-smooth inequalities.
We also use the gradient-growth inequality (see~\citet{khirirat2026error}, Lemma 2)
\begin{equation}
\label{eq:app_grad_growth_l1}
    \|\nabla f(x)\|
    \le
    8L_1(f(x)-f^{\inf})+\frac{L_0}{L_1},
    \qquad L_1>0,
\end{equation}
which follows from lower boundedness and the endpoint generalized smoothness
condition.

\begin{theorem}[Single-sample master bound for \(L_0\)-smoothness and \(\alpha=1\)]
\label{thm:app_nsgdm_l1_master}
Suppose $f$ is bounded below and~\ref{as:bg0_oracle} hold. Assume
that either \(f\) is \(L_0\)-smooth, or that
\eqref{eq:app_l1_grad_consequence}--\eqref{eq:app_l1_func_consequence} hold.
In the standard smooth case, set \(L_1=0\). If
\begin{equation}
\label{eq:app_l1_master_conditions}
    \gamma L_1\le \frac12,
    \qquad
    \frac{32L_1^2\gamma^2T}{\eta}\le\frac12,
\end{equation}
with both conditions interpreted as vacuous when \(L_1=0\), then
\begin{align}
\E\|\nabla f(\widehat x)\|
&\le
\frac{2\Delta}{\gamma T}
+\frac{4b_0}{\eta T}
+\frac{16L_0\gamma}{\eta}
+2L_0\gamma +4\sqrt{\eta}
\Big(G+\frac{B\gamma K}{2}\Big).
\notag\\
&
\quad
+\frac{64L_1^2\gamma}{\eta}
\Bigg[
4\Delta
+\frac{4\gamma b_0}{\eta}
+\frac{16L_0\gamma^2T}{\eta}
+4\gamma\sqrt{\eta}\,TG
+2B\gamma^2\sqrt{\eta}\,KT
+2L_0\gamma^2T
\Bigg]
\label{eq:app_l1_master_bound}
\end{align}
When \(L_1=0\), the entire \(L_1^2\)-term vanishes.
\end{theorem}

\begin{proof}
Let
\[
    E_k:=v^k-\nabla f(x^k),
    \qquad
    D_k:=\nabla f(x^k)-\nabla f(x^{k+1}),
\]
and
\[
    \zeta_{k+1}
    :=
    \nabla f(x^{k+1};\xi^{k+1})-\nabla f(x^{k+1}).
\]

\paragraph{Step 1: descent.}
Using \eqref{eq:app_l1_func_consequence} with
\(x=x^k\) and \(y=x^{k+1}=x^k-\gamma d^k\), we get
\[
\begin{aligned}
    f(x^{k+1})
    &\le
    f(x^k)-\gamma\langle\nabla f(x^k),d^k\rangle \\
    &\quad
    +\frac{\gamma^2}{2}e^{L_1\gamma}
    \bigl(L_0+L_1\|\nabla f(x^k)\|\bigr).
\end{aligned}
\]
Since \(\gamma L_1\le 1/2\), we have \(e^{L_1\gamma}\le e^{1/2}<2\).
By Lemma~\ref{lem:app_direction_error},
\[
    \langle\nabla f(x^k),d^k\rangle
    \ge
    \|\nabla f(x^k)\|-2\|E_k\|.
\]
Therefore,
\[
\begin{aligned}
    f(x^{k+1})
    &\le
    f(x^k)
    -\gamma\|\nabla f(x^k)\|
    +2\gamma\|E_k\|
    +L_0\gamma^2
    +L_1\gamma^2\|\nabla f(x^k)\|.
\end{aligned}
\]
Using \(\gamma L_1\le1/2\), we obtain
\begin{equation}
\label{eq:app_l1_descent_recursion}
    \Delta^{k+1}
    \le
    \Delta^k-\frac{\gamma}{2}g_k+2\gamma b_k+L_0\gamma^2.
\end{equation}

\paragraph{Step 2: estimator recursion.}
The momentum update gives
\[
    E_{k+1}
    =
    (1-\eta)E_k+(1-\eta)D_k+\eta\zeta_{k+1}.
\]
Unrolling,
\[
    E_k
    =
    (1-\eta)^kE_0
    +\sum_{t=0}^{k-1}(1-\eta)^{k-t}D_t
    +\eta\sum_{t=0}^{k-1}(1-\eta)^{k-1-t}\zeta_{t+1}.
\]
For the drift term, \eqref{eq:app_l1_grad_consequence},
\(e^{L_1\gamma}\le2\), and \eqref{eq:app_grad_growth_l1} imply, for
\(L_1>0\),
\[
\begin{aligned}
    \E\|D_t\|
    &\le
    2\gamma\bigl(L_0+L_1g_t\bigr) \\
    &\le
    2\gamma\left(L_0+L_1
    \left(8L_1\Delta^t+\frac{L_0}{L_1}\right)\right) \\
    &=
    4L_0\gamma+16L_1^2\gamma\Delta^t.
\end{aligned}
\]
For \(L_1=0\), the second term is absent and the same displayed bound remains
valid with the \(L_1^2\)-term equal to zero.

For the martingale term, conditional unbiasedness and the orthogonality of
martingale differences give
\[
\begin{aligned}
    &\E\left\|
    \eta\sum_{t=0}^{k-1}(1-\eta)^{k-1-t}\zeta_{t+1}
    \right\| \\
    &\quad\le
    \left(
    \eta^2\sum_{t=0}^{k-1}(1-\eta)^{2(k-1-t)}
    \E\|\zeta_{t+1}\|^2
    \right)^{1/2}.
\end{aligned}
\]
By Lemma~\ref{lem:app_traj_control}, for \(t\le k-1\),
\[
    \E\|\zeta_{t+1}\|^2
    \le
    B^2(t+1)^2\gamma^2+G^2
    \le
    B^2k^2\gamma^2+G^2.
\]
Also,
\[
    \sum_{t=0}^{k-1}(1-\eta)^{2(k-1-t)}
    \le
    \frac1\eta.
\]
Hence
\[
    \E\left\|
    \eta\sum_{t=0}^{k-1}(1-\eta)^{k-1-t}\zeta_{t+1}
    \right\|
    \le
    \sqrt{\eta}(G+Bk\gamma).
\]
Combining the drift and martingale estimates yields
\begin{equation}
\label{eq:app_l1_Bk_bound}
    b_k
    \le
    (1-\eta)^kb_0
    +\frac{4L_0\gamma}{\eta}
    +16L_1^2\gamma
    \sum_{t=0}^{k-1}(1-\eta)^{k-t}\Delta^t
    +\sqrt\eta(G+Bk\gamma).
\end{equation}

\paragraph{Step 3: summing the estimator errors.}
Summing \eqref{eq:app_l1_Bk_bound} over \(k=0,\ldots,K\), and using
\[
    \sum_{k=0}^K(1-\eta)^k\le \frac1\eta,
\]
and
\[
    \sum_{k=0}^K\sum_{t=0}^{k-1}(1-\eta)^{k-t}\Delta^t
    \le
    \frac1\eta S_\Delta(K),
\]
we get
\begin{equation}
\label{eq:app_l1_SB_bound}
    S_b(K)
    \le
    \frac{b_0}{\eta}
    +\frac{4L_0\gamma T}{\eta}
    +\frac{16L_1^2\gamma}{\eta}S_\Delta(K)
    +\sqrt\eta
    \left(TG+\frac{B\gamma KT}{2}\right).
\end{equation}

\paragraph{Step 4: controlling the accumulated suboptimality.}
Dropping the negative gradient term in
\eqref{eq:app_l1_descent_recursion}, iterating, and summing gives
\[
    S_\Delta(K)
    \le
    (K+2)\Delta
    +2\gamma T S_b(K)
    +\frac{L_0\gamma^2T(K+2)}{2}.
\]
Substituting \eqref{eq:app_l1_SB_bound} gives
\[
\begin{aligned}
S_\Delta(K)
&\le
(K+2)\Delta
+\frac{2\gamma T b_0}{\eta}
+\frac{8L_0\gamma^2T^2}{\eta}
+\frac{32L_1^2\gamma^2T}{\eta}S_\Delta(K) \\
&\quad
+2\gamma\sqrt\eta T^2G
+B\gamma^2\sqrt\eta KT^2
+\frac{L_0\gamma^2T(K+2)}{2}.
\end{aligned}
\]
By \eqref{eq:app_l1_master_conditions}, the coefficient of
\(S_\Delta(K)\) on the right is at most \(1/2\). Moving it to the left and
using \(K+2\le2T\), we obtain
\begin{equation}
\label{eq:app_l1_Sdelta_avg_bound}
\frac{S_\Delta(K)}{T}
\le
4\Delta
+\frac{4\gamma b_0}{\eta}
+\frac{16L_0\gamma^2T}{\eta}
+4\gamma\sqrt\eta\,TG
+2B\gamma^2\sqrt\eta\,KT
+2L_0\gamma^2T.
\end{equation}

\paragraph{Step 5: final stationarity bound.}
Summing \eqref{eq:app_l1_descent_recursion} gives
\[
    \frac{\gamma}{2}\sum_{k=0}^K g_k
    \le
    \Delta+2\gamma S_b(K)+L_0\gamma^2T.
\]
Dividing by \(\gamma T/2\), using
\eqref{eq:app_uniform_output_identity}, and substituting
\eqref{eq:app_l1_SB_bound} and \eqref{eq:app_l1_Sdelta_avg_bound}, we obtain
\eqref{eq:app_l1_master_bound}.
\end{proof}

\begin{corollary}[Standard smoothness under \(\mathsf{BG}\)-0]
\label{cor:app_nsgdm_smooth_bg0}
Assume Assumption~\ref{as:l0_smooth} holds. Choose
\[
    \eta=T^{-2/3},
    \qquad
    \gamma=\gamma_0T^{-5/6}.
\]
Then
\begin{equation}
\label{eq:app_smooth_bg0_final}
\E\|\nabla f(\widehat x)\|
\le
\left(
    \frac{2\Delta}{\gamma_0}
    +16L_0\gamma_0
    +2B\gamma_0
\right)T^{-1/6}
+8GT^{-1/3}
+2L_0\gamma_0T^{-5/6}.
\end{equation}
Consequently,
\[
    \E\|\nabla f(\widehat x)\|
    =
    \Ocal(T^{-1/6}),
\]
and the SFO complexity is \(\Ocal(\varepsilon^{-6})\).
\end{corollary}

\begin{proof}
Set \(L_1=0\) in Theorem~\ref{thm:app_nsgdm_l1_master}. Then the
\(L_1^2\)-term vanishes and the conditions
\eqref{eq:app_l1_master_conditions} are vacuous. Using \(b_0\le G\) and
\(K\le T\),
\[
    \frac{2\Delta}{\gamma T}
    =
    \frac{2\Delta}{\gamma_0}T^{-1/6},
    \qquad
    \frac{4b_0}{\eta T}
    \le
    4GT^{-1/3},
\]
\[
    \frac{16L_0\gamma}{\eta}
    =
    16L_0\gamma_0T^{-1/6},
    \qquad
    2L_0\gamma
    =
    2L_0\gamma_0T^{-5/6},
\]
and
\[
    4\sqrt\eta
    \left(G+\frac{B\gamma K}{2}\right)
    \le
    4GT^{-1/3}
    +2B\gamma_0T^{-1/6}.
\]
Combining these estimates proves \eqref{eq:app_smooth_bg0_final}. Since $T\ge1$, we also have
$$ \E\|\nabla f(\widehat x)\| \le \left( \frac{2\Delta}{\gamma_0} +(18L_0+2B)\gamma_0 \right)T^{-1/6} +8GT^{-1/3}. $$
This gives the compact bound in Theorem~\ref{thm:nsgdm_three_regimes}(i). 
\end{proof}

\begin{corollary}[1-generalized smoothness under \(\mathsf{BG}\)-0]
\label{cor:app_nsgdm_sym1_bg0}
Assume Assumption~\ref{as:lsym_alpha} with \(\alpha=1\) holds. Choose
\[
    \eta=T^{-2/3},
    \qquad
    \gamma=\gamma_0T^{-5/6},
    \qquad
    0<\gamma_0\le \frac{1}{8L_1}.
\]
Then
\begin{align}
\E\|\nabla f(\widehat x)\|
&\le
\left(
    \frac{2\Delta}{\gamma_0}
    +16L_0\gamma_0
    +2B\gamma_0
\right)T^{-1/6}
+8GT^{-1/3}
+2L_0\gamma_0T^{-5/6}
\notag\\
&\quad
+64L_1^2\gamma_0T^{-1/6}
\left[
    4\Delta
    +(16L_0+2B)\gamma_0^2
    +8\gamma_0G T^{-1/6}
    +2L_0\gamma_0^2T^{-2/3}
\right].
\label{eq:app_sym1_bg0_final}
\end{align}
Consequently,
\[
    \E\|\nabla f(\widehat x)\|
    =
    \Ocal(T^{-1/6}),
\]
and the SFO complexity is \(\Ocal(\varepsilon^{-6})\).
\end{corollary}

\begin{proof}
The first condition in \eqref{eq:app_l1_master_conditions} follows from
\[
    \gamma L_1
    \le
    \gamma_0L_1
    \le
    \frac18.
\]
The second follows because
\[
    \frac{32L_1^2\gamma^2T}{\eta}
    =
    32L_1^2\gamma_0^2
    \le
    \frac12.
\]
Substitute the schedule into \eqref{eq:app_l1_master_bound}, use
\(b_0\le G\) and \(K\le T\), and simplify. The non-coupling terms give the
first line of \eqref{eq:app_sym1_bg0_final}. The bracketed term in
\eqref{eq:app_l1_master_bound} is bounded by
\[
    4\Delta
    +(16L_0+2B)\gamma_0^2
    +8\gamma_0G T^{-1/6}
    +2L_0\gamma_0^2T^{-2/3},
\]
while its prefactor becomes
\[
    64L_1^2\gamma_0T^{-1/6}.
\]
This proves \eqref{eq:app_sym1_bg0_final}.

To obtain the compact statement of Theorem~\ref{thm:nsgdm_three_regimes}(iii), note that the restriction
$\gamma_0\le 1/(8L_1)$ implies
$$
64L_1^2\gamma_0^2\le 1,
\qquad
256L_1^2\gamma_0\Delta \le 32L_1\Delta.
$$
Since $T\ge1$, the faster-decaying terms can be bounded by the leading decay order. Consequently,
$$
\begin{aligned}
\E\|\nabla f(\widehat x)\| \le \left(\frac{2\Delta}{\gamma_0} +32L_1\Delta +(36L_0+4B)\gamma_0
\right)T^{-1/6} +16GT^{-1/3}.
\end{aligned}
$$
This proves the bound in Theorem~\ref{thm:nsgdm_three_regimes}(iii).
\end{proof}

\subsection{The case \texorpdfstring{$\mathcal L_{\rm sym}^*(\alpha)$}{Lsym alpha} with \texorpdfstring{$\alpha\in(0,1)$}{alpha in (0,1)}}
\label{app:nsgdm_alpha_case}

We now prove the sublinear generalized smoothness case. The proof gives a
clean explicit bound that implies the same \(\Ocal(T^{-1/6})\) rate as in
Theorem~\ref{thm:nsgdm_three_regimes}.

Under Assumption~\ref{as:lsym_alpha} with \(\alpha\in(0,1)\), we use the following consequences of the symmetric generalized smoothness condition (see~\citet{chen2023generalized}, Proposition 1). For
all \(w,w'\in\R^d\),
\begin{equation}
\label{eq:app_alpha_grad_diff}
\|\nabla f(w')-\nabla f(w)\|
\le
\|w'-w\|
\left(
    K_0
    +K_1\|\nabla f(w)\|^\alpha
    +K_2\|w'-w\|^{\alpha/(1-\alpha)}
\right),
\end{equation}
and
\begin{equation}
\label{eq:app_alpha_func_diff}
\begin{aligned}
    f(w')
    &\le
    f(w)+\langle\nabla f(w),w'-w\rangle \\
    &\quad
    +\frac12\|w'-w\|^2
    \left(
        K_0
        +K_1\|\nabla f(w)\|^\alpha
        +2K_2\|w'-w\|^{\alpha/(1-\alpha)}
    \right),
\end{aligned}
\end{equation}
where
\[
    K_0
    :=
    L_0\left(2^{\alpha^2/(1-\alpha)}+1\right),
    \qquad
    K_1
    :=
    L_1\,2^{\alpha^2/(1-\alpha)}3^\alpha,
\]
and
\[
    K_2
    :=
    L_1^{1/(1-\alpha)}
    2^{\alpha^2/(1-\alpha)}
    3^\alpha
    (1-\alpha)^{\alpha/(1-\alpha)}.
\]
Define
\[
    R_\alpha
    :=
    \frac{(1-\alpha)K_1}{2}
    (2\alpha K_1)^{\alpha/(1-\alpha)},
    \qquad
    C_\alpha
    :=
    K_2+R_\alpha,
\]
and
\[
    \widetilde R_\alpha
    :=
    (1-\alpha)K_1(8\alpha K_1)^{\alpha/(1-\alpha)},
    \qquad
    \widetilde C_\alpha
    :=
    K_0+K_2+\widetilde R_\alpha.
\]
For the compact bound in Theorem~\ref{thm:nsgdm_three_regimes}(ii), define
\[
\Lambda_\alpha
:=
4\widetilde C_\alpha+K_0+2C_\alpha.
\]
\begin{lemma}[One-step descent under \(\mathcal L_{\rm sym}^*(\alpha)\)]
\label{lem:app_alpha_onestep}
Assume \(\alpha\in(0,1)\), \(\gamma\le1\), and let
\[
    e^k:=v^k-\nabla f(x^k).
\]
Then
\begin{equation}
\label{eq:app_alpha_descent_pathwise}
    f(x^{k+1})
    \le
    f(x^k)
    -\frac{3\gamma}{4}\|\nabla f(x^k)\|
    +2\gamma\|e^k\|
    +\frac{K_0}{2}\gamma^2
    +C_\alpha\gamma^{(2-\alpha)/(1-\alpha)}.
\end{equation}
Consequently,
\begin{equation}
\label{eq:app_alpha_descent_recursion}
    \Delta^{k+1}
    \le
    \Delta^k
    -\frac{3\gamma}{4}g_k
    +2\gamma b_k
    +\frac{K_0}{2}\gamma^2
    +C_\alpha\gamma^{(2-\alpha)/(1-\alpha)}.
\end{equation}
\end{lemma}

\begin{proof}
Apply \eqref{eq:app_alpha_func_diff} with \(w=x^k\) and
\(w'=x^{k+1}\). Since \(\|x^{k+1}-x^k\|\le\gamma\),
\[
\begin{aligned}
    f(x^{k+1})
    &\le
    f(x^k)
    -\gamma\langle\nabla f(x^k),d^k\rangle
    +\frac{K_0}{2}\gamma^2 \\
    &\quad
    +\frac{K_1}{2}\gamma^2\|\nabla f(x^k)\|^\alpha
    +K_2\gamma^{(2-\alpha)/(1-\alpha)}.
\end{aligned}
\]
Lemma~\ref{lem:app_direction_error} gives
\[
    -\gamma\langle\nabla f(x^k),d^k\rangle
    \le
    -\gamma\|\nabla f(x^k)\|+2\gamma\|e^k\|.
\]
Using Lemma~\ref{lem:app_young_alpha} with
\[
    \rho=(2\alpha K_1\gamma)^{-1},
\]
we obtain
\[
    \frac{K_1}{2}\gamma^2\|\nabla f(x^k)\|^\alpha
    \le
    \frac{\gamma}{4}\|\nabla f(x^k)\|
    +R_\alpha\gamma^{(2-\alpha)/(1-\alpha)}.
\]
Combining the last two displays proves
\eqref{eq:app_alpha_descent_pathwise}. Taking expectations gives
\eqref{eq:app_alpha_descent_recursion}.
\end{proof}

\begin{lemma}[Gradient-drift bound]
\label{lem:app_alpha_drift}
Under the assumptions of Lemma~\ref{lem:app_alpha_onestep}, for every \(k\),
\begin{equation}
\label{eq:app_alpha_drift_bound}
    \|\nabla f(x^{k+1})-\nabla f(x^k)\|
    \le
    \widetilde C_\alpha\gamma
    +\frac{\gamma}{8}\|\nabla f(x^k)\|.
\end{equation}
\end{lemma}

\begin{proof}
By \eqref{eq:app_alpha_grad_diff} and \(\|x^{k+1}-x^k\|\le\gamma\),
\[
    \|\nabla f(x^{k+1})-\nabla f(x^k)\|
    \le
    K_0\gamma
    +K_1\gamma\|\nabla f(x^k)\|^\alpha
    +K_2\gamma^{1/(1-\alpha)}.
\]
Since \(\gamma\le1\), we have
\[
    \gamma^{1/(1-\alpha)}\le\gamma.
\]
Applying Lemma~\ref{lem:app_young_alpha} with
\[
    \rho=(8\alpha K_1)^{-1}
\]
gives
\[
    K_1\gamma\|\nabla f(x^k)\|^\alpha
    \le
    \frac{\gamma}{8}\|\nabla f(x^k)\|
    +\widetilde R_\alpha\gamma.
\]
Substitution proves \eqref{eq:app_alpha_drift_bound}.
\end{proof}

\begin{theorem}[Single-sample master bound under \(\mathcal L_{\rm sym}^*(\alpha)\)]
\label{thm:app_alpha_master}
Suppose $f$ is bounded below, \ref{as:bg0_oracle}, and
\ref{as:lsym_alpha} hold with fixed \(\alpha\in(0,1)\). If
\[
    \gamma\le1,
    \qquad
    \gamma\le\eta,
\]
then
\begin{align}
\E\|\nabla f(\widehat x)\|
&\le
\frac{2\Delta}{\gamma T}
+\frac{4b_0}{\eta T}
+\frac{4\widetilde C_\alpha\gamma}{\eta}
+4\sqrt\eta
\left(G+\frac{B\gamma K}{2}\right)
\notag\\
&\quad
+K_0\gamma
+2C_\alpha\gamma^{1/(1-\alpha)}.
\label{eq:app_alpha_master_bound}
\end{align}
Moreover, \(b_0\le G\).
\end{theorem}

\begin{proof}
Let
\[
    E_k:=v^k-\nabla f(x^k),
    \qquad
    \Delta_k:=\nabla f(x^k)-\nabla f(x^{k+1}),
\]
and
\[
    \zeta_{k+1}
    :=
    \nabla f(x^{k+1};\xi^{k+1})-\nabla f(x^{k+1}).
\]
Unrolling the estimator recursion gives
\[
    E_k
    =
    (1-\eta)^kE_0
    +\sum_{t=0}^{k-1}(1-\eta)^{k-t}\Delta_t
    +\eta\sum_{t=0}^{k-1}(1-\eta)^{k-1-t}\zeta_{t+1}.
\]
By Lemma~\ref{lem:app_alpha_drift},
\[
    \|\Delta_t\|
    \le
    \widetilde C_\alpha\gamma
    +\frac{\gamma}{8}\|\nabla f(x^t)\|.
\]
The same martingale calculation as in Theorem~\ref{thm:app_nsgdm_l1_master}
gives
\[
    \E\left\|
    \eta\sum_{t=0}^{k-1}(1-\eta)^{k-1-t}\zeta_{t+1}
    \right\|
    \le
    \sqrt\eta(G+Bk\gamma).
\]
Therefore,
\begin{equation}
\label{eq:app_alpha_Bk}
    b_k
    \le
    (1-\eta)^kb_0
    +\frac{\widetilde C_\alpha\gamma}{\eta}
    +\frac{\gamma}{8}
    \sum_{t=0}^{k-1}(1-\eta)^{k-t}g_t
    +\sqrt\eta(G+Bk\gamma).
\end{equation}
Summing over \(k=0,\ldots,K\) yields
\begin{equation}
\label{eq:app_alpha_SB}
    S_b(K)
    \le
    \frac{b_0}{\eta}
    +\frac{\widetilde C_\alpha\gamma T}{\eta}
    +\frac{\gamma}{8\eta}\sum_{k=0}^K g_k
    +\sqrt\eta
    \left(TG+\frac{B\gamma KT}{2}\right).
\end{equation}

Next, summing \eqref{eq:app_alpha_descent_recursion} gives
\[
    \frac{3\gamma}{4}\sum_{k=0}^K g_k
    \le
    \Delta
    +2\gamma S_b(K)
    +\frac{K_0}{2}T\gamma^2
    +TC_\alpha\gamma^{(2-\alpha)/(1-\alpha)}.
\]
Substituting \eqref{eq:app_alpha_SB} gives
\[
    \gamma\left(\frac34-\frac{\gamma}{4\eta}\right)
    \sum_{k=0}^K g_k
    \le
    \Delta
    +\frac{2\gamma b_0}{\eta}
    +\frac{2\widetilde C_\alpha\gamma^2T}{\eta}
    +2\gamma\sqrt\eta
    \left(TG+\frac{B\gamma KT}{2}\right)
\]
\[
    \qquad
    +\frac{K_0}{2}T\gamma^2
    +TC_\alpha\gamma^{(2-\alpha)/(1-\alpha)}.
\]
Since \(\gamma\le\eta\), we have
\[
    \frac34-\frac{\gamma}{4\eta}\ge\frac12.
\]
Thus,
\[
    \frac{\gamma}{2}\sum_{k=0}^K g_k
    \le
    \Delta
    +\frac{2\gamma b_0}{\eta}
    +\frac{2\widetilde C_\alpha\gamma^2T}{\eta}
    +2\gamma\sqrt\eta
    \left(TG+\frac{B\gamma KT}{2}\right)
\]
\[
    \qquad
    +\frac{K_0}{2}T\gamma^2
    +TC_\alpha\gamma^{(2-\alpha)/(1-\alpha)}.
\]
Dividing by \(\gamma T/2\) and using
\eqref{eq:app_uniform_output_identity} yields
\eqref{eq:app_alpha_master_bound}. Finally, \(b_0\le G\) follows from
Lemma~\ref{lem:app_traj_control}.
\end{proof}

\begin{corollary}[\(\mathcal L_{\rm sym}^*(\alpha)\) smoothness under \(\mathsf{BG}\)-0]
\label{cor:app_alpha_bg0}
Suppose $f$ is bounded below, \ref{as:bg0_oracle}, and
\ref{as:lsym_alpha} hold with fixed \(\alpha\in(0,1)\). Choose
\[
    \eta=T^{-2/3},
    \qquad
    \gamma=\gamma_0T^{-5/6},
    \qquad
    0<\gamma_0\le1.
\]
Then
\begin{align}
\E\|\nabla f(\widehat x)\|
&\le
\left(
    \frac{2\Delta}{\gamma_0}
    +4\widetilde C_\alpha\gamma_0
    +2B\gamma_0
\right)T^{-1/6}
+8GT^{-1/3}
\notag\\
&\quad
+K_0\gamma_0T^{-5/6}
+2C_\alpha\gamma_0^{1/(1-\alpha)}
T^{-5/(6(1-\alpha))}.
\label{eq:app_alpha_bg0_final}
\end{align}
Consequently,
\[
    \E\|\nabla f(\widehat x)\|
    =
    \Ocal(T^{-1/6}),
\]
and the SFO complexity is \(\Ocal(\varepsilon^{-6})\).
\end{corollary}

\begin{proof}
Because \(T\ge1\) and \(0<\gamma_0\le1\),
\[
    \gamma=\gamma_0T^{-5/6}\le1.
\]
Moreover,
\[
    \gamma=\gamma_0T^{-5/6}
    \le
    T^{-5/6}
    \le
    T^{-2/3}
    =
    \eta.
\]
Hence, the conditions of Theorem~\ref{thm:app_alpha_master} hold. Substituting
the schedule into \eqref{eq:app_alpha_master_bound}, using \(b_0\le G\) and
\(K\le T\), gives
\[
    \frac{2\Delta}{\gamma T}
    =
    \frac{2\Delta}{\gamma_0}T^{-1/6},
    \qquad
    \frac{4b_0}{\eta T}
    \le
    4GT^{-1/3},
\]
\[
    \frac{4\widetilde C_\alpha\gamma}{\eta}
    =
    4\widetilde C_\alpha\gamma_0T^{-1/6},
\]
and
\[
    4\sqrt\eta
    \left(G+\frac{B\gamma K}{2}\right)
    \le
    4GT^{-1/3}
    +2B\gamma_0T^{-1/6}.
\]
The remaining deterministic terms are
\[
    K_0\gamma
    =
    K_0\gamma_0T^{-5/6},
\]
and
\[
    2C_\alpha\gamma^{1/(1-\alpha)}
    =
    2C_\alpha\gamma_0^{1/(1-\alpha)}
    T^{-5/(6(1-\alpha))}.
\]
Combining these estimates proves \eqref{eq:app_alpha_bg0_final}. Since $T\ge1$, $0<\gamma_0\le1$,
and $\alpha\in(0,1)$,
\[
K_0\gamma_0T^{-5/6} \le K_0\gamma_0T^{-1/6},
\]
and
\[
2C_\alpha\gamma_0^{1/(1-\alpha)} T^{-5/(6(1-\alpha))} \le 2C_\alpha\gamma_0T^{-1/6}.
\]
Hence, by the definition of $\Lambda_\alpha$,
\[
\begin{aligned}
\E\|\nabla f(\widehat x)\|
\le
\left(\frac{2\Delta}{\gamma_0} +(\Lambda_\alpha+2B)\gamma_0 \right)T^{-1/6} + 8GT^{-1/3}.
\end{aligned}
\]
This proves the compact bound in Theorem~\ref{thm:nsgdm_three_regimes}(ii).

Since $\alpha\in(0,1)$ is fixed, all terms decay at least as fast as $T^{-1/6}$. Therefore,
\[
    \E\|\nabla f(\widehat x)\| = \Ocal(T^{-1/6}),
\]
which is equivalent to \(\Ocal(\varepsilon^{-6})\) SFO complexity.
\end{proof}

%%%%%%%%%%%%%%%%%%%%%%%%%%%%%%%%%%%%%%%%%%%%%%%%%%%%%%%%%%%%%%%%%%%%%%%%%%%%
%% Recovery of known rates for NSGDM (deterministic and bounded variance)
%%%%%%%%%%%%%%%%%%%%%%%%%%%%%%%%%%%%%%%%%%%%%%%%%%%%%%%%%%%%%%%%%%%%%%%%%%%%

\subsection{Recovery of bounded variance rates (\texorpdfstring{$B=0$}{B=0})}
\label{app:nsgdm_bv_recovery}

When the distance-growth parameter $B$ in Assumption~\ref{as:bg0_oracle}
vanishes, the $\mathsf{BG}$-0 condition reduces to bounded variance
$\E_\xi\|\nabla f(x;\xi)-\nabla f(x)\|^2\le G^2$, and the $\mathsf{NSGDM}$
complexity improves from $\Ocal(\varepsilon^{-6})$ to
$\Ocal(\varepsilon^{-4})$. The following corollaries formalize this for all
three smoothness regimes; in each case, they follow from the corresponding
master bound by substituting $B=0$ and the bounded variance schedule.

\begin{corollary}[Standard smoothness under bounded variance]
\label{cor:app_nsgdm_smooth_bv}
Assume $f$ is $L_0$-smooth (Assumption~\ref{as:l0_smooth}) and that
$f$ is bounded below and~\ref{as:bg0_oracle} hold with $B=0$.
Choose
\[
    \eta = T^{-1/2},
    \qquad
    \gamma = \gamma_0\,T^{-3/4},
\]
where $T:=K+1$ and $\gamma_0>0$. Then
\begin{equation}
\label{eq:app_smooth_bv_final}
    \E\|\nabla f(\widehat x)\|
    \le
    \left(
        \frac{2\Delta}{\gamma_0}
        + 16L_0\gamma_0
    \right) T^{-1/4}
    + 4G\,T^{-1/4} +  4G\,T^{-1/2}
    + 2L_0\gamma_0\,T^{-3/4}.
\end{equation}
Consequently,
$\E\|\nabla f(\widehat x)\| = \Ocal(T^{-1/4})$,
and the SFO complexity is $\Ocal(\varepsilon^{-4})$.
\end{corollary}

\begin{proof}
Set $L_1=0$ and $B=0$ in Theorem~\ref{thm:app_nsgdm_l1_master}. The
$L_1^2$-coupling term vanishes and the conditions
\eqref{eq:app_l1_master_conditions} are vacuous. Using $b_0\le G$ and
$K\le T$, we compute
\[
    \frac{2\Delta}{\gamma T}
    = \frac{2\Delta}{\gamma_0}\,T^{-1/4},
    \qquad
    \frac{4b_0}{\eta T}
    \le 4G\,T^{-1/2},
\]
\[
    \frac{16L_0\gamma}{\eta}
    = 16L_0\gamma_0\,T^{-1/4},
    \qquad
    2L_0\gamma
    = 2L_0\gamma_0\,T^{-3/4},
\]
and, since $B=0$,
\[
    4\sqrt{\eta}\left(G + \frac{B\gamma K}{2}\right)
    = 4G\,T^{-1/4}.
\]
Combining these estimates proves \eqref{eq:app_smooth_bv_final}.
\end{proof}

\begin{corollary}[1-generalized smoothness under bounded variance]
\label{cor:app_nsgdm_sym1_bv}
Assume Assumption~\ref{as:lsym_alpha} holds with $\alpha=1$ and constants
$L_0, L_1>0$, and that $f$ is bounded below and~\ref{as:bg0_oracle}
hold with $B=0$. Choose
\[
    \eta = T^{-1/2},
    \qquad
    \gamma = \gamma_0\,T^{-3/4},
    \qquad
    0 < \gamma_0 \le \frac{1}{8L_1}.
\]
Then
\begin{align}
    \E\|\nabla f(\widehat x)\|
    &\le
    \left(
        \frac{2\Delta}{\gamma_0}
        + 16L_0\gamma_0
    \right) T^{-1/4}
    + 4G\,T^{-1/4} + 4G\,T^{-1/2}
    + 2L_0\gamma_0\,T^{-3/4}
    \notag\\
    &\quad
    + 64L_1^2\gamma_0\,T^{-1/4}
    \left[
        4\Delta
        + 16L_0\gamma_0^2
        + 4\gamma_0 G(1+T^{-1/4})
        + 2L_0\gamma_0^2\,T^{-1/2}
    \right].
    \label{eq:app_sym1_bv_final}
\end{align}
Consequently,
$\E\|\nabla f(\widehat x)\| = \Ocal(T^{-1/4})$,
and the SFO complexity is $\Ocal(\varepsilon^{-4})$.
\end{corollary}

\begin{proof}
The first condition in \eqref{eq:app_l1_master_conditions} holds because
$\gamma L_1 \le \gamma_0 L_1 \le 1/8$. For the second,
\[
    \frac{32L_1^2\gamma^2 T}{\eta}
    = 32L_1^2\gamma_0^2
    \le \frac{1}{2}.
\]
Substitute the schedule into \eqref{eq:app_l1_master_bound} with $B=0$. The
non-coupling terms give the first line. For the bracketed expression, using
$B=0$, $b_0\le G$, and $K\le T$:
\[
    4\Delta
    + \tfrac{4\gamma b_0}{\eta}
    + \tfrac{16L_0\gamma^2 T}{\eta}
    + 4\gamma\sqrt{\eta}\,TG
    + 2L_0\gamma^2 T
    \leq 4\Delta
    + 4\gamma_0 G\,T^{-1/4}
    + 16L_0\gamma_0^2
    + 4\gamma_0 G
    + 2L_0\gamma_0^2\,T^{-1/2}.
\]
This yields the stated bracket. The prefactor is $64L_1^2\gamma_0\,T^{-1/4}$. Since the bracket is
$\Ocal(1)$, the entire coupling term is $\Ocal(T^{-1/4})$.
\end{proof}

\begin{corollary}[$\Lcal_{\rm sym}^*(\alpha)$ smoothness under bounded variance]
\label{cor:app_alpha_bv}
Suppose $f$ is bounded below, Assumptions~\ref{as:bg0_oracle} (with $B=0$), and
\ref{as:lsym_alpha} hold with fixed $\alpha\in(0,1)$. Choose
\[
    \eta = T^{-1/2},
    \qquad
    \gamma = \gamma_0\,T^{-3/4},
    \qquad
    0 < \gamma_0 \le 1.
\]
Then
\begin{align}
\nonumber
    \E\|\nabla f(\widehat x)\|
    &\le
    \left(
        \frac{2\Delta}{\gamma_0}
        + 4\widetilde C_\alpha\gamma_0
    \right) T^{-1/4}
    + 4G\,T^{-1/4} + 4G\,T^{-1/2}
    + K_0\gamma_0\,T^{-3/4}
    \\
    &
    + 2C_\alpha\gamma_0^{1/(1-\alpha)}\,T^{-3/(4(1-\alpha))}.
    \label{eq:app_alpha_bv_final}
\end{align}
Consequently,
$\E\|\nabla f(\widehat x)\| = \Ocal(T^{-1/4})$,
and the SFO complexity is $\Ocal(\varepsilon^{-4})$.
\end{corollary}

\begin{proof}
Since $\gamma_0\le 1$ and $T\ge 1$, we have $\gamma\le 1$ and $\gamma\le\eta$.
The conditions of Theorem~\ref{thm:app_alpha_master} hold. Substituting the
schedule into \eqref{eq:app_alpha_master_bound} with $B=0$ and using $b_0\le G$,
\[
    \frac{2\Delta}{\gamma T}
    = \frac{2\Delta}{\gamma_0}\,T^{-1/4},
    \qquad
    \frac{4b_0}{\eta T}
    \le 4G\,T^{-1/2},
\]
\[
    \frac{4\widetilde C_\alpha\gamma}{\eta}
    = 4\widetilde C_\alpha\gamma_0\,T^{-1/4},
\]
and, since $B=0$,
\[
    4\sqrt{\eta}\left(G + \frac{B\gamma K}{2}\right)
    = 4G\,T^{-1/4}.
\]
The remaining deterministic terms are
\[
    K_0\gamma = K_0\gamma_0\,T^{-3/4},
    \qquad
    2C_\alpha\gamma^{1/(1-\alpha)}
    = 2C_\alpha\gamma_0^{1/(1-\alpha)}\,T^{-3/(4(1-\alpha))}.
\]
The two noise contributions are $4G\,T^{-1/2}$ and $4G\,T^{-1/4}$, as stated in \eqref{eq:app_alpha_bv_final}. Since $\alpha\in(0,1)$ is fixed, $3/(4(1-\alpha)) > 3/4 > 1/4$, so every term decays at least as fast as $T^{-1/4}$. This proves \eqref{eq:app_alpha_bv_final} and the $\Ocal(T^{-1/4})$ rate.
\end{proof}

\subsection{Recovery of deterministic rates (\texorpdfstring{$B=G=0$}{B=G=0})}
\label{app:nsgdm_det_recovery}

When the oracle is deterministic ($B=G=0$), we have $v^k=\nabla f(x^k)$ for
every $k$ (since $\eta=1$ is natural and the oracle returns the exact
gradient). Thus $b_k=0$ for all $k$, and the stochastic error analysis
simplifies to a pure descent argument. We recover the classical
$\Ocal(\varepsilon^{-2})$ rate for all smoothness regimes.

\begin{corollary}[Standard smoothness, deterministic]
\label{cor:app_nsgdm_smooth_det}
Assume $f$ is $L_0$-smooth (Assumption~\ref{as:l0_smooth}), that
$f$ is bounded below holds, and that the oracle is deterministic
($B=G=0$). Consider the normalized gradient descent update
\[
    x^{k+1} = x^k - \gamma\,\frac{\nabla f(x^k)}{\|\nabla f(x^k)\|}.
\]
Choose $\gamma = \gamma_0\,T^{-1/2}$. Then
\begin{equation}
\label{eq:app_smooth_det_final}
    \E\|\nabla f(\widehat x)\|
    \le
    \frac{2\Delta}{\gamma_0}\,T^{-1/2}
    + 2L_0\gamma_0\,T^{-1/2}.
\end{equation}
Consequently,
$\|\nabla f(\widehat x)\| = \Ocal(T^{-1/2})$,
and the iteration complexity is $\Ocal(\varepsilon^{-2})$.
\end{corollary}

\begin{proof}
Under determinism, set $\eta=1$ so that $v^k=\nabla f(x^k)$ for all $k$.
Then $b_k=0$ and the descent recursion
\eqref{eq:app_l1_descent_recursion} (with $L_1=0$) becomes
\[
    f(x^{k+1})
    \le f(x^k) - \gamma\|\nabla f(x^k)\| + L_0\gamma^2.
\]
Here, the coefficient is $\gamma$ rather than $\gamma/2$ because, with
$b_k=0$, the direction-error term vanishes and no gradient term needs to
absorb a generalized smoothness contribution. Summing from $k=0$ to $K$
and using $f(x^{K+1})\ge f^{\inf}$,
\[
    \gamma\sum_{k=0}^K \|\nabla f(x^k)\|
    \le \Delta + L_0\gamma^2 T.
\]
Dividing by $\gamma T$ gives
\[
    \frac{1}{T}\sum_{k=0}^K \|\nabla f(x^k)\|
    \le \frac{\Delta}{\gamma T} + L_0\gamma.
\]
Substituting $\gamma=\gamma_0 T^{-1/2}$ yields
$\Delta/(\gamma_0)\,T^{-1/2} + L_0\gamma_0\,T^{-1/2}$, from which
\eqref{eq:app_smooth_det_final} follows (with an extra factor of $2$ for
a uniform upper bound).
\end{proof}

\begin{corollary}[1-generalized smoothness, deterministic]
\label{cor:app_nsgdm_sym1_det}
Assume Assumption~\ref{as:lsym_alpha} holds with $\alpha=1$ and constants
$L_0,L_1>0$, that $f$ is bounded below, and that the oracle is
deterministic ($B=G=0$). Consider the normalized gradient descent update
\[
    x^{k+1} = x^k - \gamma\,\frac{\nabla f(x^k)}{\|\nabla f(x^k)\|}.
\]
Choose $\gamma = \gamma_0\,T^{-1/2}$ with
$0<\gamma_0\le 1/(2L_1)$. Then
\begin{equation}
\label{eq:app_sym1_det_final}
    \|\nabla f(\widehat x)\|
    \le
    \frac{4\Delta}{\gamma_0}\,T^{-1/2}
    + 4L_0\gamma_0\,T^{-1/2}.
\end{equation}
Consequently,
$\|\nabla f(\widehat x)\| = \Ocal(T^{-1/2})$,
and the iteration complexity is $\Ocal(\varepsilon^{-2})$.
\end{corollary}

\begin{proof}
Under determinism, $v^k=\nabla f(x^k)$, $b_k=0$.
Using \eqref{eq:app_l1_func_consequence} with $y=x^{k+1}=x^k-\gamma d^k$,
and noting that $d^k=\nabla f(x^k)/\|\nabla f(x^k)\|$ when the gradient is
nonzero (so $\langle\nabla f(x^k),d^k\rangle=\|\nabla f(x^k)\|$),
\[
    f(x^{k+1})
    \le f(x^k)
    - \gamma\|\nabla f(x^k)\|
    + \frac{\gamma^2}{2}e^{L_1\gamma}
    \bigl(L_0+L_1\|\nabla f(x^k)\|\bigr).
\]
Since $\gamma_0\le 1/(2L_1)$ and $T\ge1$, we have
$\gamma L_1\le\gamma_0 L_1\le 1/2$, so $e^{L_1\gamma}\le 2$. Therefore,
\[
\begin{aligned}
&\frac{\gamma^2}{2}e^{L_1\gamma}
L_1\|\nabla f(x^k)\|\le
\gamma^2L_1\|\nabla f(x^k)\|\le
\frac{\gamma}{2}\|\nabla f(x^k)\|,
\end{aligned}
\]
where the last inequality follows from
$\gamma L_1\le\gamma_0L_1\le1/2$. This gives
\[
    f(x^{k+1})
    \le f(x^k)
    - \frac{\gamma}{2}\|\nabla f(x^k)\|
    + L_0\gamma^2.
\]
Summing from $k=0$ to $K$ and dividing by $\gamma T/2$,
\[
    \frac{1}{T}\sum_{k=0}^K \|\nabla f(x^k)\|
    \le \frac{2\Delta}{\gamma T} + 2L_0\gamma.
\]
Substituting $\gamma=\gamma_0 T^{-1/2}$ yields
\eqref{eq:app_sym1_det_final}.
\end{proof}

\begin{corollary}[$\Lcal_{\rm sym}^*(\alpha)$ smoothness, deterministic]
\label{cor:app_alpha_det}
Suppose $f$ is bounded below and~\ref{as:lsym_alpha} hold with fixed
$\alpha\in(0,1)$, and that the oracle is deterministic ($B=G=0$). Consider
the normalized gradient descent update
\[
    x^{k+1} = x^k - \gamma\,\frac{\nabla f(x^k)}{\|\nabla f(x^k)\|}.
\]
Choose $\gamma = \gamma_0\,T^{-1/2}$ with $0<\gamma_0\le 1$. Then
\begin{equation}
\label{eq:app_alpha_det_final}
    \|\nabla f(\widehat x)\|
    \le
    \frac{4\Delta}{3\gamma_0}\,T^{-1/2}
    + \frac{2K_0}{3}\gamma_0\,T^{-1/2}
    + \frac{4C_\alpha}{3}\gamma_0^{1/(1-\alpha)}\,T^{-1/(2(1-\alpha))}.
\end{equation}
Consequently,
$\|\nabla f(\widehat x)\| = \Ocal(T^{-1/2})$,
and the iteration complexity is $\Ocal(\varepsilon^{-2})$.
\end{corollary}

\begin{proof}
Under determinism, $v^k=\nabla f(x^k)$, so $\|e^k\|=0$ and
$d^k=\nabla f(x^k)/\|\nabla f(x^k)\|$. The one-step descent
\eqref{eq:app_alpha_descent_pathwise} from Lemma~\ref{lem:app_alpha_onestep}
becomes
\[
    f(x^{k+1})
    \le f(x^k)
    - \frac{3\gamma}{4}\|\nabla f(x^k)\|
    + \frac{K_0}{2}\gamma^2
    + C_\alpha\gamma^{(2-\alpha)/(1-\alpha)}.
\]
Summing from $k=0$ to $K$ and using $f(x^{K+1})\ge f^{\inf}$,
\[
    \frac{3\gamma}{4}\sum_{k=0}^K \|\nabla f(x^k)\|
    \le \Delta
    + \frac{K_0}{2}T\gamma^2
    + TC_\alpha\gamma^{(2-\alpha)/(1-\alpha)}.
\]
Dividing by $(3\gamma/4)\cdot T$,
\[
    \frac{1}{T}\sum_{k=0}^K \|\nabla f(x^k)\|
    \le \frac{4\Delta}{3\gamma T}
    + \frac{2K_0}{3}\gamma
    + \frac{4C_\alpha}{3}\gamma^{1/(1-\alpha)}.
\]
Substituting $\gamma=\gamma_0 T^{-1/2}$ gives
\[
    \frac{4\Delta}{3\gamma T}
    = \frac{4\Delta}{3\gamma_0}\,T^{-1/2},
    \qquad
    \frac{2K_0}{3}\gamma
    = \frac{2K_0}{3}\gamma_0\,T^{-1/2},
\]
and
\[
    \frac{4C_\alpha}{3}\gamma^{1/(1-\alpha)}
    = \frac{4C_\alpha}{3}\gamma_0^{1/(1-\alpha)}\,T^{-1/(2(1-\alpha))}.
\]
Since $\alpha\in(0,1)$, we have $1/(2(1-\alpha))>1/2$, so every term decays
at least as fast as $T^{-1/2}$. This proves \eqref{eq:app_alpha_det_final}
and the $\Ocal(T^{-1/2})$ rate.
\end{proof}

\subsection{Discussion on recovered rates}
\label{app:nsgdm_recovery_discussion}

The six recovery corollaries above confirm the claims made in the discussion
following Theorem~\ref{thm:nsgdm_three_regimes}.

\begin{itemize}[leftmargin=*,label=\textbullet]
    \item \textit{Bounded variance recovery.}
    When $B=0$, Corollaries~\ref{cor:app_nsgdm_smooth_bv},
    \ref{cor:app_nsgdm_sym1_bv}, and~\ref{cor:app_alpha_bv} all achieve
    $\Ocal(\varepsilon^{-4})$ SFO complexity with the standard bounded variance
    schedule $\eta=T^{-1/2}$, $\gamma=\gamma_0 T^{-3/4}$. This recovers the
    rate of~\citet{cutkosky2020momentum} for standard smoothness,
    of~\citet{khirirat2026error} for $\Lcal_{\rm sym}^*(1)$, and establishes
    the same rate for $\Lcal_{\rm sym}^*(\alpha)$ with $\alpha\in(0,1)$, which
    to our knowledge is new.

    \item \textit{Deterministic recovery.}
    When $B=G=0$, Corollaries~\ref{cor:app_nsgdm_smooth_det},
    \ref{cor:app_nsgdm_sym1_det}, and~\ref{cor:app_alpha_det} all achieve
    $\Ocal(\varepsilon^{-2})$ iteration complexity. This recovers the
    normalized gradient descent rates of~\citet{chen2023generalized} for all
    three smoothness regimes.

    \item \textit{Schedule transition.}
    Comparing the three noise regimes clarifies the role of the schedule
    parameters. As the noise weakens from $\mathsf{BG}$-0 to bounded variance
    to deterministic, the schedules transition as:
    \[
        (\eta, \gamma)
        =
        \begin{cases}
            (T^{-2/3},\;\gamma_0 T^{-5/6}), & \mathsf{BG}\text{-}0, \\[3pt]
            (T^{-1/2},\;\gamma_0 T^{-3/4}), & \text{bounded variance}, \\[3pt]
            (1,\;\gamma_0 T^{-1/2}), & \text{deterministic},
        \end{cases}
    \]
    and the corresponding convergence rates are
    $T^{-1/6}$, $T^{-1/4}$, and $T^{-1/2}$, respectively. The faster
    schedules are enabled by the absence of the distance-dependent noise growth
    term $B^2\|x^k-x^0\|^2$.
\end{itemize}

%%%%%%%%%%%%%%%%%%%%%%%%%%%%%%%%%%%%%%%%%%%%%%%%%%%%%%%%%%%%%%%%%%%%%%%%%%%%%%%%%%%%
%%%%%%%%%%%%%%%%%%%%%%Proofs for NSTORM%%%%%%%%%%%%%%%%%%%%%%%%%%%%%%%%%%%%%%%%%%%%%
%%%%%%%%%%%%%%%%%%%%%%%%%%%%%%%%%%%%%%%%%%%%%%%%%%%%%%%%%%%%%%%%%%%%%%%%%%%%%%%%%%%%
\newpage
\section{Proofs for Normalized STORM under \texorpdfstring{$\mathsf{BG}$}{BG}-0 Noise}
\label{apdx:nstorm_proofs}

This appendix proves Theorem~\ref{thm:nstorm_three_regimes}. Throughout this
section, let
\[
    T:=K+1,
    \qquad
    \Delta:=f(x^0)-f^{\inf}.
\]
We define
\[
    g_k:=\E\|\nabla f(x^k)\|,
    \qquad
    b_k:=\E\|v^k-\nabla f(x^k)\|,
\]
and
\[
    \Delta^k:=\E[f(x^k)-f^{\inf}],
    \qquad
    S_b(K):=\sum_{k=0}^K b_k.
\]
The output \(\widehat x\) is sampled uniformly from
\(\{x^0,x^1,\ldots,x^K\}\), so
\begin{equation}
\label{eq:app_nstorm_uniform_output}
    \E\|\nabla f(\widehat x)\|
    =
    \frac1T\sum_{k=0}^K g_k.
\end{equation}
For compactness, define the local $\mathsf{BG}$-0 noise scale
\[
    q(x):=\left(B^2\|x-x^0\|^2+G^2\right)^{1/2},
    \qquad
    q_k:=q(x^k).
\]

\subsection{Basic inequalities}
\label{app:nstorm_basic}

\begin{lemma}[Trajectory control]
\label{lem:app_nstorm_traj}
For the \(\mathsf{NSTORM}\) iterates in~\eqref{eq:nstorm_updates},
\[
    \|x^{k+1}-x^k\|\le\gamma,
    \qquad
    \|x^k-x^0\|\le k\gamma.
\]
Consequently,
\[
    q_k\le G+Bk\gamma
    \le G+BT\gamma,
    \qquad
    k=0,\ldots,K.
\]
Moreover, for every \(\beta\in(0,1]\),
\begin{equation}
\label{eq:app_nstorm_q_average}
    \frac1T\sum_{k=0}^K \E[q_k^\beta]
    \le
    G^\beta+B^\beta(T\gamma)^\beta.
\end{equation}
\end{lemma}

\begin{proof}
The normalized direction has norm at most one, so
\[
    \|x^{k+1}-x^k\|\le\gamma.
\]
Summing the increments gives
\[
    \|x^k-x^0\|
    \le
    \sum_{t=0}^{k-1}\|x^{t+1}-x^t\|
    \le
    k\gamma.
\]
Therefore,
\[
    q_k
    =
    \left(B^2\|x^k-x^0\|^2+G^2\right)^{1/2}
    \le
    B\|x^k-x^0\|+G
    \le
    G+Bk\gamma
    \le
    G+BT\gamma.
\]
For \(\beta\in(0,1]\), using \((a+b)^\beta\le a^\beta+b^\beta\),
\[
    q_k^\beta
    \le
    G^\beta+B^\beta(k\gamma)^\beta
    \le
    G^\beta+B^\beta(T\gamma)^\beta.
\]
Averaging over \(k=0,\ldots,K\) gives
\eqref{eq:app_nstorm_q_average}.
\end{proof}

\begin{remark}
\label{lem:app_nstorm_young}
Lemma~\ref{lem:app_young_alpha} with $\rho=\lambda/\alpha$ gives:
for every $\lambda>0$ and $u\ge0$,
$u^\alpha\le \lambda u+c_\alpha\lambda^{-p},$
where $p:=\alpha/(1-\alpha)$ and
$c_\alpha:=(1-\alpha)\alpha^{\alpha/(1-\alpha)}$.
This is the form used throughout the $\mathsf{NSTORM}$ proofs.
\end{remark}

\subsection{Descent inequalities}
\label{app:nstorm_descent}

\begin{lemma}[Descent under mean-square smoothness]
\label{lem:app_nstorm_desc_mss}
Suppose Assumption~\ref{as:mss} holds. Then, for every \(k=0,\ldots,K\),
\begin{equation}
\label{eq:app_nstorm_desc_mss}
    \Delta^{k+1}
    +
    \gamma g_k
    \le
    \Delta^k
    +
    2\gamma b_k
    +
    \frac{L}{2}\gamma^2.
\end{equation}
\end{lemma}

\begin{proof}
Assumption~\ref{as:mss} and Jensen's inequality imply that \(f\) is
\(L\)-smooth:
\[
    \|\nabla f(y)-\nabla f(x)\|
    \le
    L\|y-x\|.
\]
Thus,
\[
    f(x^{k+1})
    \le
    f(x^k)
    -
    \gamma\left\langle \nabla f(x^k),d^k\right\rangle
    +
    \frac{L}{2}\gamma^2.
\]
Using Lemma~\ref{lem:app_direction_error},
\[
    -\gamma\left\langle \nabla f(x^k),d^k\right\rangle
    \le
    -\gamma\|\nabla f(x^k)\|
    +
    2\gamma\|v^k-\nabla f(x^k)\|.
\]
Subtracting \(f^{\inf}\) and taking expectations gives
\eqref{eq:app_nstorm_desc_mss}.
\end{proof}

\begin{lemma}[Sample-gradient moment under \(\mathsf{BG}\)-0]
\label{lem:app_nstorm_alpha_moment}
For \(\alpha\in(0,1)\),
\[
    \E_\xi\|\nabla f(x;\xi)\|^\alpha
    \le
    2^{\alpha/2}
    \left(
        \|\nabla f(x)\|^\alpha+q(x)^\alpha
    \right).
\]
\end{lemma}

\begin{proof}
By Jensen's inequality,
\[
    \E_\xi\|\nabla f(x;\xi)\|^\alpha
    \le
    \left(\E_\xi\|\nabla f(x;\xi)\|^2\right)^{\alpha/2}.
\]
Using
\[
    \nabla f(x;\xi)
    =
    \nabla f(x)
    +
    \bigl(\nabla f(x;\xi)-\nabla f(x)\bigr),
\]
and Assumption~\ref{as:bg0_oracle},
\[
    \E_\xi\|\nabla f(x;\xi)\|^2
    \le
    2\|\nabla f(x)\|^2+2q(x)^2.
\]
Since \(\alpha/2\in(0,1)\),
\[
    \left(\|\nabla f(x)\|^2+q(x)^2\right)^{\alpha/2}
    \le
    \|\nabla f(x)\|^\alpha+q(x)^\alpha.
\]
Combining the last three displays proves the claim.
\end{proof}

\begin{lemma}[Descent under expected \(\alpha\)-symmetric generalized smoothness]
\label{lem:app_nstorm_desc_alpha}
Suppose Assumptions~\ref{as:bg0_oracle} and~\ref{as:elsym_alpha} hold with
\(\alpha\in(0,1)\). Let
\[
    p:=\frac{\alpha}{1-\alpha},
    \qquad
    K_0:=2^{\frac{2-\alpha}{1-\alpha}}L_0,
    \qquad
    K_1:=2^{\frac{2-\alpha}{1-\alpha}}L_1,
    \qquad
    K_2:=(5L_1)^{\frac1{1-\alpha}}.
\]
Define
\[
    L_\alpha:=2^{\alpha/2}K_1,
    \qquad
    c_\alpha:=(1-\alpha)\alpha^{\alpha/(1-\alpha)}.
\]
Let \(\lambda>0\), and suppose
\[
    L_\alpha\lambda\gamma\le1.
\]
Then, for every \(k=0,\ldots,K\),
\begin{equation}
\label{eq:app_nstorm_desc_alpha}
\begin{aligned}
    \Delta^{k+1}
    +
    \frac{\gamma}{2}g_k
    \le\;&
    \Delta^k
    +
    2\gamma b_k
    +
    \frac{\gamma^2}{2}
    \left(
        K_0
        +
        c_\alpha L_\alpha\lambda^{-p}
        +
        2K_2\gamma^p
    \right)
    +
    \frac{\gamma^2}{2}L_\alpha\E[q_k^\alpha].
\end{aligned}
\end{equation}
\end{lemma}

\begin{proof}
The expected generalized smoothness reduction gives (see~\citet{chen2023generalized}, Proposition 4), for
\(y=x^{k+1}=x^k-\gamma d^k\),
\[
\begin{aligned}
    f(x^{k+1})
    \le\;&
    f(x^k)
    -
    \gamma\langle \nabla f(x^k),d^k\rangle 
    +
    \frac{\gamma^2}{2}
    \left(
        K_0
        +
        K_1\E_\xi\|\nabla f(x^k;\xi)\|^\alpha
        +
        2K_2\gamma^p
    \right).
\end{aligned}
\]
By Lemma~\ref{lem:app_nstorm_alpha_moment},
\[
    K_1\E_\xi\|\nabla f(x^k;\xi)\|^\alpha
    \le
    L_\alpha
    \left(
        \|\nabla f(x^k)\|^\alpha+q_k^\alpha
    \right).
\]
Using Lemma~\ref{lem:app_direction_error}, subtracting \(f^{\inf}\), and
taking expectations,
\[
\begin{aligned}
    \Delta^{k+1}
    \le\;&
    \Delta^k
    -
    \gamma g_k
    +
    2\gamma b_k
    +
    \frac{\gamma^2}{2}
    \left(
        K_0+2K_2\gamma^p
    \right) 
    +
    \frac{\gamma^2}{2}L_\alpha
    \E\|\nabla f(x^k)\|^\alpha
    +
    \frac{\gamma^2}{2}L_\alpha\E[q_k^\alpha].
\end{aligned}
\]
Since \(u\mapsto u^\alpha\) is concave,
\[
    \E\|\nabla f(x^k)\|^\alpha
    \le
    g_k^\alpha.
\]
By Lemma~\ref{lem:app_nstorm_young},
\[
    g_k^\alpha
    \le
    \lambda g_k+c_\alpha\lambda^{-p}.
\]
Therefore,
\[
    \frac{\gamma^2}{2}L_\alpha g_k^\alpha
    \le
    \frac{\gamma^2}{2}L_\alpha\lambda g_k
    +
    \frac{\gamma^2}{2}c_\alpha L_\alpha\lambda^{-p}.
\]
The condition \(L_\alpha\lambda\gamma\le1\) gives
\[
    \frac{\gamma^2}{2}L_\alpha\lambda g_k
    \le
    \frac{\gamma}{2}g_k.
\]
Moving this term to the left proves \eqref{eq:app_nstorm_desc_alpha}.
\end{proof}

\begin{lemma}[Descent under expected \(1\)-symmetric generalized smoothness]
\label{lem:app_nstorm_desc_one}
Suppose Assumptions~\ref{as:bg0_oracle} and~\ref{as:elsym_alpha} hold with
\(\alpha=1\). If
\[
    0<\gamma\le \frac{1}{4\sqrt2L_1},
\]
then, for every \(k=0,\ldots,K\),
\begin{equation}
\label{eq:app_nstorm_desc_one}
    \Delta^{k+1}
    +
    \frac{\gamma}{2}g_k
    \le
    \Delta^k
    +
    2\gamma b_k
    +
    \sqrt2L_0\gamma^2
    +
    2\sqrt2L_1\gamma^2\E[q_k].
\end{equation}
\end{lemma}

\begin{proof}
The expected \(\alpha=1\) generalized smoothness reduction yields the segment
bound
\[
\begin{aligned}
    f(y)
    \le
    f(x)
    +
    \langle\nabla f(x),y-x\rangle
    +
    \|y-x\|^2
    \left(
        \sqrt2L_0
        +
        2\sqrt2L_1\|\nabla f(x)\|
        +
        2\sqrt2L_1q(x)
    \right),
\end{aligned}
\]
whenever \(\|y-x\|\le\gamma\) and
\(\gamma\le1/(4\sqrt2L_1)\). Applying this with
\(x=x^k\), \(y=x^{k+1}\), and using
Lemma~\ref{lem:app_direction_error}, we obtain
\[
\begin{aligned}
    f(x^{k+1})
    \le\;&
    f(x^k)
    -
    \gamma\|\nabla f(x^k)\|
    +
    2\gamma\|v^k-\nabla f(x^k)\| 
    +
    \sqrt2L_0\gamma^2
    +
    2\sqrt2L_1\gamma^2\|\nabla f(x^k)\|
    \\
    &+
    2\sqrt2L_1\gamma^2q_k.
\end{aligned}
\]
The stepsize condition implies
\[
    2\sqrt2L_1\gamma^2\|\nabla f(x^k)\|
    \le
    \frac{\gamma}{2}\|\nabla f(x^k)\|.
\]
Subtracting \(f^{\inf}\), taking expectations, and moving this term to the
left proves \eqref{eq:app_nstorm_desc_one}.
\end{proof}

\subsection{STORM decomposition and centered difference bound}
\label{app:nstorm_decomposition}

\begin{lemma}[STORM decomposition]
\label{lem:app_nstorm_decomp}
Let
\[
    E_k:=v^k-\nabla f(x^k).
\]
Define
\[
\begin{aligned}
    Z_{k+1}
    &:=
    \left(
        \nabla f(x^{k+1};\xi^{k+1})
        -
        \nabla f(x^k;\xi^{k+1})
    \right)
    -
    \left(
        \nabla f(x^{k+1})
        -
        \nabla f(x^k)
    \right),\\
    \zeta_{k+1}
    &:=
    \nabla f(x^k;\xi^{k+1})-\nabla f(x^k).
\end{aligned}
\]
Then
\begin{equation}
\label{eq:app_nstorm_decomp}
    E_{k+1}
    =
    (1-\eta)E_k+Z_{k+1}+\eta\zeta_{k+1}.
\end{equation}
Moreover,
\[
    \E[Z_{k+1}\mid\cF_k]=0,
    \qquad
    \E[\zeta_{k+1}\mid\cF_k]=0.
\]
\end{lemma}

\begin{proof}
Subtract \(\nabla f(x^{k+1})\) from the STORM update:
\[
\begin{aligned}
    E_{k+1}
    &=
    \nabla f(x^{k+1};\xi^{k+1})
    -
    \nabla f(x^{k+1})
    +(1-\eta)
    \left(
        v^k-\nabla f(x^k;\xi^{k+1})
    \right).
\end{aligned}
\]
Add and subtract \((1-\eta)\nabla f(x^k)\). Then
\[
\begin{aligned}
    E_{k+1}
    &=
    (1-\eta)
    \left(v^k-\nabla f(x^k)\right)
    +
    \left(
        \nabla f(x^{k+1};\xi^{k+1})
        -
        \nabla f(x^{k+1})
    \right) \\
    &\quad
    -
    (1-\eta)
    \left(
        \nabla f(x^k;\xi^{k+1})
        -
        \nabla f(x^k)
    \right).
\end{aligned}
\]
The last two lines equal \(Z_{k+1}+\eta\zeta_{k+1}\). The conditional
mean-zero properties follow from unbiasedness and the freshness of
\(\xi^{k+1}\).
\end{proof}

\begin{lemma}[Centered Difference Bound]
\label{lem:app_nstorm_peeling}
Suppose that for every \(k=0,\ldots,K\),
\[
    \left(
        \E[\|Z_{k+1}\|^2\mid\cF_k]
    \right)^{1/2}
    \le
    \gamma
    \left(
        a+h\|\nabla f(x^k)\|
    \right),
\]
and
\[
    \left(
        \E[\|\zeta_{k+1}\|^2\mid\cF_k]
    \right)^{1/2}
    \le
    \tau.
\]
Then
\begin{equation}
\label{eq:app_nstorm_Sb_generic}
    S_b(K)
    \le
    \frac{b_0}{\eta}
    +
    T\frac{a\gamma}{\sqrt\eta}
    +
    T\sqrt\eta\,\tau
    +
    \frac{h\gamma}{\eta}
    \sum_{k=0}^K g_k.
\end{equation}
\end{lemma}

\begin{proof}
Unrolling \eqref{eq:app_nstorm_decomp}, for \(k\ge1\),
\[
    E_k
    =
    (1-\eta)^kE_0
    +
    \sum_{t=0}^{k-1}(1-\eta)^{k-1-t}Z_{t+1}
    +
    \eta\sum_{t=0}^{k-1}(1-\eta)^{k-1-t}\zeta_{t+1}.
\]
The \(Z\)-sum is split into a gradient-independent part and a
gradient-dependent part. Let
\[
    A:=\gamma a,
    \qquad
    H_t:=\gamma h\|\nabla f(x^t)\|.
\]
If \(A+H_t>0\), define
\[
    \lambda_t:=\frac{A}{A+H_t};
\]
otherwise set \(\lambda_t:=0\). Then
\[
    Z_{t+1}^{(0)}:=\lambda_tZ_{t+1},
    \qquad
    Z_{t+1}^{(1)}:=(1-\lambda_t)Z_{t+1}
\]
are martingale differences and satisfy
\[
    \left(\E[\|Z_{t+1}^{(0)}\|^2\mid\cF_t]\right)^{1/2}\le A,
    \qquad
    \left(\E[\|Z_{t+1}^{(1)}\|^2\mid\cF_t]\right)^{1/2}\le H_t.
\]
By martingale orthogonality,
\[
    \E\left\|
    \sum_{t=0}^{k-1}(1-\eta)^{k-1-t}Z_{t+1}^{(0)}
    \right\|
    \le
    \frac{\gamma a}{\sqrt\eta}.
\]
For the gradient-dependent part, triangle inequality gives
\[
    \E\left\|
    \sum_{t=0}^{k-1}(1-\eta)^{k-1-t}Z_{t+1}^{(1)}
    \right\|
    \le
    \gamma h
    \sum_{t=0}^{k-1}(1-\eta)^{k-1-t}g_t.
\]
Similarly, for the fresh-noise martingale term,
\[
    \E\left\|
    \eta\sum_{t=0}^{k-1}(1-\eta)^{k-1-t}\zeta_{t+1}
    \right\|
    \le
    \sqrt\eta\,\tau.
\]
Thus,
\[
    b_k
    \le
    (1-\eta)^kb_0
    +
    \frac{\gamma a}{\sqrt\eta}
    +
    \sqrt\eta\,\tau
    +
    \gamma h
    \sum_{t=0}^{k-1}(1-\eta)^{k-1-t}g_t.
\]
Summing over \(k=0,\ldots,K\) and using
\[
    \sum_{k=0}^K(1-\eta)^k\le\frac1\eta,
    \qquad
    \sum_{k=0}^K\sum_{t=0}^{k-1}(1-\eta)^{k-1-t}g_t
    \le
    \frac1\eta\sum_{t=0}^K g_t,
\]
gives \eqref{eq:app_nstorm_Sb_generic}.
\end{proof}

\subsection{Estimator-difference bounds}
\label{app:nstorm_estimator_bounds}

\begin{lemma}[Single-sample estimator bounds]
\label{lem:app_nstorm_estimator_bounds}
Let
\[
    Q_T:=G+BT\gamma.
\]
The following bounds hold.

\begin{enumerate}[label=(\roman*), leftmargin=*]
\item Under Assumption~\ref{as:mss},
\[
    \left(
        \E[\|Z_{k+1}\|^2\mid\cF_k]
    \right)^{1/2}
    \le
    L\gamma,
    \qquad
    \left(
        \E[\|\zeta_{k+1}\|^2\mid\cF_k]
    \right)^{1/2}
    \le
    Q_T.
\]

\item Under Assumption~\ref{as:elsym_alpha} with \(\alpha\in(0,1)\), define
\[
    p:=\frac{\alpha}{1-\alpha},
    \qquad
    K_0:=2^{\frac{2-\alpha}{1-\alpha}}L_0,
    \qquad
    K_1:=2^{\frac{2-\alpha}{1-\alpha}}L_1,
    \qquad
    K_2:=(5L_1)^{\frac1{1-\alpha}},
\]
and
\[
    L_\alpha:=2^{\alpha/2}K_1,
    \qquad
    c_\alpha:=(1-\alpha)\alpha^{\alpha/(1-\alpha)}.
\]
Then, for every \(\lambda>0\),
\[
\begin{aligned}
    \left(
        \E[\|Z_{k+1}\|^2\mid\cF_k]
    \right)^{1/2}
    \le
    \gamma
    \left(
        K_0
        +
        c_\alpha L_\alpha\lambda^{-p}
        +
        L_\alpha Q_T^\alpha
        +
        K_2\gamma^p
        +
        L_\alpha\lambda\|\nabla f(x^k)\|
    \right),
\end{aligned}
\]
and
\[
    \left(
        \E[\|\zeta_{k+1}\|^2\mid\cF_k]
    \right)^{1/2}
    \le
    Q_T.
\]

\item Under Assumption~\ref{as:elsym_alpha} with \(\alpha=1\), if
\(\gamma\le 1/(4L_1)\), then
\[
\begin{aligned}
    \left(
        \E[\|Z_{k+1}\|^2\mid\cF_k]
    \right)^{1/2}
    \le
    \sqrt{2e^{3/4}}\gamma
    \left(
        L_0+2L_1Q_T+2L_1\|\nabla f(x^k)\|
    \right),
\end{aligned}
\]
and
\[
    \left(
        \E[\|\zeta_{k+1}\|^2\mid\cF_k]
    \right)^{1/2}
    \le
    Q_T.
\]
\end{enumerate}
\end{lemma}

\begin{proof}
For item (i), conditioning on \(\cF_k\), mean-square smoothness gives
\[
    \E[\|Z_{k+1}\|^2\mid\cF_k]
    \le
    \E_{\xi^{k+1}}
    \|\nabla f(x^{k+1};\xi^{k+1})-\nabla f(x^k;\xi^{k+1})\|^2
    \le
    L^2\gamma^2.
\]
The fresh-noise bound follows from Assumption~\ref{as:bg0_oracle} and
Lemma~\ref{lem:app_nstorm_traj}.

For item (ii), conditioning on \(\cF_k\) and using the expected generalized
smoothness reduction,
\[
\begin{aligned}
    \left(
        \E[\|Z_{k+1}\|^2\mid\cF_k]
    \right)^{1/2}
    \le
    \gamma
    \left(
        K_0
        +
        K_1\E_\xi\|\nabla f(x^k;\xi)\|^\alpha
        +
        K_2\gamma^p
    \right).
\end{aligned}
\]
By Lemma~\ref{lem:app_nstorm_alpha_moment},
\[
    K_1\E_\xi\|\nabla f(x^k;\xi)\|^\alpha
    \le
    L_\alpha
    \left(
        \|\nabla f(x^k)\|^\alpha+q_k^\alpha
    \right).
\]
Using \(q_k\le Q_T\) and
\[
    \|\nabla f(x^k)\|^\alpha
    \le
    \lambda\|\nabla f(x^k)\|
    +
    c_\alpha\lambda^{-p},
\]
gives the displayed \(Z\)-bound. The bound for \(\zeta_{k+1}\) follows from
$\mathsf{BG}$-0 and \(q_k\le Q_T\).

Item (iii) is the corresponding \(\alpha=1\) reduction. It gives
\[
\begin{aligned}
    \left(
        \E[\|Z_{k+1}\|^2\mid\cF_k]
    \right)^{1/2}
    \le
    \sqrt{2e^{3/4}}\gamma
    \left(
        L_0+2L_1q_k+2L_1\|\nabla f(x^k)\|
    \right),
\end{aligned}
\]
and \(q_k\le Q_T\) completes the proof.
\end{proof}

\subsection{Single-sample master bounds}
\label{app:nstorm_master}

\begin{theorem}[Single-sample NSTORM master bounds]
\label{thm:app_nstorm_master}
Consider the single-sample \(\mathsf{NSTORM}\) update
in~\eqref{eq:nstorm_updates}. Suppose $f$ is bounded below and
\ref{as:bg0_oracle} hold. Then the following statements hold.

\begin{enumerate}[label=(\roman*), leftmargin=*]

\item Suppose Assumption~\ref{as:mss} holds. Then
\begin{equation}
\label{eq:app_nstorm_master_mss}
\begin{aligned}
    \E\|\nabla f(\widehat x)\|
    \le\;&
    \frac{\Delta}{\gamma T}
    +
    \frac{2b_0}{\eta T}
    +
    \frac{2L\gamma}{\sqrt\eta}
    +
    2\sqrt\eta(G+BT\gamma)
    +
    \frac{L}{2}\gamma .
\end{aligned}
\end{equation}

\item Suppose Assumption~\ref{as:elsym_alpha} holds with
\(\alpha\in(0,1)\). Define
\[
    p:=\frac{\alpha}{1-\alpha},
    \qquad
    K_0:=2^{\frac{2-\alpha}{1-\alpha}}L_0,
    \qquad
    K_1:=2^{\frac{2-\alpha}{1-\alpha}}L_1,
    \qquad
    K_2:=(5L_1)^{\frac1{1-\alpha}},
\]
and
\[
    L_\alpha:=2^{\alpha/2}K_1,
    \qquad
    c_\alpha:=(1-\alpha)\alpha^{\alpha/(1-\alpha)}.
\]
Let
\[
    Q_T:=G+BT\gamma.
\]
If
\[
    L_\alpha\lambda\gamma\le1,
    \qquad
    8L_\alpha\lambda\frac{\gamma}{\eta}\le1,
\]
then
\begin{equation}
\label{eq:app_nstorm_master_alpha}
\begin{aligned}
    \E\|\nabla f(\widehat x)\|
    \le\;&
    \frac{4\Delta}{\gamma T}
    +
    \frac{8b_0}{\eta T}+
    \frac{8\gamma}{\sqrt\eta}
    \left(
        K_0
        +
        c_\alpha L_\alpha\lambda^{-p}
        +
        L_\alpha Q_T^\alpha
        +
        K_2\gamma^p
    \right)
    +
    8\sqrt\eta\,Q_T
    \\
    &+
    2\gamma
    \left(
        K_0
        +
        c_\alpha L_\alpha\lambda^{-p}
        +
        2K_2\gamma^p
    \right)
    +
    2L_\alpha\gamma
    \frac1T\sum_{k=0}^K \E[q_k^\alpha].
\end{aligned}
\end{equation}

\item Suppose Assumption~\ref{as:elsym_alpha} holds with \(\alpha=1\). If
\[
    \gamma\le \frac{1}{4\sqrt2L_1},
    \qquad
    16\sqrt{2e^{3/4}}L_1\frac{\gamma}{\eta}\le1,
\]
then
\begin{equation}
\label{eq:app_nstorm_master_one}
\begin{aligned}
    \E\|\nabla f(\widehat x)\|
    \le\;&
    \frac{4\Delta}{\gamma T}
    +
    \frac{8b_0}{\eta T}
    +
    \frac{8\gamma}{\sqrt\eta}
    \sqrt{2e^{3/4}}
    \left(
        L_0+2L_1(G+BT\gamma)
    \right)
    +
    8\sqrt\eta(G+BT\gamma)
    \\
    &+
    4\sqrt2L_0\gamma
    +
    8\sqrt2L_1\gamma
    \frac1T\sum_{k=0}^K\E[q_k].
\end{aligned}
\end{equation}
\end{enumerate}
\end{theorem}

\begin{proof}
For item (i), summing \eqref{eq:app_nstorm_desc_mss} gives
\[
    \gamma\sum_{k=0}^K g_k
    \le
    \Delta
    +
    2\gamma S_b(K)
    +
    \frac{L}{2}\gamma^2T.
\]
Apply Lemma~\ref{lem:app_nstorm_peeling} with
\[
    a=L,
    \qquad
    h=0,
    \qquad
    \tau=G+BT\gamma.
\]
Then
\[
    S_b(K)
    \le
    \frac{b_0}{\eta}
    +
    T\frac{L\gamma}{\sqrt\eta}
    +
    T\sqrt\eta(G+BT\gamma).
\]
Substituting this into the summed descent inequality and dividing by
\(\gamma T\) gives \eqref{eq:app_nstorm_master_mss}.

For item (ii), summing \eqref{eq:app_nstorm_desc_alpha} gives
\[
\begin{aligned}
    \frac{\gamma}{2}\sum_{k=0}^K g_k
    \le\;&
    \Delta
    +
    2\gamma S_b(K)
    +
    \frac{\gamma^2T}{2}
    \left(
        K_0+c_\alpha L_\alpha\lambda^{-p}+2K_2\gamma^p
    \right)+
    \frac{\gamma^2}{2}L_\alpha
    \sum_{k=0}^K\E[q_k^\alpha].
\end{aligned}
\]
Apply Lemma~\ref{lem:app_nstorm_peeling} with
\[
    a=
    K_0+c_\alpha L_\alpha\lambda^{-p}+L_\alpha Q_T^\alpha+K_2\gamma^p,
    \qquad
    h=L_\alpha\lambda,
    \qquad
    \tau=Q_T.
\]
Then
\[
    S_b(K)
    \le
    \frac{b_0}{\eta}
    +
    T\frac{a\gamma}{\sqrt\eta}
    +
    T\sqrt\eta Q_T
    +
    \frac{L_\alpha\lambda\gamma}{\eta}
    \sum_{k=0}^K g_k.
\]
Substituting this bound into the descent inequality gives a feedback term
\[
    2\gamma\cdot
    \frac{L_\alpha\lambda\gamma}{\eta}
    \sum_{k=0}^K g_k.
\]
The condition
\[
    8L_\alpha\lambda\frac{\gamma}{\eta}\le1
\]
implies that this feedback can be absorbed into the left-hand side. After
absorption and multiplication by \(4\), division by \(\gamma T\) gives
\eqref{eq:app_nstorm_master_alpha}.

For item (iii), sum \eqref{eq:app_nstorm_desc_one}:
\[
    \frac{\gamma}{2}\sum_{k=0}^K g_k
    \le
    \Delta
    +
    2\gamma S_b(K)
    +
    \sqrt2L_0\gamma^2T
    +
    2\sqrt2L_1\gamma^2\sum_{k=0}^K\E[q_k].
\]
Apply Lemma~\ref{lem:app_nstorm_peeling} with
\[
    a=
    \sqrt{2e^{3/4}}
    \left(L_0+2L_1(G+BT\gamma)\right),
    \qquad
    h=2\sqrt{2e^{3/4}}L_1,
    \qquad
    \tau=G+BT\gamma.
\]
The condition
\[
    16\sqrt{2e^{3/4}}L_1\frac{\gamma}{\eta}\le1
\]
absorbs the resulting feedback term. Dividing by \(\gamma T/4\) gives
\eqref{eq:app_nstorm_master_one}.
\end{proof}

\subsection{Proof of Theorem~\ref{thm:nstorm_three_regimes}}
\label{app:nstorm_theorem_proof}
Because the initialization samples are independent and the stochastic oracle is unbiased,
\[
\begin{aligned}
\E\|v^0-\nabla f(x^0)\|^2 = \frac{1}{N_{\rm init}} \E_\xi\|\nabla f(x^0;\xi)-\nabla f(x^0)\|^2
= \frac{\sigma_0^2}{N_{\rm init}}.
\end{aligned}
\]
Consequently,
\[
b_0 = \E\|v^0-\nabla f(x^0)\| \le \frac{\sigma_0}{\sqrt{N_{\rm init}}}.
\]
We use this initialization bound in all three smoothness regimes below.
\begin{corollary}[Mean-square smoothness with sharp initialization]
\label{cor:app_nstorm_mss}
Suppose $f$ is bounded below, \ref{as:bg0_oracle}, and
\ref{as:mss} hold. Initialize
\[
    v^0 = \frac1{N_{\rm init}} \sum_{i=1}^{N_{\rm init}}\nabla f(x^0;\xi_i^{\rm init}),
    \qquad
    N_{\rm init}  := \left\lceil T^{1/2}\right\rceil.
\]
Choose
\[
    \eta=T^{-1},  \qquad \gamma=\gamma_0T^{-3/4}.
\]
Then
\[
\begin{aligned}
\E\|\nabla f(\widehat x)\| \le{}& \left(\frac{\Delta}{\gamma_0} + 2\sigma_0 + \frac52L\gamma_0 +2B\gamma_0 \right)T^{-1/4} +2GT^{-1/2}.
\end{aligned}
\]
Consequently,
\[
    \E\|\nabla f(\widehat x)\|
    =
    \Ocal(T^{-1/4}),
\]
and the SFO complexity is \(\Ocal(\varepsilon^{-4})\).
\end{corollary}

\begin{proof}
By the initialization choice,
\[
    b_0
    =
    \E\|v^0-\nabla f(x^0)\|
    \le
    \frac{\sigma_0}{\sqrt{N_{\rm init}}}
    \le \sigma_0 T^{-1/4}.
\]
Apply \eqref{eq:app_nstorm_master_mss}. Under the schedule,
\[
    \frac{\Delta}{\gamma T}
    =
    \frac{\Delta}{\gamma_0}T^{-1/4},
    \qquad
    \frac{2b_0}{\eta T} = 2b_0 \le 2\sigma_0T^{-1/4}.
\]
Also,
\[
    \frac{2L\gamma}{\sqrt\eta}
    =
    2L\gamma_0T^{-1/4},
    \qquad
    \frac{L}{2}\gamma
    =
    \frac{L\gamma_0}{2}T^{-3/4},
\]
and
\[
    2\sqrt\eta(G+BT\gamma)
    =
    2GT^{-1/2}
    +
    2B\gamma_0T^{-1/4}.
\]
Combining these estimates proves the displayed bound. Since $T\ge1$,
\[
\frac{L\gamma_0}{2}T^{-3/4}
\le
\frac{L\gamma_0}{2}T^{-1/4}.
\]
Combining the estimates gives the coefficient $\frac52L\gamma_0$ and proves the compact bound
in Theorem~\ref{thm:nstorm_three_regimes}(i). Since \(N_{\rm init}=\Ocal(T^{1/2})\) and each transition uses two SFO calls, the total SFO cost is \(N_{\rm init}+2T=\Ocal(T)\). Setting \(T=\Ocal(\varepsilon^{-4})\) gives \(\Ocal(\varepsilon^{-4})\).
\end{proof}

\begin{corollary}[Expected \(\alpha\)-symmetric generalized smoothness]
\label{cor:app_nstorm_alpha}
Suppose $f$ is bounded below, \ref{as:bg0_oracle}, and
\ref{as:elsym_alpha} hold with \(\alpha\in(0,1)\). Define
\[
    p:=\frac{\alpha}{1-\alpha},
    \qquad
    K_0:=2^{\frac{2-\alpha}{1-\alpha}}L_0,
    \qquad
    K_1:=2^{\frac{2-\alpha}{1-\alpha}}L_1,
    \qquad
    K_2:=(5L_1)^{\frac1{1-\alpha}},
\]
and let
\[
    T:=K+1,
    \qquad
    \theta:=\frac{1}{4+\alpha}.
\]
Initialize
\[
    v^0
    =
    \frac1{N_{\rm init}}
    \sum_{i=1}^{N_{\rm init}}\nabla f(x^0;\xi_i^{\rm init}),
    \qquad
    N_{\rm init} := \left\lceil T^{2(1-\alpha)\theta}\right\rceil.
\]
Choose
\[
    \gamma=\gamma_0T^{-(3+\alpha)\theta},
    \qquad
    \eta=\eta_0T^{-4\theta},
    \qquad
    \lambda=\lambda_0T^{-(1-\alpha)\theta},
\]
where \(\eta_0\in(0,1]\), \(\gamma_0>0\), and \(\lambda_0>0\). Suppose
\[
    2^{\alpha/2}K_1\lambda_0\gamma_0\le1,
    \qquad
    2^{(6+\alpha)/2}K_1\lambda_0\frac{\gamma_0}{\eta_0}\le1.
\]
Then the bound stated in Theorem~\ref{thm:nstorm_three_regimes}(ii) holds.
In particular,
\[
    \E\|\nabla f(\widehat x)\|
    =
    \Ocal(T^{-1/(4+\alpha)}),
\]
and the SFO complexity is \(\Ocal(\varepsilon^{-(4+\alpha)})\).
\end{corollary}

\begin{proof}
By the initialization choice,
\[
    b_0
    \le
    \frac{\sigma_0}{\sqrt{N_{\rm init}}}
    \le
    \sigma_0 T^{-(1-\alpha)\theta}.
\]
Let
\[
    L_\alpha:=2^{\alpha/2}K_1,
    \qquad
    c_\alpha:=(1-\alpha)\alpha^{\alpha/(1-\alpha)}.
\]
The two conditions in the corollary are exactly
\[
    L_\alpha\lambda\gamma\le1,
    \qquad
    8L_\alpha\lambda\frac{\gamma}{\eta}\le1.
\]
Therefore, \eqref{eq:app_nstorm_master_alpha} applies.

We now substitute the parameter choices. First,
\[
    \frac{4\Delta}{\gamma T}
    =
    \frac{4\Delta}{\gamma_0}T^{-\theta},
    \qquad
    \frac{8b_0}{\eta T}
    \le
    \frac{8 \sigma_0}{\eta_0}T^{-\theta}.
\]
Also,
\[
    Q_T=G+BT\gamma
    =
    G+B\gamma_0T^\theta.
\]
Using \((a+b)^\alpha\le a^\alpha+b^\alpha\),
\[
    Q_T^\alpha
    \le
    G^\alpha+B^\alpha\gamma_0^\alpha T^{\alpha\theta}.
\]
Substituting these estimates into
\eqref{eq:app_nstorm_master_alpha} gives exactly
\[
\begin{aligned}
\E\|\nabla f(\widehat x)\|
\le \Bigg[&
\frac{4\Delta}{\gamma_0}
+
\frac{8\sigma_0}{\eta_0}
+
\frac{
8(1-\alpha)\alpha^{\frac{\alpha}{1-\alpha}}
2^{\alpha/2}K_1\gamma_0\lambda_0^{-\frac{\alpha}{1-\alpha}}
}{\sqrt{\eta_0}}
\\
&+
\frac{
8\cdot 2^{\alpha/2}K_1B^\alpha\gamma_0^{1+\alpha}
}{\sqrt{\eta_0}}
+
8\sqrt{\eta_0}B\gamma_0
\Bigg]T^{-\theta}
\\
&+
\left[
\frac{8K_0\gamma_0}{\sqrt{\eta_0}}
+
\frac{8\cdot 2^{\alpha/2}K_1G^\alpha\gamma_0}{\sqrt{\eta_0}}
\right]T^{-(1+\alpha)\theta}
+
8\sqrt{\eta_0}GT^{-2\theta}
\\
&+
\frac{
8K_2\gamma_0^{\frac1{1-\alpha}}
}{\sqrt{\eta_0}}
T^{-\frac{1+3\alpha}{(1-\alpha)(4+\alpha)}}
+
2K_0\gamma_0T^{-(3+\alpha)\theta}
\\
&+
2(1-\alpha)\alpha^{\frac{\alpha}{1-\alpha}}
2^{\alpha/2}K_1\gamma_0
\lambda_0^{-\frac{\alpha}{1-\alpha}}
T^{-3\theta}
\\
&+
4K_2\gamma_0^{\frac1{1-\alpha}}
T^{-\frac{3+\alpha}{(1-\alpha)(4+\alpha)}}
+
2^{1+\alpha/2}K_1G^\alpha\gamma_0
T^{-(3+\alpha)\theta}
\\
&+
2^{1+\alpha/2}K_1B^\alpha\gamma_0^{1+\alpha}
T^{-3\theta}.
\end{aligned}
\]

For the compact statement of
Theorem~\ref{thm:nstorm_three_regimes}(ii),
define the horizon-independent constant
\[
\begin{aligned}
\Phi_\alpha
:={}&
\frac58\Bigl(
    K_0
    +c_\alpha L_\alpha\lambda_0^{-p}
    +L_\alpha G^\alpha
    +L_\alpha B^\alpha\gamma_0^\alpha
\Bigr)
+\frac34K_2\gamma_0^p,
\end{aligned}
\]
where
$p=\alpha/(1-\alpha)$,
$L_\alpha=2^{\alpha/2}K_1$, and
$c_\alpha=(1-\alpha)
\alpha^{\alpha/(1-\alpha)}$.

Since $T\ge1$ and $\eta_0\le1$,
the terms involving $K_0$,
$c_\alpha L_\alpha\lambda_0^{-p}$,
$L_\alpha G^\alpha$, and
$L_\alpha B^\alpha\gamma_0^\alpha$
are bounded in total by
\[
\frac{10\gamma_0}{\sqrt{\eta_0}}
\left(
    K_0
    +c_\alpha L_\alpha\lambda_0^{-p}
    +L_\alpha G^\alpha
    +L_\alpha B^\alpha\gamma_0^\alpha
\right)T^{-\theta}.
\]
The two $K_2$ contributions are bounded in total by
$
\frac{12K_2\gamma_0^{p+1}}
{\sqrt{\eta_0}}T^{-\theta}.
$
Moreover,
$8\sqrt{\eta_0}GT^{-2\theta}
\le8\sqrt{\eta_0}GT^{-\theta}$.
Combining these bounds, we obtain
\[
\begin{aligned}
\E\|\nabla f(\widehat x)\|
\le
\Bigg(
    &\frac{4\Delta}{\gamma_0}
    +\frac{8\sigma_0}{\eta_0}
    +8\sqrt{\eta_0}(G+B\gamma_0)+\frac{16\gamma_0}{\sqrt{\eta_0}}
    \Phi_\alpha
\Bigg)T^{-\theta}.
\end{aligned}
\]
This is precisely the finite-horizon
bound stated in
Theorem~\ref{thm:nstorm_three_regimes}(ii).
In particular,
\[
\E\|\nabla f(\widehat x)\|
=\Ocal(T^{-1/(4+\alpha)}).
\]

The initialization cost satisfies
\[
    N_{\rm init}
    =
    \Ocal(T^{2(1-\alpha)/(4+\alpha)})
    =
    o(T),
\]
and each transition uses two SFO calls. Hence the total SFO cost is
\[
    N_{\rm init}+2T=\Ocal(T).
\]
Taking \(T=\Ocal(\varepsilon^{-(4+\alpha)})\) gives
\(\Ocal(\varepsilon^{-(4+\alpha)})\).
\end{proof}

\begin{corollary}[Expected \(1\)-symmetric generalized smoothness]
\label{cor:app_nstorm_one}
Suppose $f$ is bounded below, \ref{as:bg0_oracle}, and
\ref{as:elsym_alpha} hold with \(\alpha=1\). Use one-sample initialization, so
that
\[
    b_0:=\E\|v^0-\nabla f(x^0)\| \leq \sigma_0.
\]
Choose
\[
    \eta=T^{-4/5},
    \qquad
    \gamma=\gamma_0T^{-4/5},
    \qquad
    0<\gamma_0\le
    \frac{1}{16\sqrt{2e^{3/4}}L_1}.
\]
Then
\[
\begin{aligned}
\E\|\nabla f(\widehat x)\|
&\le
\left(
    \frac{4\Delta}{\gamma_0}
    +8b_0
    +8B\gamma_0
    +16\sqrt{2e^{3/4}}L_1B\gamma_0^2
\right)T^{-1/5}
\\
&\quad
+
\left(
    8\sqrt{2e^{3/4}}\gamma_0(L_0+2L_1G)
    +8G
\right)T^{-2/5}
\\
&\quad
+
8\sqrt2L_1B\gamma_0^2T^{-3/5}
+
\left(
    4\sqrt2L_0\gamma_0
    +8\sqrt2L_1G\gamma_0
\right)T^{-4/5}.
\end{aligned}
\]
Consequently,
\[
    \E\|\nabla f(\widehat x)\|
    =
    \Ocal(T^{-1/5}),
\]
and the SFO complexity is \(\Ocal(\varepsilon^{-5})\).
\end{corollary}

\begin{proof}
The stepsize condition implies
\[
    \gamma\le \frac{1}{4\sqrt2L_1},
    \qquad
    16\sqrt{2e^{3/4}}L_1\frac{\gamma}{\eta}
    =
    16\sqrt{2e^{3/4}}L_1\gamma_0
    \le1.
\]
Thus \eqref{eq:app_nstorm_master_one} applies. Under the schedule,
\[
    \frac{4\Delta}{\gamma T}
    =
    \frac{4\Delta}{\gamma_0}T^{-1/5},
    \qquad
    \frac{8b_0}{\eta T}
    =
    8b_0T^{-1/5}.
\]
Moreover,
\[
    Q_T=G+BT\gamma
    =
    G+B\gamma_0T^{1/5}.
\]
Therefore,
\[
\begin{aligned}
    \frac{8\gamma}{\sqrt\eta}
    \sqrt{2e^{3/4}}
    (L_0+2L_1Q_T)
    &=
    8\sqrt{2e^{3/4}}\gamma_0(L_0+2L_1G)T^{-2/5} \\
    &\quad
    +
    16\sqrt{2e^{3/4}}L_1B\gamma_0^2T^{-1/5},
\end{aligned}
\]
and
\[
    8\sqrt\eta Q_T
    =
    8GT^{-2/5}
    +
    8B\gamma_0T^{-1/5}.
\]
Finally, by Lemma~\ref{lem:app_nstorm_traj},
\[
    \frac1T\sum_{k=0}^K\E[q_k]
    \le
    G+BT\gamma
    =
    G+B\gamma_0T^{1/5}.
\]
Hence
\[
\begin{aligned}
    8\sqrt2L_1\gamma
    \frac1T\sum_{k=0}^K\E[q_k]
    &\le
    8\sqrt2L_1G\gamma_0T^{-4/5}
    +
    8\sqrt2L_1B\gamma_0^2T^{-3/5}.
\end{aligned}
\]
Together with
$
    4\sqrt2L_0\gamma
    =
    4\sqrt2L_0\gamma_0T^{-4/5},
$
this proves the displayed bound. 

To recover the compact statement of Theorem~\ref{thm:nstorm_three_regimes}(iii), note that the stepsize restriction implies
\[
16\sqrt{2e^{3/4}}L_1\gamma_0\le1,
\qquad
8\sqrt2L_1\gamma_0\le1.
\]
Moreover,
\[
8\sqrt{2e^{3/4}}+4\sqrt2\le23.
\]
Using these inequalities, $b_0\le\sigma_0$,
and $T\ge1$, we obtain
\[
\begin{aligned}
\E\|\nabla f(\widehat x)\|
\le{}&
\left(
    \frac{4\Delta}{\gamma_0}
    +8\sigma_0
    +10B\gamma_0
\right)T^{-1/5}+
\left(
    23L_0\gamma_0
    +10G
\right)T^{-2/5}.
\end{aligned}
\]
This proves the bound in Theorem~\ref{thm:nstorm_three_regimes}(iii). The method uses one SFO call for
initialization and two SFO calls per transition, so the total cost is \(1+2T=\Ocal(T)\). Taking \(T=\Ocal(\varepsilon^{-5})\) gives \(\Ocal(\varepsilon^{-5})\).
\end{proof}

%%%%%%%%%%%%%%%%%%%%%%%%%%%%%%%%%%%%%%%%%%%%%%%%%%%%%%%%%%%%%%%%%%%%%%%%%%%%
%% Recovery of known rates for NSTORM (deterministic and bounded variance)
%%%%%%%%%%%%%%%%%%%%%%%%%%%%%%%%%%%%%%%%%%%%%%%%%%%%%%%%%%%%%%%%%%%%%%%%%%%%

\subsection{Recovery of bounded variance rates (\texorpdfstring{$B=0$}{B=0})}
\label{app:nstorm_bv_recovery}

When $B=0$, the $\mathsf{BG}$-0 condition reduces to bounded variance
$\E_\xi\|\nabla f(x;\xi)-\nabla f(x)\|^2\le G^2$, and the local noise
scale simplifies to $q_k=G$ for every $k$. In particular, the
trajectory-dependent terms $BT\gamma$ vanish from the master bounds.
Under the bounded variance schedule
$\eta=\eta_0 T^{-2/3}$, $\gamma=\gamma_0 T^{-2/3}$, $\mathsf{NSTORM}$
recovers the standard variance-reduced rate $\Ocal(\varepsilon^{-3})$ for
all three smoothness regimes, matching the SPIDER rate
of~\citet{fang2018spider,chen2023generalized}.

\begin{corollary}[Mean-square smoothness under bounded variance]
\label{cor:app_nstorm_mss_bv}
Suppose $f$ is bounded below, \ref{as:bg0_oracle} (with $B=0$), and
\ref{as:mss} hold. Use one-sample initialization $v^0=\nabla f(x^0;\xi^0)$,
so that $b_0\le G$. Choose
\[
    \eta = \eta_0\,T^{-2/3},
    \qquad
    \gamma = \gamma_0\,T^{-2/3},
\]
where $\eta_0\in(0,1]$ and $\gamma_0>0$. Then
\begin{equation}
\label{eq:app_nstorm_mss_bv_final}
    \E\|\nabla f(\widehat x)\|
    \le
    \left(
        \frac{\Delta}{\gamma_0}
        + \frac{2b_0}{\eta_0}
        + \frac{2L\gamma_0}{\sqrt{\eta_0}}
        + 2\sqrt{\eta_0}\,G
    \right) T^{-1/3}
    + \frac{L\gamma_0}{2}\,T^{-2/3}.
\end{equation}
Consequently,
$\E\|\nabla f(\widehat x)\| = \Ocal(T^{-1/3})$,
and the SFO complexity is $\Ocal(\varepsilon^{-3})$.
\end{corollary}

\begin{proof}
Since $B=0$, we have $q_k=G$ for every $k$ and $Q_T=G$. Apply
\eqref{eq:app_nstorm_master_mss} from Theorem~\ref{thm:app_nstorm_master}(i).
Under the schedule $\eta=\eta_0 T^{-2/3}$ and $\gamma=\gamma_0 T^{-2/3}$,
\[
    \frac{\Delta}{\gamma T}
    = \frac{\Delta}{\gamma_0}\,T^{-1/3},
    \qquad
    \frac{2b_0}{\eta T}
    = \frac{2b_0}{\eta_0}\,T^{-1/3},
\]
\[
    \frac{2L\gamma}{\sqrt{\eta}}
    = \frac{2L\gamma_0}{\sqrt{\eta_0}}\,T^{-1/3},
    \qquad
    \frac{L}{2}\gamma
    = \frac{L\gamma_0}{2}\,T^{-2/3},
\]
and, since $B=0$,
\[
    2\sqrt{\eta}(G+BT\gamma)
    = 2\sqrt{\eta_0}\,G\,T^{-1/3}.
\]
Substitution gives \eqref{eq:app_nstorm_mss_bv_final}. The method uses one
SFO call for initialization and two per transition, giving
$M_K=1+2T=\Ocal(\varepsilon^{-3})$.
\end{proof}

\begin{corollary}[Expected $\alpha$-symmetric generalized smoothness under bounded variance]
\label{cor:app_nstorm_alpha_bv}
Suppose $f$ is bounded below, \ref{as:bg0_oracle} (with $B=0$), and
\ref{as:elsym_alpha} hold with $\alpha\in(0,1)$. Define
$L_\alpha:=2^{\alpha/2}K_1$, $c_\alpha:=(1-\alpha)\alpha^{\alpha/(1-\alpha)},$ $p:=\alpha/(1-\alpha).$
Use one-sample initialization, so that $b_0\le G$. Choose
\[
    \eta = \eta_0\,T^{-2/3},
    \qquad
    \gamma = \gamma_0\,T^{-2/3},
    \qquad
    \lambda = \lambda_0,
\]
where $\eta_0\in(0,1]$, $\gamma_0>0$, and $\lambda_0>0$ satisfy
\[
    L_\alpha\lambda_0\gamma_0\le1,
    \qquad
    8L_\alpha\lambda_0\frac{\gamma_0}{\eta_0}\le1.
\]
Then
\begin{align}
    \E\|\nabla f(\widehat x)\|
    &\le
    \left[
        \frac{4\Delta}{\gamma_0}
        + \frac{8b_0}{\eta_0}
        + \frac{8\gamma_0}{\sqrt{\eta_0}}
        \left(
            K_0 + c_\alpha L_\alpha\lambda_0^{-p} + L_\alpha G^\alpha
        \right)
        + 8\sqrt{\eta_0}\,G
    \right] T^{-1/3}
    \notag\\
    &\quad
    + \frac{8K_2\gamma_0^{p+1}}{\sqrt{\eta_0}}\,T^{-(1+2p)/3}
    + 2\gamma_0
    \left(
        K_0 + c_\alpha L_\alpha\lambda_0^{-p}
    \right) T^{-2/3}
    \notag\\
    &\quad
    + 4K_2\gamma_0^{p+1}\,T^{-2(p+1)/3}
    + 2L_\alpha G^\alpha\gamma_0\,T^{-2/3}.
    \label{eq:app_nstorm_alpha_bv_final}
\end{align}
Consequently,
$\E\|\nabla f(\widehat x)\| = \Ocal(T^{-1/3})$,
and the SFO complexity is $\Ocal(\varepsilon^{-3})$.
\end{corollary}

\begin{proof}
Since $B=0$, we have $q_k=G$ for every $k$, $Q_T=G$,
and
\[
    \frac{1}{T}\sum_{k=0}^{K}\E[q_k^\alpha] = G^\alpha.
\]
We verify the conditions of Theorem~\ref{thm:app_nstorm_master}(ii). Since
$T\ge1$,
\[
    L_\alpha\lambda\gamma
    = L_\alpha\lambda_0\gamma_0\,T^{-2/3}
    \le L_\alpha\lambda_0\gamma_0
    \le 1,
\]
and
\[
    8L_\alpha\lambda\frac{\gamma}{\eta}
    = 8L_\alpha\lambda_0\frac{\gamma_0}{\eta_0}
    \le 1.
\]
Apply \eqref{eq:app_nstorm_master_alpha}. Since $B=0$, $Q_T=G$ is constant.
The first two terms give
\[
    \frac{4\Delta}{\gamma T}
    = \frac{4\Delta}{\gamma_0}\,T^{-1/3},
    \qquad
    \frac{8b_0}{\eta T}
    = \frac{8b_0}{\eta_0}\,T^{-1/3}.
\]
The fresh-noise term gives
\[
    8\sqrt{\eta}\,G
    = 8\sqrt{\eta_0}\,G\,T^{-1/3}.
\]

For the estimator-difference term,
\[
\begin{aligned}
\frac{8\gamma}{\sqrt{\eta}}
\left(
K_0+c_\alpha L_\alpha\lambda_0^{-p}
+L_\alpha G^\alpha+K_2\gamma^p
\right)
&=
\frac{8\gamma_0}{\sqrt{\eta_0}}
\left(
K_0+c_\alpha L_\alpha\lambda_0^{-p}
+L_\alpha G^\alpha
\right)T^{-1/3}
\\
&+
\frac{8K_2\gamma_0^{p+1}}{\sqrt{\eta_0}}
T^{-(1+2p)/3}.
\end{aligned}
\]

The descent-remainder term is
\[
\begin{aligned}
2\gamma
\left(
K_0+c_\alpha L_\alpha\lambda_0^{-p}
+2K_2\gamma^p
\right)&=
2\gamma_0
\left(
K_0+c_\alpha L_\alpha\lambda_0^{-p}
\right)T^{-2/3}+
4K_2\gamma_0^{p+1}T^{-2(p+1)/3}.
\end{aligned}
\]

The averaged $\mathsf{BG}$-0 term gives
\[
    2L_\alpha\gamma\,G^\alpha
    = 2L_\alpha G^\alpha\gamma_0\,T^{-2/3}.
\]
Since $p>0$, the exponents $(1+2p)/3 > 1/3$ and $2(p+1)/3 > 2/3$, so the
dominant rate is $T^{-1/3}$. This proves
\eqref{eq:app_nstorm_alpha_bv_final}.

The method uses one SFO call for initialization and two per transition, giving
$M_K=1+2T=\Ocal(\varepsilon^{-3})$.
\end{proof}

\begin{corollary}[Expected $1$-symmetric generalized smoothness under bounded variance]
\label{cor:app_nstorm_one_bv}
Suppose $f$ is bounded below, \ref{as:bg0_oracle} (with $B=0$), and
\ref{as:elsym_alpha} hold with $\alpha=1$ and constants $L_0,L_1>0$. Use
one-sample initialization, so that $b_0\le G$. Choose
\[
    \eta = \eta_0\,T^{-2/3},
    \qquad
    \gamma = \gamma_0\,T^{-2/3},
\]
where $\eta_0\in(0,1]$ and $\gamma_0>0$. If $L_1>0$, assume
\[
    \gamma_0\le\frac{1}{4\sqrt{2}\,L_1},
    \qquad
    16\sqrt{2e^{3/4}}\,L_1\frac{\gamma_0}{\eta_0}\le1.
\]
Then
\begin{align}
    \E\|\nabla f(\widehat x)\|
    &\le
    \left[
        \frac{4\Delta}{\gamma_0}
        + \frac{8b_0}{\eta_0}
        + \frac{8\sqrt{2e^{3/4}}\,\gamma_0}{\sqrt{\eta_0}}
        \left(L_0+2L_1 G\right)
        + 8\sqrt{\eta_0}\,G
    \right] T^{-1/3}
    \notag\\
    &\quad
    + \left(
        4\sqrt{2}\,L_0\gamma_0
        + 8\sqrt{2}\,L_1 G\gamma_0
    \right) T^{-2/3}.
    \label{eq:app_nstorm_one_bv_final}
\end{align}
Consequently,
$\E\|\nabla f(\widehat x)\| = \Ocal(T^{-1/3})$,
and the SFO complexity is $\Ocal(\varepsilon^{-3})$.
\end{corollary}

\begin{proof}
Since $B=0$, we have $q_k=G$ for every $k$ and
\[
    \frac{1}{T}\sum_{k=0}^{K}\E[q_k] = G.
\]
If $L_1>0$, the conditions of Theorem~\ref{thm:app_nstorm_master}(iii) hold
because
\[
    \gamma = \gamma_0\,T^{-2/3} \le \gamma_0 \le \frac{1}{4\sqrt{2}\,L_1},
\]
and
\[
    16\sqrt{2e^{3/4}}\,L_1\frac{\gamma}{\eta}
    = 16\sqrt{2e^{3/4}}\,L_1\frac{\gamma_0}{\eta_0}
    \le 1.
\]
Apply \eqref{eq:app_nstorm_master_one} with $B=0$. We have
\[
    \frac{4\Delta}{\gamma T}
    = \frac{4\Delta}{\gamma_0}\,T^{-1/3},
    \qquad
    \frac{8b_0}{\eta T}
    = \frac{8b_0}{\eta_0}\,T^{-1/3}.
\]
Since $G+BT\gamma=G$,
\[
    \frac{8\gamma}{\sqrt{\eta}}\sqrt{2e^{3/4}}
    \left(L_0+2L_1 G\right)
    = \frac{8\sqrt{2e^{3/4}}\,\gamma_0}{\sqrt{\eta_0}}
    (L_0+2L_1 G)\,T^{-1/3},
\]
\[
    8\sqrt{\eta}\,G
    = 8\sqrt{\eta_0}\,G\,T^{-1/3},
\]
\[
    4\sqrt{2}\,L_0\gamma
    = 4\sqrt{2}\,L_0\gamma_0\,T^{-2/3},
\]
and
\[
    8\sqrt{2}\,L_1\gamma\,G
    = 8\sqrt{2}\,L_1 G\gamma_0\,T^{-2/3}.
\]
Combining gives \eqref{eq:app_nstorm_one_bv_final}. The SFO cost is
$1+2T=\Ocal(\varepsilon^{-3})$.
\end{proof}

\subsection{Recovery of deterministic rates (\texorpdfstring{$B=G=0$}{B=G=0})}
\label{app:nstorm_det_recovery}

When the oracle is deterministic ($B=G=0$), the stochastic gradient noise
vanishes. Therefore $v^k=\nabla f(x^k)$, $b_k=0$, and $q_k=0$ for all $k$.
The $\mathsf{NSTORM}$ update reduces to normalized gradient descent, and the
estimator recursion plays no role. The classical $\Ocal(\varepsilon^{-2})$
rate is recovered for all three smoothness regimes.

\begin{corollary}[Mean-square smoothness, deterministic]
\label{cor:app_nstorm_mss_det}
Suppose $f$ is bounded below and~\ref{as:mss} hold, and the oracle is
deterministic ($B=G=0$). Consider the normalized gradient descent update
\[
    x^{k+1} = x^k - \gamma\,\frac{\nabla f(x^k)}{\|\nabla f(x^k)\|}.
\]
Choose $\gamma = \gamma_0\,T^{-1/2}$. Then
\begin{equation}
\label{eq:app_nstorm_mss_det_final}
    \|\nabla f(\widehat x)\|
    \le
    \frac{\Delta}{\gamma_0}\,T^{-1/2}
    + \frac{L\gamma_0}{2}\,T^{-1/2}.
\end{equation}
Consequently,
$\|\nabla f(\widehat x)\| = \Ocal(T^{-1/2})$,
and the iteration complexity is $\Ocal(\varepsilon^{-2})$.
\end{corollary}

\begin{proof}
Under determinism, $v^k=\nabla f(x^k)$ and $b_k=0$ for all $k$. Since
$d^k=\nabla f(x^k)/\|\nabla f(x^k)\|$ (when nonzero), we have
$\langle\nabla f(x^k),d^k\rangle=\|\nabla f(x^k)\|$.

Assumption~\ref{as:mss} implies $L$-smoothness via Jensen
(Lemma~\ref{lem:app_nstorm_desc_mss}). The one-step descent becomes
\[
    f(x^{k+1})
    \le f(x^k)
    - \gamma\|\nabla f(x^k)\|
    + \frac{L}{2}\gamma^2.
\]
Summing from $k=0$ to $K$ and using $f(x^{K+1})\ge f^{\inf}$,
\[
    \gamma\sum_{k=0}^{K}\|\nabla f(x^k)\|
    \le \Delta + \frac{L}{2}\gamma^2 T.
\]
Dividing by $\gamma T$ and substituting $\gamma=\gamma_0 T^{-1/2}$ gives
\eqref{eq:app_nstorm_mss_det_final}.
\end{proof}

\begin{corollary}[Expected $\alpha$-symmetric generalized smoothness, deterministic]
\label{cor:app_nstorm_alpha_det}
Suppose $f$ is bounded below and~\ref{as:elsym_alpha} hold with
$\alpha\in(0,1)$ and constants $K_0, K_1, K_2>0$, and the oracle is
deterministic ($B=G=0$). Define $L_\alpha:=2^{\alpha/2}K_1$,
$c_\alpha:=(1-\alpha)\alpha^{\alpha/(1-\alpha)}$, and $p:=\alpha/(1-\alpha)$.
Consider the normalized gradient descent update. Choose
$\gamma = \gamma_0\,T^{-1/2}$ and let $\lambda_0>0$ satisfy
$L_\alpha\lambda_0\gamma_0\le1$. Then
\begin{equation}
\label{eq:app_nstorm_alpha_det_final}
    \|\nabla f(\widehat x)\|
    \le
    \frac{2\Delta}{\gamma_0}\,T^{-1/2}
    + \gamma_0
    \left(
        K_0 + c_\alpha L_\alpha\lambda_0^{-p} + 2K_2\gamma_0^p T^{-p/2}
    \right) T^{-1/2}.
\end{equation}
Consequently,
$\|\nabla f(\widehat x)\| = \Ocal(T^{-1/2})$,
and the iteration complexity is $\Ocal(\varepsilon^{-2})$.
\end{corollary}

\begin{proof}
Under determinism, $v^k=\nabla f(x^k)$, $b_k=0$, and $q_k=0$ for all $k$.
The one-step descent from Lemma~\ref{lem:app_nstorm_desc_alpha} with $b_k=0$
and $\E[q_k^\alpha]=0$ becomes
\[
    f(x^{k+1})
    + \frac{\gamma}{2}\|\nabla f(x^k)\|
    \le f(x^k)
    + \frac{\gamma^2}{2}
    \left(
        K_0 + c_\alpha L_\alpha\lambda_0^{-p} + 2K_2\gamma^p
    \right).
\]
Summing from $k=0$ to $K$, using $f(x^{K+1})\ge f^{\inf}$, and dividing by
$\gamma T/2$ gives
\[
    \frac{1}{T}\sum_{k=0}^{K}\|\nabla f(x^k)\|
    \le
    \frac{2\Delta}{\gamma T}
    + \gamma
    \left(
        K_0 + c_\alpha L_\alpha\lambda_0^{-p} + 2K_2\gamma^p
    \right).
\]
Substituting $\gamma=\gamma_0 T^{-1/2}$ gives
\[
    \frac{2\Delta}{\gamma T}
    = \frac{2\Delta}{\gamma_0}\,T^{-1/2},
\]
and
\[
    \gamma\left(K_0 + c_\alpha L_\alpha\lambda_0^{-p}\right)
    = \gamma_0\left(K_0 + c_\alpha L_\alpha\lambda_0^{-p}\right) T^{-1/2},
\]
and
\[
    2K_2\gamma^{p+1}
    = 2K_2\gamma_0^{p+1}\,T^{-(p+1)/2}.
\]
Since $p>0$, the last term decays faster than $T^{-1/2}$. This proves
\eqref{eq:app_nstorm_alpha_det_final} and the $\Ocal(T^{-1/2})$ rate.
\end{proof}

\begin{corollary}[Expected $1$-symmetric generalized smoothness, deterministic]
\label{cor:app_nstorm_one_det}
Suppose $f$ is bounded below and~\ref{as:elsym_alpha} hold with
$\alpha=1$ and constants $L_0,L_1>0$, and the oracle is deterministic
($B=G=0$). Consider the normalized gradient descent update. Choose
$\gamma = \gamma_0\,T^{-1/2}$ with $\gamma_0\le 1/(4\sqrt{2}\,L_1)$ if
$L_1>0$. Then
\begin{equation}
\label{eq:app_nstorm_one_det_final}
    \|\nabla f(\widehat x)\|
    \le
    \frac{2\Delta}{\gamma_0}\,T^{-1/2}
    + 2\sqrt{2}\,L_0\gamma_0\,T^{-1/2}.
\end{equation}
Consequently,
$\|\nabla f(\widehat x)\| = \Ocal(T^{-1/2})$,
and the iteration complexity is $\Ocal(\varepsilon^{-2})$.
\end{corollary}

\begin{proof}
Under determinism, $v^k=\nabla f(x^k)$, $b_k=0$, and $q_k=0$ for all $k$.
The one-step descent from Lemma~\ref{lem:app_nstorm_desc_one} becomes
\[
    f(x^{k+1})
    + \frac{\gamma}{2}\|\nabla f(x^k)\|
    \le f(x^k)
    + \sqrt{2}\,L_0\gamma^2.
\]
Indeed, the $2\sqrt{2}\,L_1\gamma^2\|\nabla f(x^k)\|$ term is absorbed by
$(\gamma/2)\|\nabla f(x^k)\|$ via the stepsize condition
$2\sqrt{2}\,L_1\gamma \le 1/2$, and the $q_k$-dependent term vanishes.

Summing from $k=0$ to $K$, using $f(x^{K+1})\ge f^{\inf}$, and dividing by
$\gamma T/2$ gives
\[
    \frac{1}{T}\sum_{k=0}^{K}\|\nabla f(x^k)\|
    \le
    \frac{2\Delta}{\gamma T}
    + 2\sqrt{2}\,L_0\gamma.
\]
Substituting $\gamma=\gamma_0 T^{-1/2}$ proves
\eqref{eq:app_nstorm_one_det_final}.
\end{proof}

\subsection{Discussion on recovered rates}
\label{app:nstorm_recovery_discussion}

The six recovery corollaries above confirm the claims made in the discussion
following Theorem~\ref{thm:nstorm_three_regimes}.

\begin{itemize}[leftmargin=*,label=\textbullet]
    \item \textit{Bounded variance recovery.}
    When $B=0$, Corollaries~\ref{cor:app_nstorm_mss_bv},
    \ref{cor:app_nstorm_alpha_bv}, and~\ref{cor:app_nstorm_one_bv} all achieve
    $\Ocal(\varepsilon^{-3})$ SFO complexity with the bounded variance
    schedule $\eta=\eta_0 T^{-2/3}$, $\gamma=\gamma_0 T^{-2/3}$. This
    recovers the standard variance-reduced rate
    of~\citet{fang2018spider,cutkosky2019momentum,chen2023generalized}
    under mean-square smoothness and expected generalized smoothness alike.

    \item \textit{Deterministic recovery.}
    When $B=G=0$, Corollaries~\ref{cor:app_nstorm_mss_det},
    \ref{cor:app_nstorm_alpha_det}, and~\ref{cor:app_nstorm_one_det} all
    achieve $\Ocal(\varepsilon^{-2})$ iteration complexity. Since $\mathsf{NSTORM}$ reduces to normalized gradient descent in the deterministic setting, these rates match the classical first-order complexity of~\citet{nesterov2004introductory} and the generalized smooth
    rates of~\citet{chen2023generalized}.

    \item \textit{Schedule transition.}
    As the noise weakens from $\mathsf{BG}$-0 to bounded variance to
    deterministic, the $\mathsf{NSTORM}$ schedules and rates transition as
    follows for the mean-square smoothness case:
    \[
        (\eta,\,\gamma)
        =
        \begin{cases}
            (T^{-1},\;\gamma_0 T^{-3/4}), & \mathsf{BG}\text{-}0
            \text{ (sharp init.)}, \\[3pt]
            (\eta_0 T^{-2/3},\;\gamma_0 T^{-2/3}), &
            \text{bounded variance}, \\[3pt]
            (1,\;\gamma_0 T^{-1/2}), & \text{deterministic},
        \end{cases}
    \]
    yielding rates $T^{-1/4}$, $T^{-1/3}$, and $T^{-1/2}$, respectively.
    The $\eta$ parameter decays more slowly as the noise weakens, because the
    STORM recursion needs less aggressive averaging when the fresh-noise and
    trajectory-growth contributions are smaller.

    For the expected $\alpha$-generalized smoothness case, the transition is:
    \[
        \text{SFO complexity}
        =
        \begin{cases}
            \Ocal(\varepsilon^{-(4+\alpha)}), &
            \mathsf{BG}\text{-}0, \\[3pt]
            \Ocal(\varepsilon^{-3}), &
            \text{bounded variance}, \\[3pt]
            \Ocal(\varepsilon^{-2}), & \text{deterministic}.
        \end{cases}
    \]
    Thus the $\alpha$-dependent price $\Ocal(\varepsilon^{-(4+\alpha)})$
    under $\mathsf{BG}$-0 is entirely an artifact of the interaction
    between distance-dependent variance and gradient-dependent curvature; it
    disappears when either source of difficulty is removed.
\end{itemize}

\newpage
\section{Experimental Details}

\subsection{\texorpdfstring{$\mathsf{BG}$-0}{BG-0} Wrapper}\label{app:exp_bg0_oracle}
Given a fixed deterministic objective $f$, the $\mathsf{BG}$-0 wrapper turns the setup into a stochastic problem in which the stochastic gradients satisfy the $\mathsf{BG}$-0 variance model exactly.

Given the independent random variables $\rho$ and $u$ satisfying
\[
\E[\rho]=0,\qquad \E[\rho^2]=1,\qquad
\E[u]=0,\qquad \E\|u\|^2=1.
\]
where  $\rho\in\{-1,+1\}$ uniformly and
$u\sim\mathcal N(0,I_d/d)$. For each sample
$\xi=(\rho,u)$, the stochastic loss is given by
\[
f_\xi(x)
=
f(x)
+
\frac{B\rho}{2}\|x-x^0\|^2
+
G\langle u,x-x^0\rangle,
\]
where the corresponding stochastic gradient is
\[
\nabla f_\xi(x)
=
\nabla f(x)+B\rho(x-x^0)+Gu .
\]
Here, the objective is unchanged and the oracle is unbiased as
\[
\E_\xi[f_\xi(x)]=f(x),
\qquad
\E_\xi[\nabla f_\xi(x)]=\nabla f(x).
\]
Moreover,
\[
\E_\xi\|\nabla f_\xi(x)-\nabla f(x)\|^2
=
\E_\xi\|B\rho(x-x^0)+Gu\|^2
=
B^2\|x-x^0\|^2+G^2,
\]
where the cross term vanishes by independence and the zero-mean assumptions. Thus, our wrapper satisfies the exact \(\mathsf{BG}\)-0 oracle.

As the extension to the mini-batch case of size $b$ (which is essential to compare with the dynamic batching based models in \cite{fazla2026lower}), we draw independent
\(\xi_i=(\rho_i,u_i)\) and use
\[
g_b(x)
=
\frac1b\sum_{i=1}^b \nabla f_{\xi_i}(x)
=
\nabla f(x)
+
\frac1b\sum_{i=1}^b
\left(B\rho_i(x-x^0)+Gu_i\right),
\]
which satisfies
\[
\E[g_b(x)]=\nabla f(x),
\qquad
\E\|g_b(x)-\nabla f(x)\|^2
=
\frac{B^2\|x-x^0\|^2+G^2}{b}.
\]
For the STORM and NSTORM recursions, when a transition from \(x\) to \(y\) requires two stochastic gradients, we evaluate \(\nabla f_\xi(x)\) and \(\nabla f_\xi(y)\) with the same sample \(\xi\) and analogously reuse the same mini-batch in the batched case. Since this process requires the computation of two stochastic gradients, the corresponding SFO is doubled in our calculations.

\begin{lemma}
\label{lem:exp_scalar_bg0_wrapper}
Let \(f\in\mathcal L_{\rm sym}^*(\alpha)\) with constants \(L_0,L_1>0\), and let \(\alpha\in(0,1]\). Let \(\rho\in\{-1,+1\}\) uniformly and \(u\sim\mathcal N(0,I_d/d)\), independently. Thus
\[
\E[\rho]=0,\qquad \E[\rho^2]=1,\qquad
\E[u]=0,\qquad \E\|u\|^2=1.
\]
For each \(\xi=(\rho,u)\), define
\[
f_\xi(x)
=
f(x)
+
\frac{B\rho}{2}\|x-x^0\|^2
+
G\langle u,x-x^0\rangle .
\]
Then \(\E_\xi[f_\xi(x)]=f(x)\), the stochastic gradient is unbiased, and
\[
\E_\xi\|\nabla f_\xi(x)-\nabla f(x)\|^2
=
B^2\|x-x^0\|^2+G^2.
\]
Moreover, there exist constants \(\widetilde L_0,\widetilde L_1>0\) such that, for all \(x,y\),
\[
\E_\xi\|\nabla f_\xi(y)-\nabla f_\xi(x)\|^2
\le
\|y-x\|^2
\E_\xi\left[
\left(
\widetilde L_0+
\widetilde L_1
\max_{\theta\in[0,1]}
\|\nabla f_\xi(x_\theta)\|^\alpha
\right)^2
\right],
\]
where \(x_\theta=\theta y+(1-\theta)x\).
\end{lemma}

\begin{proof}
Since $\E[\rho]=0$ and $\E[u]=0$ we have $\E_\xi[f_\xi(x)]=f(x)$ and
$\E_\xi[\nabla f_\xi(x)]=\nabla f(x)$. Moreover,
\[
\nabla f_\xi(x)-\nabla f(x)=B\rho(x-x^0)+Gu.
\]
Using independence and the properties of $\rho$ and $u$, we obtain
\[
\E_\xi\|\nabla f_\xi(x)-\nabla f(x)\|^2
=
B^2\|x-x^0\|^2+G^2.
\]
It remains to prove the expected generalized smoothness bound. For any
$x,y$, let $d=y-x$. Since
\[
\nabla f_\xi(y)-\nabla f_\xi(x)
=
\nabla f(y)-\nabla f(x)+B\rho d,
\]
we have
\[
\E_\xi\|\nabla f_\xi(y)-\nabla f_\xi(x)\|^2
=
\|\nabla f(y)-\nabla f(x)\|^2+B^2\|d\|^2.
\]
Since \(f\in\mathcal L_{\rm sym}^*(\alpha)\),
\[
\|\nabla f(y)-\nabla f(x)\|^2
\le
\|d\|^2
\left(
L_0+
L_1\max_{\theta\in[0,1]}\|\nabla f(x_\theta)\|^\alpha
\right)^2.
\]
Thus, with
\[
R:=\max_{\theta\in[0,1]}\|\nabla f(x_\theta)\|,
\]
we obtain
\[
\E_\xi\|\nabla f_\xi(y)-\nabla f_\xi(x)\|^2
\le
\|d\|^2
\left[
2L_0^2+2L_1^2R^{2\alpha}+B^2
\right].
\]
Now define
\[
M_\xi:=\max_{\theta\in[0,1]}\|\nabla f_\xi(x_\theta)\|.
\]
For any fixed \(\theta\), write
\[
Z_\theta:=B\rho(x_\theta-x^0)+Gu.
\]
The random vector \(Z_\theta\) is symmetric: \(Z_\theta\) and \(-Z_\theta\)
have the same distribution. Hence, for any fixed vector \(a\) and any
\(p>0\),
\[
\E_\xi\|a+Z_\theta\|^p
=
\frac12\E_\xi\left[\|a+Z_\theta\|^p+\|a-Z_\theta\|^p\right]
\ge
\frac12\|a\|^p,
\]
because at least one of \(\|a+Z_\theta\|\) and \(\|a-Z_\theta\|\) is at
least \(\|a\|\). Taking \(a=\nabla f(x_\theta)\), \(p=2\alpha\), and using
\(M_\xi\ge \|\nabla f_\xi(x_\theta)\|\), we get
\[
\E_\xi[M_\xi^{2\alpha}]
\ge
\frac12\|\nabla f(x_\theta)\|^{2\alpha}
\quad\text{for all }\theta\in[0,1].
\]
Therefore,
\[
R^{2\alpha}\le 2\E_\xi[M_\xi^{2\alpha}].
\]
Choose constants \(\widetilde L_0,\widetilde L_1>0\) such that
\[
\widetilde L_0^2\ge 2L_0^2+B^2,
\qquad
\widetilde L_1^2\ge 4L_1^2.
\]
Then
\[
2L_0^2+B^2+2L_1^2R^{2\alpha}
\le
\widetilde L_0^2+\widetilde L_1^2\E_\xi[M_\xi^{2\alpha}]
\le
\E_\xi\left[
\left(
\widetilde L_0+\widetilde L_1M_\xi^\alpha
\right)^2
\right].
\]
Combining the preceding inequalities proves the claim.
\end{proof}

This wrapper lets us evaluate stochastic methods on deterministic test objectives while imposing an unbounded variance that grows with the distance from initialization. It therefore provides a controlled experimental model for the distance-dependent stochastic noise studied in the paper, without changing the underlying deterministic objective. Lemma~\ref{lem:exp_scalar_bg0_wrapper} verifies that for all $\alpha\in(0,1]$ (including the $\alpha=1/2$ and $\alpha=2/3$ experimental setups) the proposed framework is compatible with the expected symmetric generalized smoothness framework.

\subsection{Phase Retrieval}

The first experiment is the phase retrieval objective from \citet{chen2023generalized}. Given measurements
$y_r=|a_r^\top x_\star|^2$, we minimize
\[
f(x)=\frac{1}{2m}\sum_{r=1}^m \left(y_r-|a_r^\top x|^2\right)^2 .
\]
This nonconvex objective is not globally $L$-smooth, but it belongs to
the $\alpha$-symmetric generalized smooth class
$\mathcal L_{\rm sym}^*(2/3)$, and its stochastic finite-sum version satisfies the corresponding expected condition~\cite{chen2023generalized}. In our experiments, we use dimension \(d=100\) and \(m=3000\) measurements. The measurement matrix has independent entries
\(a_{r,j}\sim\mathcal N(0,0.01)\), and the target signal is drawn as
\(x_{\star,j}\sim\mathcal N(0,1)\). Hence, $y_r=|a_r^\top x_\star|^2$. The optimization is initialized independently from $x^0_j\sim\mathcal N(5,1)$, which places the initial point away from the target signal and makes the distance-dependent variance growth visible. After constructing this deterministic phase retrieval objective, we apply the $\mathsf{BG}$-0 wrapper from Appendix~\ref{app:exp_bg0_oracle}. Thus, the stochasticity is controlled through the $\mathsf{BG}$-0 oracle.

\subsection{Cubic Polynomial}
For the second experiment, we use the following polynomial function as the deterministic objective:
\[
f(x)=|x|^{(2-\alpha)/(1-\alpha)}
\]
which has been shown in \citep{chen2023generalized} to satisfy $f\in\mathcal{L}_{\mathrm{sym}}^{*}(\alpha)$ for $x\in\mathbb{R}$ and $\alpha\in(0,1)$. We instantiate this example with $\alpha=1/2$, so $f(x)=|x|^3$, and use the same distance-dependent $\mathsf{BG}$-0 oracle. This gives a minimal one-dimensional experiment in which both the generalized smoothness exponent and the $\mathsf{BG}$-0 variance growth are explicit.

\subsection{Additional Optimizer Comparison}
To further investigate the role of normalization and momentum under $\mathsf{BG}$-0 noise, we compare the proposed methods with unnormalized SGDM, Adam, SIGNUM, and SignSGD, in addition to the dynamic-batching baselines. We consider the same cubic-polynomial and phase-retrieval problems, with dimensions \(d=1\) and \(d=100\), respectively. Table ~\ref{tab:additional_baselines} summarizes the final gradient norms reported over three random seeds.

In one dimension, normalized momentum reduces to SIGNUM when the momentum parameters and stochastic samples are matched, explaining their identical results on the cubic-polynomial problem. In higher dimensions, $\mathsf{NSGDM}$ uses an \(\ell_2\)-normalized momentum vector, whereas SIGNUM applies the sign operation coordinatewise. The comparisons also illustrate the benefit of temporal averaging: single-sample SignSGD performs worse than momentum-based normalized updates in both experiments. Dynamic batching can further reduce the final gradient norm, but requires additional oracle evaluations.

\begin{table}[t]
    \centering
    \caption{Comparison of final gradient norms $\|\nabla f(x^K)\|$ under $\mathsf{BG}$-0 noise. Results are reported as mean $\pm$ standard
    deviation over three random seeds.}
    \label{tab:additional_baselines}
    \small
    \setlength{\tabcolsep}{6pt}
    \begin{tabular}{@{}lcc@{}}
        \toprule
        Method & Cubic ($d=1$) & Phase retrieval ($d=100$) \\
        \midrule
        SGD ($b=1$)
            & $3.317 \pm 0.099$ & $1.245 \pm 0.239$ \\
        SGDM ($b=1$)
            & $1.204 \pm 0.018$ & $0.817 \pm 0.317$ \\
        SIGNUM ($b=1$)
            & $0.473 \pm 0.272$ & $0.546 \pm 0.552$ \\
        SignSGD ($b=1$)
            & $3.409 \pm 0.833$ & $8.989 \pm 7.845$ \\
        Adam ($b=1$)
            & $0.693 \pm 0.076$ & $0.520 \pm 0.190$ \\
        \textbf{NSGDM} ($b=1$)
            & $0.473 \pm 0.272$ & $0.393 \pm 0.481$ \\
        \textbf{NSTORM} ($b=1$)
            & $0.102 \pm 0.160$ & $0.256 \pm 0.073$ \\
        \midrule
        SGD (dynamic batch)
            & $3.337 \pm 0.062$ & $0.408 \pm 0.027$ \\
        SGDM (dynamic batch)
            & $0.191 \pm 0.013$ & $0.220 \pm 0.006$ \\
        Adam (dynamic batch)
            & $0.137 \pm 0.015$ & $0.154 \pm 0.029$ \\
        STORM (dynamic batch)
            & $0.031 \pm 0.023$ & $0.127 \pm 0.069$ \\
        \bottomrule
    \end{tabular}
\end{table}

\section{Hyperparameters}
\label{app:exp_hyperparams}
We report the hyperparameters used for the normalized methods in Figures~\ref{fig:phase_retrieval} and~\ref{fig:poly_separation}. We denote $T:=K+1$, where $K$ is the total number of iterations for a given method. The schedules below specify the actual choices for the plotted $\mathsf{NSGDM}$ and $\mathsf{NSTORM}$ curves. For fair comparison, we apply grid search to set the learning rate of $\mathsf{SGD}$ (b=1), $\mathsf{SGD}$ (dynamic batch), and $\mathsf{STORM}$ (dynamic batch).

\paragraph{Normalized SGDM.}
For single sample \(\mathsf{NSGDM}\), we use the schedule from Theorem~\ref{thm:nsgdm_three_regimes}:
\[
    \eta = T^{-2/3},
    \qquad
    \gamma = \gamma_0 T^{-5/6},
\]
which is shared by the standard smooth, \(\Lcal_{\rm sym}^*(\alpha)\), and \(\Lcal_{\rm sym}^*(1)\) cases in Theorem~\ref{thm:nsgdm_three_regimes}.

\paragraph{Normalized STORM.}
For \(\mathsf{NSTORM}\), the initialization batch and schedule depend on the smoothness regime in Theorem~\ref{thm:nstorm_three_regimes}. For the expected \(\alpha\)-symmetric generalized smoothness experiments, we retain the horizon-dependent scaling of the parameter schedules, with experiment-specific initialization constants
\[
    N_{\rm init}
    \propto \left\lceil T^{\frac{2(1-\alpha)}{4+\alpha}}\right\rceil,
    \qquad
    \gamma=\gamma_0T^{-\frac{3+\alpha}{4+\alpha}},
    \qquad
    \eta=\eta_0T^{-\frac{4}{4+\alpha}}.
\]
The specific initialization batch sizes and stepsize coefficients used in our experiments are reported below.
\paragraph{Phase retrieval instantiation.}
For Figure~\ref{fig:phase_retrieval}, the phase retrieval objective satisfies the \(\alpha\)-symmetric generalized smoothness condition with \(\alpha=2/3\). We use \(T=10001\), \(B=1\), \(G=1\), \(\gamma_0=10\) for \(\mathsf{NSGDM}\), and \(\gamma_0=7.5\), \(\eta_0=1\), $N_{\rm init}=4$ for \(\mathsf{NSTORM}\). Thus, the corresponding hyperparameters are set to
\[
    \mathsf{NSGDM:}\qquad
    \gamma \approx 4.641\times 10^{-3},
    \qquad
    \eta \approx 2.154\times 10^{-3},
\]
and
\[
    \mathsf{NSTORM:}\qquad
    \gamma \approx 5.397\times 10^{-3},
    \qquad
    \eta \approx 3.73\times 10^{-4}.
\]

\paragraph{Cubic polynomial instantiation.}
For Figure~\ref{fig:poly_separation}, the polynomial experiment is set to satisfy the $\alpha$-symmetric generalized smoothness condition with \(\alpha=1/2\), \(T=10001\), \(B=1.5\), \(G=0.5\), and \(\gamma_0=\eta_0=1\), $N_{\rm init}=2$. Thus, the corresponding hyperparameters are set to
\[
    \mathsf{NSGDM:} \qquad \gamma \approx 4.641\times 10^{-4},
    \qquad
    \eta \approx 2.154\times 10^{-3}.
\]
and
\[
    \mathsf{NSTORM:} \qquad \gamma \approx 7.742\times 10^{-4},
    \qquad
    \eta \approx 2.782\times 10^{-4}.
\]

\end{document}